\documentclass{article}
\usepackage{kr}

\usepackage{times}
\usepackage{soul}
\usepackage{url}
\usepackage[hidelinks]{hyperref}
\usepackage[utf8]{inputenc}
\usepackage[small]{caption}
\usepackage{graphicx}
\usepackage{amsmath}
\usepackage{amsthm}
\usepackage{amsfonts}
\usepackage{booktabs}
\usepackage{algorithm}
\usepackage{algorithmic}
\usepackage{needspace}
\usepackage{pifont}
\usepackage{caption}
\usepackage{amsmath}
\usepackage{amsthm}
\usepackage{amsfonts}
\usepackage{amssymb}
\usepackage{listings}
\usepackage{multicol}
\usepackage{todonotes}
\usepackage{subcaption}
\usepackage{graphicx}
\usepackage[dvipsnames]{xcolor}

\usepackage{tikz}
\usetikzlibrary{shapes,decorations,shadows,arrows,calc}%,snakes}
\makeatletter

\tikzstyle{arg}=[draw,circle,fill=gray!15,inner sep=1pt,minimum size=.5cm]
\tikzstyle{argd}=[draw,circle,gray!70,inner sep=1pt,minimum size=.5cm,dashed]

\tikzstyle{argTD}=[draw, thick, circle, fill=gray!15,inner sep=0pt,minimum size=0.6cm,font=\small]
\tikzstyle{argR}=[draw, thick, circle, fill=gray!15,inner sep=0pt,minimum size=0.45cm,font=\small]
\tikzstyle{scc}=[draw, thick, rectangle,align=center, fill=gray!15,inner sep=0pt,minimum size=0.8cm,font=\small, rounded corners=0ex,]
\tikzstyle{argTDX}=[draw, dotted,thick, circle, inner sep=0pt,minimum size=0.6cm,font=\small]
\tikzstyle{sccX}=[draw,dotted, thick, rectangle,align=center, inner sep=0pt,minimum size=0.8cm,font=\small, rounded corners=0ex,]
\tikzstyle{argsmall}=[draw, thick, circle, fill=gray!15,inner sep=0pt,minimum size=0.4cm]
\tikzstyle{argsmallX}=[draw, thick, circle, inner sep=0pt,minimum size=0.3cm,dotted]
\tikzstyle{bag}=[draw,rectangle, rounded corners=1ex,
align=center, inner sep=1ex]

\tikzstyle{atts}=[draw,thick, inner sep=5pt, rounded corners=3pt]

\tikzstyle{nullarg}=[inner sep=0pt,outer sep=0pt,minimum size=0cm]
\tikzstyle{nullattack}=[draw, thick, |->]
\newcommand{\cl}{\mathit{cl}}

\newcommand{\clink}{^c}
\newcommand{\redEone}{^{E_1}}

\newcommand{\cf}{\textit{cf}}
\newcommand{\adm}{\textit{adm}}
\newcommand{\com}{\textit{com}}
\newcommand{\comp}{\textit{com}}

\newcommand{\stb}{\textit{stb}}

\newcommand{\prf}{\textit{pref}}
\newcommand{\pref}{\textit{pref}}

\newcommand{\grd}{\textit{grd}}

\newcommand{\tuple}[1]{\ensuremath{\left(#1\right)}}

\newcommand{\BF}{F}

\newcommand{\ie}{i.e.}

\newcommand{\SF}{\ensuremath{\mathit{SF}}}

\usepackage{thm-restate}
\usepackage{etoolbox}

\newtheorem{theorem}{Theorem}%[section]

\newtheorem{example}[theorem]{Example}

\newtheorem{lemma}[theorem]{Lemma}

\newtheorem{definition}[theorem]{Definition}

\renewcommand\thmcontinues[1]{ctd}

\definecolor{Blue}{cmyk}{0.95,0.54,0.00,0.05}
\definecolor{OrangeYellow}{cmyk}{0.02,0.42,0.74,0.10}
\definecolor{Green}{HTML}{5CB338}

\title{Splitting Argumentation Frameworks with Collective Attacks and Supports}

\author{%
Matti Berthold$^1$\and
Lydia Blümel$^1$\and
Giovanni Buraglio$^2$\and
Anna Rapberger$^3$ \\
\affiliations
$^1$FernUniversität in Hagen, Germany\\
$^2$TU Wien, Austria\\
$^3$TU Dortmund, Germany\\
\emails
\{lydia.bluemel, matti.berthold\}@fernuni-hagen.de,
giovanni.buraglio@tuwien.ac.at,
anna.rapberger@tu-dortmund.de
}

\begin{document}

\maketitle

\begin{abstract}

This work proposes novel splitting techniques for argumentation formalisms that incorporate supports between defeasible elements. We base our studies on \emph{bipolar set-based argumentation frameworks (BSAFs)} which generalize argumentation frameworks with collective attacks (SETAFs), as well as bipolar argumentation frameworks (BAFs), by incorporating both collective attacks and supports. 
Notably, BSAFs establish a crucial link to structured argumentation as they naturally capture general (potentially non-flat) assumption-based argumentation.
The increase in expressiveness calls for diverse forms of splitting. We consider splits over collective attacks (thereby generalizing the recently proposed splitting techniques for SETAFs), splits over collective supports, as well as splits over both collective attacks and supports. 
We establish suitable splitting schemata and prove their correctness for the most common argumentation semantics. 
\end{abstract}

\section{Introduction}

In the field of knowledge representation and reasoning, formal models of argumentation \cite{arguHandbookv1} provide computational approaches for representing and reasoning about argumentative scenarios, and have been successfully applied across a wide range of domains such as medical decision-making, law and explainable AI~\cite{AtkinsonBGHPRST17,Cyras0ABT21,LeofanteADFGJP024}.
In particular, abstract models of argumentation have gained significant popularity due to their flexible and rich modelling capabilities. Starting from Dung’s seminal work~\cite{Dung95}, Argumentation Frameworks (AFs) have become a central formalism for capturing debates and reasoning over them, with the ultimate goal of extracting rational viewpoints. In Dung-style AFs, arguments are treated as primitive entities, while their internal structure, e.g. the premises supporting their conclusions, is abstracted away. Conflicts between arguments are captured by a binary attack relation. Debates are thus represented via simple directed graphs where nodes are the arguments exchanged and edges keep track of their conflicts.
Rational viewpoints are then identified by computing sets of jointly acceptable arguments, so-called extensions. 

Over the years, numerous generalizations of AFs have been proposed to enrich their simple structure and increase its modelling capabilities~\cite{BrewkaESWW18,Amgoud08,NielsenP06,Modgil09}. Notably, bipolar argumentation frameworks (BAFs)~\cite{Amgoud08,BoellaGTV10} allow also for the explicit representation of support relations, enabling the modelling of scenarios in which arguments may reinforce one another.
In parallel, argumentation frameworks with collective attacks (also called SETAFs)~\cite{NielsenP06} have been introduced to capture situations in which an argument is insufficient to attack another on its own, but can do so jointly with others.
More recently, bipolar set-based argumentation frameworks (BSAFs)~\cite{BertholdRU24} have been proposed as a unifying generalization of both BAFs and SETAFs by combining the accrual of arguments with a notion of support. 
\textcolor{black}{With this, BSAFs can model scenarios in which arguments ‘join forces' to support or defeat another argument collectively.
In addition, they} have been shown to faithfully capture rule-based argumentation formalisms such as general (non-flat) assumption-based argumentation~\cite{BertholdRU24}. As a result, BSAFs offer a simple yet expressive means of representation for complex argumentative scenarios.

However, similarly to AFs, BAFs, and SETAFs, bipolar set-based argumentation frameworks exhibit a high computational complexity for most standard reasoning tasks~\cite{BertholdRU24}. 
In particular, the size of the search space may grow exponentially with the size of the framework, making direct computation of extensions infeasible in practice.
To address this challenge, incremental reasoning techniques have been developed that decompose a framework into smaller sub-frameworks, compute their extensions independently, and then combine the results~\cite{Baumann11,Liao13,BaroniGL14,BengelT25ADFSerialisation,BlumelRTT25,GiacominBC21}. Among these techniques, splitting~\cite{Baumann11,BaumannBDW12,Linsbichler14,BuraglioDKW24,Buraglio25} has proven to be particularly effective in abstract argumentation~\cite{BaumannBW11}. 
The central idea underlying splitting procedures is to compute the extensions of a given framework $F$ by decomposing it into suitable sub-frameworks $F_1$ and $F_2$ and then computing the extensions of $F_1$ and $F_2$ independently. 
Naturally, this requires suitable modifications of $F_2$ to take into account the acceptance status of the arguments in $F_1$.
\textcolor{black}{Moreover, although computational efficiency is the primary motivation for employing splitting techniques, they also provide a useful tool for the study of argumentation dynamics~\cite{RapbergerU23,Prakken23,CayrolSL10}, as previously computed parts of the framework can be reused when new information is introduced~\cite{Baumann11,LiaoJK11}.}

Splitting has been successfully applied to AFs and generalizations thereof; most crucial in our context, to SETAFs.
Nevertheless, despite the increasing demand for efficient approaches that take into account positve links, the development of splitting techniques for argumentation formalisms that feature support relations, particularly in the presence of collective relations, remains largely unaddressed.

In this paper, we start to fill this gap by investigating the possibilities and limitations of splitting in AFs that feature both collective attacks and supports. The added expressiveness of BSAFs allows for a richer variety of cuts: a framework may be split along collective attacks, collective supports, or combinations thereof. 
We introduce the necessary ingredients for splitting BSAFs in a progressive manner, addressing the unique challenges that arise due to the presence of supports. 
We investigate splitting techniques for attacks and supports separately before we present a unified splitting pipeline that enables the incremental computation of extensions for arbitrary cuts of a BSAF.

Our main contributions are as follows.
\begin{itemize}
    \item First, we develop a suitable attack splitting procedure for BSAFs by extending and revising the splitting procedure for SETAFs. 
    We prove the correctness of our splitting schema wrt.\ the most common argumentation semantics. Notably, we show that splitting works only partially under grounded semantics.
    \hfill {\color{gray}Section~\ref{sec:slitting over collective attacks}}
    \item Next, we investigate the splitting of supports. We define suitable modifications that enable splitting wrt.\ the common argumentation semantics. Notably, for preferred and grounded semantics, the splitting procedure selects only some, and not all, extensions. 
    \hfill {\color{gray}Section~\ref{sec:CollSupp}}
    \item Finally, we combine the previous techniques and consider arbitrary splits over collective attacks and supports. We prove a general splitting theorem, thereby paving the way for the incremental computation under most of the semantics. The combined splitting inherits the limitations concerning preferred and grounded semantics.  \hfill {\color{gray}Section~\ref{sec:combined splitting}}
\end{itemize}
All proofs are provided in the supplementary material.
\section{Background}
We recall BSAFs~\cite{BertholdRU24} which generalize SETAFs~\cite{NielsenP06}, and basics of splitting for SETAFs~\cite{BuraglioDKW24}. 

\subsection{Bipolar SETAFs}
Bipolar argumentation frameworks with collective attacks (BSAFs)~\cite{BertholdRU24} are a generalization of Dung's abstract argumentation frameworks (AFs)~\cite{Dung95}.
They combine the ideas underlying argumentation frameworks with collective attacks (SETAFs)~\cite{NielsenP06} and bipolar argumentation frameworks (BAFs)~\cite{CayrolL05a,Amgoud08,PoRaUlAAAI2024}. In particular, they incorporate notions of support and attack among arguments, and generalize both SETAFs and BAFs by modeling \emph{collective} attacks and supports. 

\begin{definition}
    A \emph{bipolar set-argumentation framework (BSAF)} is a tuple $\BF = \tuple{A,R,S}$, where $A$ is a finite set of arguments, $R \subseteq 2^A\times A$ is the attack relation and $S\subseteq 2^A\times A$ is the support relation.
\end{definition}
A \emph{SETAF} is a BSAF  $\BF = \tuple{A,R,S}$ with $S=\emptyset$; 
an \emph{AF} is a SETAF with $|T|=1$ for all $(T,h)\in R$.

For an attack (support) $(T,h)$, we call $T$ the \emph{tail} of the attack (support) and $h$ the \emph{head} of the attack (support). 
We  call $(T,h)$ a \emph{negative link} if $(T,h)\in R$, a \emph{positive link} if $(T,h)\in S$, and simply a \emph{link} if $(T,h)\in R\cup S$.
If $|T|=1$, we write $(t,h)$ instead of $(\{t\},h)$, where $T=\{t\}$.
Furthermore, given BSAFs $F=(A,R,S)$ and $F'=(A',R',S')$, we define set-operations component-wise; e.g., the union is defined as $F\cup F'=(A\cup A',R\cup R',S\cup S')$.

\textcolor{black}{In contrast to the broad agreement on the interpretation of attacks, there exist several interpretations of support in the literature~\cite{CohenGGS14,CohenPSM18,Polberg16}. BSAFs adopt a form of \emph{deductive support}~\cite{BoellaGTV10,BertholdRU24}. A central notion in this context is the \emph{closure} of a set.}
\begin{definition}
	Given BSAF $\BF\! = \! (A,R,S)$ and $E\subseteq A$, let
		$$supp_{\BF}(E):=E\cup \{h\in A\mid \exists\tuple{T,h}\in S: T\subseteq E\}.$$ 
	The \emph{closure} of $E$ is defined as $\cl_{\BF}(E):=\bigcup_{i\geq 1} supp_{\BF}^i(E)$; Hence, $E$ is \emph{closed} if $cl_{\BF}(E)=E$. 
\end{definition}

\begin{definition}
\label{def:BSAF gamma and more}
Given BSAF $\BF\! = \! (A,R,S)$, a set $E\!\subseteq\! A$ 
\emph{defends} $a\!\in\! A$ if for each closed attacker $E'\!\subseteq\! A$ of~$a$, $E$ attacks $E'$;
$E$ \emph{defends} $E'$ if $E$ defends each $a\!\in\! E'$.
    \end{definition}
    We omit the subscript $\BF$ for
$\cl$ if clear from context.

Let us now head to BSAF semantics. 
A set $E$ is conflict-free ($E\!\in\!\cf(\BF)$) if it does not attack itself;
$E$ is admissible ($E\in\adm(\BF)$) if it is conflict-free, closed and defends itself.%
\begin{definition}\label{def:BSAF semantics}
	Let $\BF$ be a BSAF and let $E\in\adm(\BF)$.
	\begin{itemize}
		\item 
		$E\in \comp(\BF)$ iff $E$ contains every argument it defends; 
		\item 
		$E\in \grd(\BF)$ iff $E$ is $\subseteq$-minimal in $\comp(\BF)$;
		\item 
		$E\in \pref(\BF)$ iff $E$ is $\subseteq$-maximal in $\adm(\BF)$.
		\item 
		$E\in \stb(\BF)$ iff $E$ attacks each $x\in  A \setminus E$.
	\end{itemize}
\end{definition}
Given a BSAF $F=(A,R,S)$ and a set of arguments $E\subseteq A$, we denote $E^+_R:=\{h\mid \exists T\subseteq E: (T,h)\in R\}$ and the range of $E$; by $E^\oplus_R:= E_R\cup E^+_R$. We omit $R$ if clear from the context.
Given a $\sigma$-extension $E$ of $\BF$, we say that an argument $a\in A$ is:  \emph{accepted} or \emph{in} if $a\in E$,  \emph{rejected}, \emph{defeated} or \emph{out} if $a\in E^+_R$ and \emph{undecided} if $a\notin E^\oplus_R$. 

Graphically, we depict the attack and the support relations of a BSAF via solid and dashed edges respectively.

\newcommand\distance{1.5}
\begin{example}\label{exm:bsaf}
We consider a BSAF $\BF$ with arguments 
$A=\{a,b,c,d,e\}$, attacks $R=\{$\textcolor{black}{$\tuple{f,c}$}, $\textcolor{violet}{\tuple{\{d,e\},f}},$ $\textcolor{black}{\tuple{c,d}}\}$, and support $S=\{\textcolor{cyan}{\tuple{\{a,b\},c}}\}$, as depicted below.
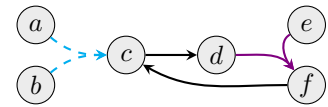
\begin{center}
		\begin{tikzpicture}[scale=0.8,>=stealth]
		\path
            (0,0.5)node[arg] (a){$a$}
            (0,-0.5)node[arg] (b){$b$}
			(1.5,0)node[arg] (c){$c$}
			(3,0)node[arg] (d){$d$}
			(4.5,0.5)node[arg] (e){$e$}
			(4.5,-0.5)node[arg] (f){$f$}
			;
			\path[thick,->]
            (c) edge[] (d)
			(a)edge[dashed,color=cyan,out=-40,in=180](c)
			(b)edge[dashed,color=cyan,out=40,in=180](c)
			(d)edge[color=violet,out=0,in=130](f)
			(e)edge[color=violet,out=210,in=130](f)
			(f)edge[out=180,in=-30](c)
            ;
				
		\end{tikzpicture}
	\end{center}
The arguments $a$ and $b$ jointly support $c$; thus, whenever $\{a,b\}$ is accepted, the argument $c$ must be true as well. 
Accepting $c$ is, however, impossible: to defend $c$ against $f$, both arguments $d$ and $e$ are required; only then, the joint attack $(\{d,e\},f)$ fires. However, accepting both is not possible since $c$ attacks $d$. 
Thus, $a$ and $b$ cannot be accepted together. It follows that $F$ has no complete and grounded extension; the admissible sets are $\emptyset$, $\{a\}$, $\{b\}$, $\{e\}$, $\{a,e\}$, and $\{b,e\}$.

\end{example}

\subsection{Splitting SETAFs}
In their recent work, \citeauthor{BuraglioDKW24}~(\citeyear{BuraglioDKW24}) investigated splitting procedures for SETAFs, extending splitting procedures for AFs~\cite{Baumann11}. 
A splitting of a SETAF $\SF=(A,R)$ separates the framework into two frameworks $\SF_1$ and $\SF_2$ which share attacks directed towards $\SF_2$. 
    \begin{definition}
		\label{def:setaf splitting}
		Let $\SF=(A,R)$ be a SETAF, $\SF_1=(A_1,R_1)$ and $\SF_2=(A_2,R_2)$ two sub-frameworks of $\SF$ such that $A_1\cap A_2=\emptyset$, $A=A_1 \cup A_2$ and $R=R_1 \cup R_2 \cup R_3$ with $R_3\subseteq\{(T,h)\in R\mid T\cap A_1\neq\emptyset, T\subseteq A, h\in A_2\}$.
        The triple $(\SF_1,\SF_2,R_3)$ is called a \emph{splitting} of $\SF$ and $R_3$ its set of \emph{negative links}.
        A link is \emph{undecided} if no argument in its tail is defeated, but at least one is undecided.
	\end{definition}
\begin{example}\label{ex:background split}
We consider a splitting $(\SF_1,\SF_2,R_3)$ of a SETAF $\SF$, as depicted below; $\SF_1$ and $\SF_2$ are left resp.\ right  of the dotted line; the set of shared negative links of $\SF_1$ and $\SF_2$ is given by $R_3=\{{\color{black}(\{a,z\},x)},{\color{black}(\{b,z\},y)}\}$.%
            	\begin{center}
			\begin{tikzpicture}[yscale=0.8,>=stealth]
			\path
            (1,2.3)node[arg] (a){$a$}
            (1,1.25)node[arg] (b){$b$}
			(0,2)node[arg] (c){$c$}
			(2.3,2.3)node[arg] (x){$x$}
			(3.1,1.8)node[arg] (y){$y$}
			(2,1.25)node[arg] (z){$z$}
			;
			\path[thick,->]
            (c)edge[<->](a)
			(b)edge[out=40,in=180](y)
			(z)edge[out=70,in=180,looseness=0.9](y)
			(z)edge[out=130,in=200,looseness=1.1](x)
			(a)edge[out=-20,in=200](x)
            (b) edge[loop left] (b)
            ;
			
			\draw [thick,dotted,red] (1.5,2.7) -- (1.5,0.8);
			
			\end{tikzpicture}
		\end{center}
\end{example}
\citeauthor{BuraglioDKW24}~(\citeyear{BuraglioDKW24}) developed a procedure that enables the incremental computation of the extensions of $\SF$ wrt.\ a given semantics $\sigma$. 
The extensions of $\SF$ can be computed as a combination of extensions of $\SF_1$ and (an adjusted version of) $\SF_2$.
The second sub-framework $\SF_2$ gets modified on the basis of the information contained in the extension(s) of $\SF_1$. Such alteration is tailored to account for the prior accepted and rejected, and undecided arguments.
In the literature~\cite{Baumann11,BuraglioDKW24}, this is achieved in a two-step procedure, by appealing to the notions of so-called \emph{reduct} and \emph{modification} of the sub-framework $SF_2$.
In the first step, the reduct takes care of the arguments in $SF_2$ that are rejected wrt.\ $E_1$ by deleting them, and modifies the links by projecting the part of the attack to $SF_2$.

\begin{definition}[Reduct]\label{def_reduct}	
		Let $(SF_1,SF_2,R_3)$ be a splitting for a SETAF $\SF$. We define the \emph{$(E_1,R_3)$-reduct} (or simply \emph{reduct}) of $\SF_2$ for some extension $E_1$ of $\SF_1$ as the SETAF $\SF\redEone_2=(A\redEone_2,R\redEone_2)$ where,
		\begin{align*}
        A\redEone_2=&\{a\in A_2\mid a \notin (E_1)_{R_3}^+\}\text{ and}\\
		R\redEone_2=&
		\{ (T,h)\in R_2\mid T\subseteq A\redEone_2,h\in A\redEone_2  \} \;\cup \\
	    &\{(T\cap A_2, h) \mid (T,h)\in R_3,\ T\cap A_1\subseteq E_1,\\
	    &\hspace{55pt} T\cap (E_1)^+_{R_3}=\emptyset, \ T \setminus A_1 \neq \emptyset,\ h\in A\redEone_2\}
		\end{align*}
\end{definition}
\begin{example}\label{ex:background split ctd}
We continue Example~\ref{ex:background split}.
Below, we depict the reduct wrt.\ $E_1=\{a\}$ on the left-hand side, the reduct wrt.\ $E_1'=\{c\}$ on the right-hand side;
the attacks and arguments not contained in the reducts are grayed out.%
            	
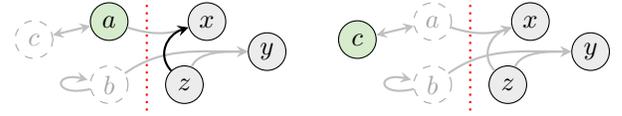
\begin{center}
			\begin{tikzpicture}[yscale=0.8,>=stealth]
            			\path
            (1,2.3)node[arg,fill=Green!25] (a){$a$}
            (1,1.25)node[argd] (b){$b$}
			(0,2)node[argd] (c){$c$}
			(2.3,2.3)node[arg] (x){$x$}
			(3.1,1.8)node[arg] (y){$y$}
			(2,1.25)node[arg] (z){$z$}
			;
			\path[thick,->]
            (c)edge[<->,gray!50](a)
			(b)edge[color=cyan,out=40,in=180,gray!50](y)
			(z)edge[color=cyan,out=70,in=180,looseness=0.9,gray!50](y)
			(a)edge[color=violet,out=-20,in=200,gray!50](x)
            (b) edge[loop left,gray!50] (b)			(z)edge[out=130,in=200,looseness=1.1](x)
            ;
			
			\draw [thick,dotted,red] (1.5,2.7) -- (1.5,0.8);

			\begin{scope}
			    [xshift=4.3cm]
            \path
            (1,2.3)node[argd] (a){$a$}
            (1,1.25)node[argd] (b){$b$}
			(0,2)node[arg,fill=Green!25] (c){$c$}
			(2.3,2.3)node[arg] (x){$x$}
			(3.1,1.8)node[arg] (y){$y$}
			(2,1.25)node[arg] (z){$z$}
			;
			\path[thick,->]
            (c)edge[<->,gray!50](a)
			(b)edge[color=cyan,out=40,in=180,gray!50](y)
			(z)edge[color=cyan,out=70,in=180,looseness=0.9,gray!50](y)
			(a)edge[color=violet,out=-20,in=200,gray!50](x)
            (b) edge[loop left,gray!50] (b)			(z)edge[out=130,in=200,looseness=1.1,gray!50](x)
            ;
			
			\draw [thick,dotted,red] (1.5,2.7) -- (1.5,0.8);
			\end{scope}
			\end{tikzpicture}
		\end{center}
We begin by computing the reduct wrt.\ $E_1$.
Since $R_2=\emptyset$, thus we only deal with the links in $R_3$.
For the attack $(\{a,z\},x)$, we have $\{a,z\}\cap A_1=\{a\}\subseteq E_1$, $\{a,z\}\setminus A_1\neq \emptyset$, and $\{a,z\}$ is not attacked by $E_1$, thus the attack $(z,x)$ is in the reduct.
The attack $(\{b,z\},y)$ does not satisfy these conditions, thus it is not part of the reduct.
The latter also holds for the reduct for $E_1'$; in addition, the attack $(\{a,z\},x)$ is defeated since $a$ is defeated.
\end{example}
As per the second step, the notion of \emph{modification} is introduced to account for possibly undecided arguments wrt.\ $E_1$. As these may affect other arguments in $SF_2$ via the links, a notion of undecided links is introduce to keep track of such external influences on arguments in (the reduct of) $SF_2$.

\begin{definition}[Undecided Links]\label{def: SETAF undec links}
Given a splitting $(\SF_1,\SF_2,R_3)$ for a SETAF $\SF$ and a set $E_1\subseteq A_1$ we define the \emph{set of undecided links} wrt.\ $E_1$ as:
\begin{align*}
U^{E_1}_{R_3}:= \; &   \{(T,h)\in R_3 \mid T\cap (E_1)^+_{R}=\emptyset, (T\cap  A_1)\not\subseteq E_1 \}.
\end{align*}	
\end{definition}
It holds that $(E_1)^+_{R}=(E_1)^+_{R_1\cup R_3}$ since $A_1$ is not attacked by links in $R_2$.
Next, we recall the \emph{modification} which accounts for the effects of the undecided links. 

\begin{definition}[Modification]\label{def_mod_SETAF}
	Let $(SF_1,SF_2,R_3)$ be a splitting for SETAF $\SF$ and $E_1\subseteq A_1$. Let $\SF\redEone_2=(A_2\redEone,R_2\redEone)$ be the $(E_1,R_3)$-reduct of $\SF_2$ and $U^{E_1}_{R_3}$ the set of undecided links wrt.\ $E_1$. By $mod^{E_1}_{R_3}(\SF\redEone_2):=\SF^{\star}_2=(A^{\star}_2,R^{\star}_2)$ we denote the \emph{$U^{E_1}_{R_3}$-modification} (or simply \emph{modification}) of $\SF\redEone_2$ s.t. $A^{\star}_2:=A\redEone_2$ and $R^{\star}_2$ is given by:
	\begin{align*}
		R\redEone_2\cup \{((T \cap A\redEone_2) \cup \{h\}, h) \mid (T,h)\in U^{E_1}_{R_3\clink}, \; h\in A_2\redEone\}.
	\end{align*}
\end{definition}
Taken together, when splitting up a SETAF, we either separate the tail of collective attacks or remove them altogether.%
\begin{example}
In Example~\ref{ex:background split ctd}, the link $(\{b,z\},y)$ is undecided for both $E_1$ and $E_1'$; resulting in the following modifications.%
            	\begin{center}
			\begin{tikzpicture}[yscale=0.8,>=stealth]
            			\path
            (1,2.3)node[argd] (a){$a$}
            (1,1.25)node[argd] (b){$b$}
			(0,2)node[argd] (c){$c$}
			(2.3,2.3)node[arg] (x){$x$}
			(3.1,1.8)node[arg] (y){$y$}
			(2,1.25)node[arg] (z){$z$}
			;
			\path[thick,->]
            (c)edge[<->,gray!50](a)
			(b)edge[color=cyan,out=40,in=180,gray!50](y)
			(a)edge[color=violet,out=-20,in=200,gray!50](x)
            (b) edge[loop left,gray!50] (b)			(z)edge[out=130,in=200,looseness=1.1](x)
            (z)edge[out=70,in=180,looseness=0.9](y)
            (y)edge[out=-80,in=180,looseness=5](y)
            ;
			
			\draw [thick,dotted,red] (1.5,2.7) -- (1.5,0.8);

			\begin{scope}
			    [xshift=4.3cm]
            \path
            (1,2.3)node[argd] (a){$a$}
            (1,1.25)node[argd] (b){$b$}
			(0,2)node[argd] (c){$c$}
			(2.3,2.3)node[arg] (x){$x$}
			(3.1,1.8)node[arg] (y){$y$}
			(2,1.25)node[arg] (z){$z$}
			;
			\path[thick,->]
            (c)edge[<->,gray!50](a)
			(b)edge[color=cyan,out=40,in=180,gray!50](y)
			(a)edge[color=violet,out=-20,in=200,gray!50](x)
            (b) edge[loop left,gray!50] (b)			(z)edge[out=130,in=200,looseness=1.1,gray!50](x)
            (z)edge[out=70,in=180,looseness=0.9](y)
            (y)edge[out=-80,in=180,looseness=5](y)
            ;
			
			\draw [thick,dotted,red] (1.5,2.7) -- (1.5,0.8);
			\end{scope}
			\end{tikzpicture}
		\end{center}
\end{example}
\citeauthor{BuraglioDKW24}~(\citeyear{BuraglioDKW24}) proved the correctness of the procedure wrt.\ the semantics under consideration.
They showed that (1) combining extensions of $\SF_1$ and (the altered version of) $\SF_2$ yields extensions of $\SF$, and (2) extensions of $\SF$ induce extensions in $\SF_1$ and (the altered version of) $\SF_2$, when restricted to the respective argument-sets.

    \begin{theorem}
    \label{theorem: splitting}
		Let $(\SF_1,\SF_2,R_3)$ be a splitting for a SETAF $\SF=(A,R)$ with $\SF_1=(A_1,R_1)$, $\SF_2=(A_2,R_2)$, and $\sigma \in \{\stb,\adm,\com,\prf,\grd\}$. Below, we let $\SF^{\star}_2=mod^{E_1}_{R_3}(\SF^{E_1}_2)$.
		\begin{enumerate}
			\item $E_1\!\in\! \sigma(\SF_1)$ and $E_2\!\in\! \sigma(\SF^{\star}_2)$ implies $E_1\cup E_2\in \sigma(\SF)$.
			\item If $E\!\in\! \sigma(\SF)$, then $E_1=E\cap A_1\!\in\! \sigma(\SF_1)$ and $E_2=E\cap A_2\!\in\!\sigma(SF^{\star}_2)$.
		\end{enumerate}
	\end{theorem}
\begin{example}
We use the result to compute the admissible extensions of $\SF$ from Example~\ref{ex:background split} incrementally. 
Both sets $E_1$ and $E_1'$ are admissible in $\SF_1$.
For $E_1$, $\{z\}$ is the only (non-empty) admissible extension of $\SF^\star_2$; therefore, $\{a,z\}\in \adm(\SF)$. For $E_1'$,  the sets $\{x\}$ and $\{x,z\}$ are admissible. Using the splitting theorem, we correctly deduce that $\{c,x\}$, $\{c,z\}$ and $\{c,x,z\}$ are admissible in $\SF$.
\end{example}

\section{Splitting over Collective Attacks}
\label{sec:slitting over collective attacks}

In this section, we introduce a splitting algorithm for BSAFs. In particular, we take into account the situation where we can split over collective attacks only. Notably, the presence of support relations forces significant changes with respect to the SETAF notions of reduct and modification.

\begin{definition}[Attack Splitting]
	\label{def:bsetaf splitting}
	Let $F=(A,R,S)$ be a BSAF, $F_1=(A_1,R_1,S_1)$ and $F_2=(A_2,R_2,S_2)$ two sub-frameworks of $F$ such that $A_1\cap A_2=\emptyset$, $A=A_1 \cup A_2$, $S=S_1\cup S_2$, and $R=R_1 \cup R_2 \cup R_3$ with $R_3\subseteq\{(T,h)\in R\mid T\cap A_1\neq\emptyset, T\subseteq A, h\in A_2\}$. We call the triple $(F_1,F_2,R_3)$ an \emph{attack splitting of $F$}. Moreover, we call $R_3$ the set of \emph{negative links} wrt. $(F_1,F_2,R_3)$.
\end{definition}

Ideally, we would like to transfer the results for SETAFs to the case of attack splitting for BSAFs.
There are, however, certain subtleties that need to be taken care of. We illustrate these in the following example.
\begin{example}
    Consider the BSAF $F$ below left and attack splitting $(F_1,F_2,R_3)$;  $F_1$ is left and $F_2$ right of the dashed line.
 Let $\{d\}=E_1\in \adm(F_1)$.
We illustrate in the Figure below right how the notions of SETAF reduct and modification apply to the present example.
    
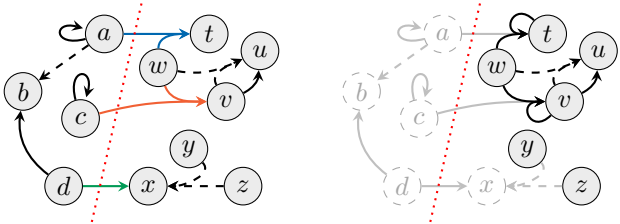
\begin{center}
\begin{tikzpicture}[scale=0.9,>=stealth]
            \path
            (0.7,2.5)node[arg] (a){$a$}
            (0.35,1.25) node[arg] (c) {$c$}
            (0.10,0.25)node[arg] (d){$d$}
            (-0.5,1.65) node[arg] (b) {$b$}
            (2.5,1.5) node[arg] (v) {$v$}
            (3,2.25) node[arg] (u){$u$}
            (1.5,2)node[arg] (w){$w$}
            (1.35,0.25)node[arg](x){$x$}
            (1.95,0.8)node[arg] (y) {$y$}
            (2.75,0.25) node[arg] (z) {$z$}
            (2.25,2.5) node [arg] (t) {$t$}
            ;
            
			\path[thick,->]
            (a)edge[loop left](a)
            (a)edge[dashed] (b)
            (d)edge[bend left] (b)
            (c)edge[loop above](c)
            (c)edge[out=15,in=180,RedOrange](v)
            (w)edge[out=-70,in=180,RedOrange](v)
            (a)edge[Blue](t)
            (w)edge[out=90,in=180,Blue](t)
            (v)edge[bend right] (u)
            (v)edge[dashed,out=120,in=205] (u)
            (w)edge[dashed,out=-10,in=205] (u)
            (d)edge[ForestGreen](x)
            (y)edge[dashed, out=-45,in=0] (x)
            (z)edge[dashed] (x)
            ;
			
			\draw [thick,dotted,red] (1.25,2.95) -- (0.5,-0.35);

            \begin{scope}[shift={(5,0)}]
			\path
            (0.7,2.5)node[argd] (a){$a$}
            (0.35,1.25) node[argd] (c) {$c$}
            (0.10,0.25)node[argd] (d){$d$}
            (-0.5,1.65) node[argd] (b) {$b$}
            (2.5,1.5) node[arg] (v) {$v$}
            (3,2.25) node[arg] (u){$u$}
            (1.5,2)node[arg] (w){$w$}
            (1.35,0.25)node[argd](x){$x$}
            (1.95,0.8)node[arg] (y) {$y$}
            (2.75,0.25) node[arg] (z) {$z$}
            (2.25,2.5) node [arg] (t) {$t$}
			;
			\path[thick,->]
            (a)edge[loop left,gray!50](a)
            (a)edge[dashed,gray!50] (b)
            (d)edge[bend left,gray!50] (b)
            (c)edge[loop above,gray!50](c)
            (c)edge[out=15,in=180,gray!50](v)
            (w)edge[out=-70,in=180](v)
            (a)edge[gray!50](t)
            (w)edge[out=90,in=180](t)
            (v)edge[bend right] (u)
            (v)edge[dashed,out=120,in=205] (u)
            (w)edge[dashed,out=-10,in=205] (u)
            (d)edge[gray!50](x)
            (y)edge[dashed, out=-45,in=0,gray!50] (x)
            (z)edge[dashed,gray!50] (x)
            (t)edge[out=135,in=180,looseness=5] (t)
            (v)edge[out=-135,in=180,looseness=5] (v)
            ;
			
			\draw [thick,dotted,red] (1.25,2.95) -- (0.5,-0.35);
			
			\end{scope}	
			
		\end{tikzpicture}	
			 \end{center}
First, we consider ${\color{Blue}(\{a,w\},t)}$. By definition, it is an undecided link, therefore we apply the modification to create the attack $(\{w,t\},t)$. However, this makes $\{t\}$ not admissible. Further, the same happens for the undecided link ${\color{RedOrange}(\{c,w\},v)}$. However, the presence of support $(\{v,w\},u)$ makes $\{v\}$ admissible again by defeating its closed attacker $\{v,w,u\}$. Hence, although $(\{c,w\},v)$ is truly an undecided link, the SETAF modification is not adequate. Finally, the link ${\color{ForestGreen}(d,x)}$ triggers the elimination of $x$ via the reduct, together with its support, leaving $\{y,z\}$ as an admissible extension of $F_2\redEone$. 
However, $\{d,y,z\}$ is not admissible in $F$ since it is not even closed.
\end{example}

In what follows, we adapt the splitting schema and show how to solve the aforementioned issues. 

\begin{example}
    \label{ex:problem with undec links}
    Let us have a closer look into a part of the BSAF from Example~\ref{ex:attack splitting first ex}.
    Consider the following attack splitting $(G_1,G_2,U_3)$ of the BSAF $G\subseteq F$ (displayed left).
        	
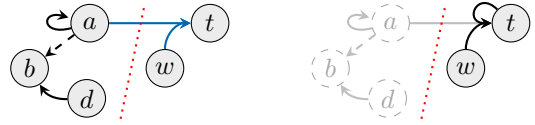
\begin{center}
			\begin{tikzpicture}[scale=0.8,>=stealth]
			\begin{scope}[shift={(0,0)}]

            \path
            (0.45,1.5)node[arg] (d){$d$}
            (-0.5,2)node[arg] (b){$b$}
			(0.5,2.75)node[arg] (a){$a$}
			(2.5,2.75)node[arg] (t){$t$}
			(1.75,2)node[arg] (w){$w$}
			;
			\path[thick,->]
            (d)edge[bend left](b)
			(a)edge[color=Blue](t)
			(w)edge[color=Blue,out=90,in=180](t)
            (a) edge[dashed] (b)
            (a) edge[loop left] (a)
            ;
			
			\draw [thick,dotted,red] (1.4,3) -- (1,1.1); -- (1,1);
			\end{scope}
			\begin{scope}[shift={(5,0)}]
			\path
            (0.45,1.5)node[argd] (d){$d$}
            (-0.5,2)node[argd] (b){$b$}
			(0.5,2.75)node[argd] (a){$a$}
			(2.5,2.75)node[arg] (t){$t$}
			(1.75,2)node[arg] (w){$w$}
			;
			\path[thick,->]
            (d)edge[bend left,gray!50](b)
			(a)edge[gray!50](t)
			(w)edge[out=90,in=180](t)
			(t)edge[out=135,in=180,looseness=5] (t)
            (a) edge[dashed,gray!50] (b)
            (a) edge[loop left,gray!50] (a)
            ;
			
			\draw [thick,dotted,red] (1.4,3) -- (1,1.1);
			\end{scope}
			\end{tikzpicture}
		\end{center}
    The set $E=\{d,t,w\}$ is admissible in $G$ since $d$ defends $t$ against the closed set $\{a,b,w\}$. Note that $\{a,w\}$ is not closed thus it is not necessary to defend $t$ against $(\{a,w\},t)$. We depict $G_2^\star$ wrt.\ $E_1=\{d\}$ above on the right-hand side.
    
    Observe that $\{d\}$ is admissible in $G_1$, but $\{t\}$ is not admissible in $G_2^\star$. The issue is that the attack on $t$ is indirectly defeated since its closure is attacked by $d$. This information, however, is lost when computing the modification $G_2^\star$. 
\end{example}
A negative link may be deemed as undecided under Definition \ref{def: SETAF undec links}, but become defeated due to some support in $F_1$, as the example shows. 
To account for this issue, we must deal with the negative links first in order to make explicit the information encoded in the support relation. 
Before we can construct the reduct, we alter $F$ and apply the \emph{link-closing} procedure, replacing all attacks $(T,a)\in R_3$ with $(cl(T),a)$.%
\begin{definition}[Closed Negative Links]
	\label{def:bsetaf splitting closed}
	Let $\BF=(A,R,S)$ be a BSAF, $\BF_1=(A_1,R_1,S_1)$ and $\BF_2=(A_2,R_2,S_2)$ two sub-frameworks of $\BF$ such that $(\BF_1,\BF_2,R_3)$ is an attack splitting of $F$. We define the set of \emph{closed negative links} as
    $$ R_3\clink:=\{(cl_S(T),h)\mid (T,h)\in R_3\}.$$ 
\end{definition}

As a result, we are able to adapt the SETAF splitting schema while avoiding aforementioned issues concerning the interplay between negative links and supports in $SF_1$.

\begin{example}
    Let $(G_1,G_2,U_3)$ be as in Example \ref{ex:problem with undec links}. We modify $U_3=\{(\{a,w\},t)\}$ as follows: first, we compute the closure of $\{a,w\}$ and replace the result with the source of the attack, obtaining $U_3\clink=\{(\{a,b,w\},t)\}$ (\textit{left}).
    Then, the reduct wrt.\ $E_1=\{d\}$ makes the attack defeated. Thus, the closed negative link $(\{a,b,w\},t)$ is not undecided. We depict the link-closure $(G_1,G_2,U_3\clink)$ and $F^\star_2$ below.
        	
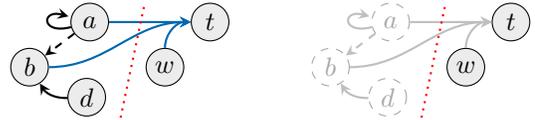
\begin{center}
			\begin{tikzpicture}[scale=0.8,>=stealth]
            \begin{scope}[shift={(0,0)}]

            \path
            (0.45,1.5)node[arg] (d){$d$}
            (-0.5,2)node[arg] (b){$b$}
			(0.5,2.75)node[arg] (a){$a$}
			(2.5,2.75)node[arg] (t){$t$}
			(1.75,2)node[arg] (w){$w$}
			;
			\path[thick,->]
            (d)edge[bend left](b)
			(a)edge[color=Blue] (t)
            (b)edge[color=Blue,out=0,in=180](t)
			(w)edge[color=Blue,out=90,in=180](t)
            (a) edge[dashed] (b)
            (a) edge[loop left] (a)
            ;
			
			\draw [thick,dotted,red] (1.4,3) -- (1,1.1);
			\end{scope}
			\begin{scope}[shift={(5,0)}]
			\path
            (0.45,1.5)node[argd] (d){$d$}
            (-0.5,2)node[argd] (b){$b$}
			(0.5,2.75)node[argd] (a){$a$}
			(2.5,2.75)node[arg] (t){$t$}
			(1.75,2)node[arg] (w){$w$}
			;
			\path[thick,->]
            (d)edge[bend left,gray!50](b)
			(a)edge[gray!50](t)
            (b)edge[gray!50,out=0,in=180](t)
			(w)edge[out=90,in=180,gray!50](t)
            (a) edge[dashed,gray!50] (b)
            (a) edge[loop left,gray!50] (a)
            ;
			
			\draw [thick,dotted,red] (1.4,3) -- (1,1.1);
			\end{scope}
			\end{tikzpicture}
		\end{center}
    Now, $\{w,t\}$ is admissible in $G_2^\star$ and we retrieve the admissible extension $\{d,w,t\}$ of $G$. 
\end{example}
Observe that closing the negative links preserves the semantics of the original framework. 
\begin{restatable}{proposition}{CompleteAttacksInR}\label{prop:close the attacks}
	Let $F=(A,R,S)$ be a BSAF, $(T,h)\in  R$ an attack and let $F'=(A,R',S)$ with $R'=(R\setminus\{(T,h)\})\cup\{(cl(T),h)\}$.
    Then $\sigma(F)=\sigma(F')$ for all $\sigma \in \{\stb,\adm,\com,\prf,\grd\}$.
\end{restatable}

After this preliminary step, we follow the splitting schema and appeal to the notions of reduct and modification. 
As observed in Example~\ref{ex:attack splitting first ex}, the SETAF reduct cannot be directly inherited in our context. 
The reason is that the removal of arguments attacked by an extension of $F_1$ imposed by the reduct may interfere with the closed sets in $F_2$. If defeated arguments in $F_2$ are deleted, all their supporting sets would wrongly be considered closed in the reduct of~$F_2$.%
\begin{example}\label{ex:attack splitting first ex}
We consider another problematic part of the BSAF $F$ from Example~\ref{ex:attack splitting first ex}.
Let $H\subseteq F$ with a splitting $(H_1,H_2,V_3)$ as below left.
Note that the link in $V_3$ is closed. 
The only preferred extension of $H_1$ is $E_1=\{d\}$. By simply applying the original notion of $(E,V_3)$-reduct, we obtain the framework $H\redEone_2$ (\textit{right}) which has one preferred extension $E_2=\{y,z\}$, from which we would erroneously derive that $E_1\cup E_2=\{d,y,z\}$ is a preferred extension of $H$. 
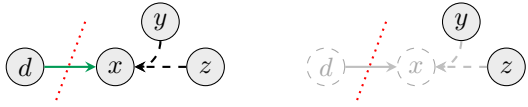
    	\begin{center}
			\begin{tikzpicture}[scale=0.8,>=stealth]
			\begin{scope}[shift={(0,0)}]
			\path
            (-1,2)node[arg] (d){$d$}
			(0.5,2)node[arg] (x){$x$}
			(1.25,2.75)node[arg] (y){$y$}
			(2,2)node[arg] (z){$z$}
			;
			\path[thick,->]
            (d)edge[color=ForestGreen] (x)
			(y)edge[dashed,out=-90,in=0](x)
			(z)edge[dashed](x);
			
			\draw [thick,dotted,red] (0,2.75) -- (-.5,1.4);

			\end{scope}
			\begin{scope}[shift={(5,0)}]
			\path
            (-1,2)node[argd] (d){$d$}
			(0.5,2)node[argd] (x){$x$}
			(1.25,2.75)node[arg] (y){$y$}
			(2,2)node[arg] (z){$z$}
			;
			\path[thick,->]
            (d)edge[gray!50](x)
			(y)edge[dashed,out=-90,in=0,gray!50](x)
			(z)edge[dashed,gray!50](x);
			
			\draw [thick,dotted,red] (0,2.75) -- (-.5,1.4);

			\end{scope}
			\end{tikzpicture}
		\end{center}
\end{example}

To account for this issue, we change the reduct as follows: instead of removing arguments defeated by $E_1$, we let them be attacked by the empty set. This syntactical difference preserves the defeated status of such arguments, while preserving its supports in $F_2$. 
In fact, for each defeated argument $h$ which is supported by some set of arguments $T$ in $F_2$, such an addition works as constraint: not all elements in $T$ can be accepted, otherwise, $h$ should be accepted as well. 

\begin{definition}[R-reduct]\label{def:R-reduct}
        Let $(F_1,F_2,R_3)$ be an attack splitting for a BSAF $F$; let $R\clink_3$ denote the closed negative links.
        We define the \emph{$(E_1,R_3\clink)$-reduct} (or simply \emph{R-reduct}) of $F_2$ for some extension $E_1$ of $F_1$ as the BSAF $F\redEone_2:=(A_2,R\redEone_2,S_2)$ with:
\begin{align*}
	R\redEone_2:={}&
    R_2\cup 
	\{(T\cap A_2, h) \mid (T,h)\in R_3\clink, \\
    &\hspace{35pt} T\cap (E_1)^+_{R_1\cup R_3\clink}=\emptyset,\ T\cap A_1\subseteq E_1\} 
\end{align*}
\end{definition}
Note that in the updated version of the reduct, no arguments are deleted. On top of that, negative links are projected to their (possibly empty) part in $F_2$.

\begin{example}\label{ex:SS-attack-splitting}
Let $(H_1,H_2,V_3)$ (\textit{left}) be as in Example~\ref{ex:attack splitting first ex}.
The $(E_1,V_3\clink)$-reduct consists of the framework $H\redEone_2$ (\textit{right}) which has two preferred extensions $E_2=\{y\}$ and $E'_2=\{z\}$, from which we now obtain the two preferred extensions of $H$, i.e.\ $E=\{a,y\}$ and $E'=\{a,z\}$. 

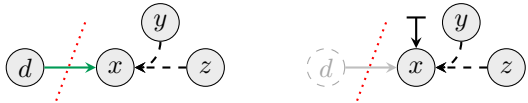
    	\begin{center}
			\begin{tikzpicture}[scale=0.8,>=stealth]
			\begin{scope}[shift={(0,0)}]
            \path
            (-1,2)node[arg] (d){$d$}
			(0.5,2)node[arg] (x){$x$}
			(1.25,2.75)node[arg] (y){$y$}
			(2,2)node[arg] (z){$z$}
			;
			\path[thick,->]
            (d)edge[color=ForestGreen] (x)
			(y)edge[dashed,out=-90,in=0](x)
			(z)edge[dashed](x);
			
			\draw [thick,dotted,red] (0,2.75) -- (-.5,1.4);
			
			\end{scope}
			\begin{scope}[shift={(5,0)}]
			\path
            (-1,2)node[argd] (d){$d$}
			(0.5,2)node[arg] (x){$x$}
			(1.25,2.75)node[arg] (y){$y$}
			(2,2)node[arg] (z){$z$}
			;
			\path[thick,->]
            (d)edge[gray!50](x)
			(y)edge[dashed,out=-90,in=0](x)
			(z)edge[dashed](x)
            (0.5,2.85)edge[|->](x)
            ;
			
			\draw [thick,dotted,red] (0,2.75) -- (-.5,1.4);

			 \end{scope}
			\end{tikzpicture}
		\end{center}
\end{example}

It remains to state the final modification to deal with the \emph{closed} undecided links in $R_3\clink$. 
For a given attack splitting $(F_1,F_2,R_3)$ for a BSAF $F$ and a set $E_1$ of arguments in $F_1$,
we use the notion of undecided links as defined for SETAFs (see Definition~\ref{def: SETAF undec links}, but applied to the BSAF splitting $(F_1,F_2,R\clink_3)$, i.e., after closing the links in $R_3$. 
We are now almost ready to apply the modification relative to $U^{E_1}_{R_3\clink}$. 
There is, however, a final subtlety that we need to deal with.{}
\begin{example}\label{ex:wrong mod}
Let us zoom into another part of the BSAF $F$ from Example~\ref{ex:attack splitting first ex}, as displayed below. 
We consider the attack splitting $(I_1,I_2,W_3)$ for the BSAF $I\subseteq F$ (\textit{left}), with $W_3\clink=W_3=\{(\{c,w\},v)\}$, i.e., the attack is closed. 
Since $E_1=\emptyset$ is an admissible extension of $I_1$, we know that the attack $(\{c,w\},v)$ is an undecided link in $W_3\clink$. Via the standard SETAF $U^{E_1}_{W_3\clink}$-modification, we would add the set-self attack $(\{w,v\},v)$ to obtain $I_2^\star$ (\textit{right}).%
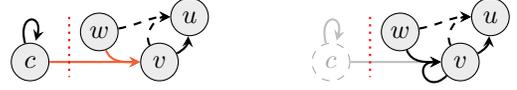
\begin{center}
\begin{tikzpicture}[scale=0.8,>=stealth]
			\begin{scope}[shift={(0,0)}]
            \path
            (0.35,1.5) node[arg] (c) {$c$}
            (2.5,1.5) node[arg] (v) {$v$}
            (3,2.25) node[arg] (u){$u$}
            (1.5,2)node[arg] (w){$w$}
            ;
            
			\path[thick,->]
            (c)edge[loop above](c)
            (c)edge[color=RedOrange] (v)
            (w)edge[color=RedOrange,out=-70,in=180](v)
            (v)edge[bend right] (u)
            (v)edge[dashed,out=120,in=180] (u)
            (w)edge[dashed,out=10,in=180] (u)
            ;
			
			\draw [thick,dotted,red] (1,2.25) -- (1,1.25);
			\end{scope}
			\begin{scope}
			    [shift={(5,0)}]
            \path
            (0.35,1.5) node[argd] (c) {$c$}
            (2.5,1.5) node[arg] (v) {$v$}
            (3,2.25) node[arg] (u){$u$}
            (1.5,2)node[arg] (w){$w$}
            ;
            
			\path[thick,->]
            (c)edge[loop above,gray!50](c)
            (c)edge[gray!50](v)
            (w)edge[out=-70,in=180](v)
            (v)edge[bend right] (u)
            (v)edge[dashed,out=120,in=180] (u)
            (w)edge[dashed,out=10,in=180] (u)
            (v)edge[out=-135,in=180,looseness=5] (v)
            ;
			
			\draw [thick,dotted,red] (1,2.25) -- (1,1.25);
			\end{scope}
		\end{tikzpicture}	
        \end{center}
As previously noticed, the argument $v$ now defend itself against its only \emph{closed} attacker $\{w,v,u\}$ by attacking $u$. Hence, we erroneously conclude that $\{v\}\in \adm(I)$. 
\end{example}

A new version of the $U^{E_1}_{R_3\clink}$-modification is then needed to deal with scenarios of this kind. Instead of using the target $v$ of the attack to build a new self-attack in $F_2$, we introduce a self-attacking dummy argument $*_0$ that jointly attacks $v$ together with the right-most part of the undecided link. Formally, the BSAF  $U^{E_1}_{R_3\clink}$-modification is defined as follows.

\begin{definition}[Modification]\label{def_mod}
    Let $(F_1,F_2,R_3)$ be an attack splitting for a BSAF $F$; let $R\clink_3$ denote the closed negative links.
    Let $E_1$ an extension of $F_1$. Take $F\redEone_2$ as the $(E_1,R_3\clink)$-reduct of $F_2$ and $U^{E_1}_{R_3\clink}$ as the set of undecided links wrt.\ $E_1$. By $mod^{E_1}_{R_3\clink}(F\redEone_2):=F^{\star}_2=(A^{\star}_2,R^{\star}_2,S^{\star}_2)$ we denote the \emph{$U^{E_1}_{R_3\clink}$-modification} (or simply \emph{modification}) of $F\redEone_2=(A_2,R_2\redEone,S_2)$ s.t.\ $A^{\star}_2:=A_2\cup \{*_0\}$, $S^{\star}_2:=S_2$ and $R^{\star}_2$ is given by:
	\begin{align*}
		R\redEone_2\cup \{(*_0,*_0)\} \cup \{((T \cap A_2) \cup \{*_0\}, h) \mid (T,h)\in U^{E_1}_{R_3}\}.
	\end{align*}
\end{definition}

\begin{example}
Let $(I_1,I_2,W_3)$ as in Example~\ref{ex:wrong mod} (left).
The BSAF $U^{E_1}_{W_3\clink}$-modification of $I_2$ is $I_2^\star$ (\textit{right}); here, $v$ does not defend itself against $\{w,*_0\}$ and is therefore not admissible.

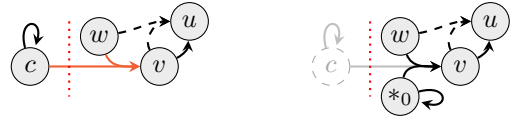
\begin{center}
\begin{tikzpicture}[scale=0.8,>=stealth]
			\begin{scope}[shift={(0,0)}]
            \path
            (0.35,1.5) node[arg] (c) {$c$}
            (2.5,1.5) node[arg] (v) {$v$}
            (3,2.25) node[arg] (u){$u$}
            (1.5,2)node[arg] (w){$w$}
            ;
            
			\path[thick,->]
            (c)edge[loop above](c)
            (c)edge[color=RedOrange] (v)
            (w)edge[color=RedOrange,out=-70,in=180](v)
            (v)edge[bend right] (u)
            (v)edge[dashed,out=120,in=180] (u)
            (w)edge[dashed,out=10,in=180] (u)
            ;
			
			\draw [thick,dotted,red] (1,2.25) -- (1,1);
			\end{scope}
			\begin{scope}
			    [shift={(5,0)}]
            \path
            (0.35,1.5) node[argd] (c) {$c$}
            (2.5,1.5) node[arg] (v) {$v$}
            (3,2.25) node[arg] (u){$u$}
            (1.5,2)node[arg] (w){$w$}
            (1.5,1)node[arg] (*) {$*_0$}
            ;
            
			\path[thick,->]
            (c)edge[loop above,gray!50](c)
            (c)edge[gray!50](v)
            (w)edge[out=-70,in=180](v)
            (v)edge[bend right] (u)
            (v)edge[dashed,out=120,in=180] (u)
            (w)edge[dashed,out=10,in=180] (u)
            (*)edge[out=80,in=180] (v)
            (*)edge[loop right] (*)
            ;
			
			\draw [thick,dotted,red] (1,2.25) -- (1,1);
			\end{scope}
		\end{tikzpicture}	
        \end{center}
\end{example}

We are now in the position to state the overall splitting procedure. Given an attack splitting $(F_1,F_2,R_3)$ for a BSAF $F$,
we proceed as follows.
\begin{enumerate}
    \item We close the negative links following Definition~\ref{def:bsetaf splitting closed}. We obtain a set of new negative closed links $R_3\clink$. 
    \item We apply our BSAF reduct following Definition~\ref{def:R-reduct}. As a result, we obtain two separate frameworks $F_1$ and $F_2\redEone$. They have the same arguments $A_1$ and $A_2$ as before; all arguments in $F_2$ that are defeated by attacks in $F_1$ are attacked by the empty set. 
    \item To account for undecided links, we apply the modification from Definition \ref{def_mod} and preserve undecidedness in $F_2$. The BSAF $F_2$ has been extended by the self-attacker $*_0$.
\end{enumerate}

After these considerations, we are now in possession of all the ingredients to prove the correctness of our splitting schema for BSAFs. 
As a first step towards it, we ensure that the adjusted version of R-reduct is not in conflict with the notion of closed set of arguments. 

\begin{restatable}{proposition}{Closure}\label{prop:closure}
    Let $(F_1,F_2,R_3)$ be an attack splitting for a BSAF $F=(A,R,S)$ with $F_1=(A_1,R_1,S_1)$, $F_2=(A_2,R_2,S_2)$; let $R\clink_3$ denote the closed negative links.
    Further, let $E_1\subseteq A_1$, $F\redEone_2$ the $(E_1,R_3\clink)$-reduct of $F_2$, $F_2^{\star}=mod^{E_1}_{R_3\clink}(F\redEone_2)$, and $E_2\subseteq A^{\star}_2$. 
	\begin{enumerate}
	    \item If $E_1=cl_{F_1}(E_1)$ and $E_2=cl_{F^{\star}_2}(E_2)$, then $E=(E_1\cup E_2)\setminus \{*_0\}=cl_F(E)$.
		\item If $E=cl_F(E)$ for some $E\subseteq A$, then $E\cap A_1=E_1=cl_{F_1}(E_1)$ and $E\cap A_2=E_2=cl_{F^{\star}_2}(E_2)$.
	\end{enumerate}
\end{restatable}

From this, we get an attack splitting theorem for BSAFs.

\begin{restatable}{theorem}{AttSplit}\label{thrm:AttSplit}
Let $(F_1,F_2,R_3)$ be an attack splitting for a BSAF $F=(A,R,S)$ with $F_1=(A_1,R_1,S_1)$, $F_2=(A_2,R_2,S_2)$; let $R\clink_3$ denote the closed negative links.
    Let $\sigma \in \{\stb,\adm,\com,\prf\}$.
    Further, let $F_2^{\star}=mod^{E_1}_{R_3\clink}(F\redEone_2)$ where $F\redEone_2$ is the $R$-reduct wrt. $E_1$. 
    \begin{enumerate}
        \item If $E_1\in \sigma(F_1)$ and $E_2\in \sigma(F_2^{\star})$, then $E_1\cup E_2\in \sigma(F)$. 
		\item If $E\in \sigma(F)$, then $E_1=E\cap A_1\in \sigma(F_1)$ and $E\cap A_2\in \sigma(F_2^{\star})$.
    \end{enumerate}
\end{restatable}

For grounded semantics, only the first direction of the splitting theorem can be guaranteed, due to a clash between the minimality and completeness conditions. 

\begin{example}
Consider the BSAF $F$ with splitting $(F_1,F_2,R_3)$ as below. The empty set defends $d$ in $F$, so it is not complete. However, $\{d\}$ does not defend $e$, so that $E=\{a,d,e\}$ is the only grounded extension of $F$. Note, however, that the empty set is grounded in $F_1$, making it impossible to reconstruct $E$.

    \begin{center}
			\begin{tikzpicture}[scale=0.8,>=stealth]
			\begin{scope}[shift={(0,0)}]
			\path
            (0.5,3.5)node[arg] (a){$a$}
			(-1,3.5)node[arg] (b){$b$}
			(3.25,3.5)node[arg] (y){$e$}
			(4.5,3.5)node[arg] (x){$d$}
            (2,3.5)node[arg] (w) {$c$}
			;
			\path[thick,->]
            (a)edge(b)
            (b)edge[bend left] (a)
			(x)edge[dashed](y)
			(w)edge(y)
			(a)edge(w);
            
			\draw [thick,dotted,red] (1.25,3.75) -- (1.25,3.15);

			\end{scope}
			\end{tikzpicture}
		\end{center}
\end{example}

\begin{restatable}{theorem}{GrdAttSplit}\label{thm:GrdAttSplit}
    Let $(F_1,F_2,R_3)$ be an attack splitting for a BSAF $F=(A,R,S)$ and $R_3\clink$ the set of closed negative links. Let $F_2^{\star}=mod^{E_1}_{R_3\clink}(F\redEone_2)$ where $F\redEone_2$ is the $(E_1,R_3\clink)$-reduct. If $E_1\in \grd(F_1)$ and $E_2\in \grd(F^\star_2)$, then $E_1\cup E_2\in \grd(F)$.  
\end{restatable}

\section{Splitting over Collective Supports}\label{sec:CollSupp}

We continue our investigation of how splitting translates from SETAFs to full BSAFs by focusing on the support relation. Analogously to attack splitting, we first consider support splitting in an \emph{isolated setting}; i.e., we split a BSAF $F$ with sub-BSAFs $F_1$ and $F_2$ which share \emph{only positive links}.

In the case of attack splitting, our schema relies on the subsequent evaluation of $F_1$ and (a modification of) $F_2$ for a given BSAF $F$, unpacking the evaluation of the entire framework in an incremental fashion.
Our goal is to find appropriate modifications that enable the modular computation of the semantics if $F_1$ and $F_2$ share positive links.
However, a closer investigation of the effects of the supporting links reveals a fundamental issue.
\begin{example}
\label{ex:why forward support fails}
    Consider the BSAF $F$ with the subframeworks $\BF_1$ (left) and $\BF_2$ (right) and the shared support $(b,c)$ below.%
    
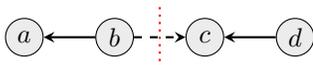
\begin{center}
		\begin{tikzpicture}[scale=0.8,>=stealth]
            \path
    		(0.5,2)node[arg] (a){$a$}
    		(2,2)node[arg] (b){$b$}
    		(3.5,2)node[arg] (c){$c$}
            (5,2)node[arg] (cp){$d$}
    		;
    		\path[thick,->]
    		(b)edge(a)
    		(b)edge[dashed](c)
            (cp)edge(c)
            ;
			\draw [thick,dotted,red] (2.75,2.5) -- (2.75,1.5);
		\end{tikzpicture}
	\end{center}
    In $F$, the set $\{a,d\}$ is admissible since $d$ defends $a$ against the closed attacker $\{b,c\}$. 
    In $F_1$, however, the argument $a$ is defeated by $b$. The acceptability of $a$ in $F_1$ cannot be accomplished without modifying $F_1$; this is, however, not foreseen in the traditional approach to splitting. 
\end{example}
The underlying issue is that sets in $F_2$ are not closed; 
as we have demonstrated in the above example, the acceptance status of arguments in $F_1$ may thus depend on those of the arguments in $F_2$. 
Intuitively, the support in Example~\ref{ex:why forward support fails} has the effect of directing an attack from $F_2$ to $F_1$.

If we consider shared supports directed in the opposite direction, where the head $h$ of each positive link $(T,h)$ is contained in $F_1$, all sets in $F_1$ are closed.
Therefore, this problem cannot occur in this case. 
We will therefore focus our efforts on \emph{backward} support splitting.

\begin{definition}[Support Splitting]
    \label{def:bsetaf supp splitting backwards}
	Let $\BF=(A,R,S)$ be a BSAF, $\BF_1=(A_1,R_1,S_1)$ and $\BF_2=(A_2,R_2,S_2)$ two sub-frameworks of $\BF$ such that $A_1\cap A_2=\emptyset$, $A=A_1 \cup A_2$, $R=R_1 \cup R_2$, and $S=S_1\cup S_2 \cup  S_3$ with $S_3\subseteq\{(T,h)\in S\mid T\cap A_2\neq\emptyset, T\subseteq A, h\in A_1\}$. We call the triple $(\BF_1,\BF_2,S_3)$ \emph{support splitting} of $\BF$. Moreover, we call $S_3$ the set of \emph{positive links} wrt. $(\BF_1,\BF_2,S_3)$.
\end{definition}

\begin{example}\label{ex:support split 1}
We consider a similar example as before; now, the support is directed in the other direction. 
The support splitting $(\BF_1,\BF_2,S_3)$ for the BSAF $F$ below consists of $F_1$ (left), $F_2$ (right) and a single shared support $S_3=\{(c,b)\}$. 
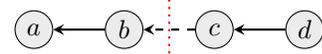
    \begin{center}
		\begin{tikzpicture}[scale=0.8,>=stealth]
            \path
    		(0.5,2)node[arg] (a){$a$}
    		(2,2)node[arg] (b){$b$}
    		(3.5,2)node[arg] (c){$c$}
            (5,2)node[arg] (cp){$d$}
    		;
    		\path[thick,->]
    		(b)edge(a)
    		(c)edge[dashed](b)
            (cp)edge(c)
            ;
			\draw [thick,dotted,red] (2.75,2.5) -- (2.75,1.5);
		\end{tikzpicture}
	\end{center}
Note that $d$ defends $a$ against the closed attacker $\{b,c\}$ as before; 
however, in the given BSAF, the set $\{b\}$ is closed as well. Since $b$ is not attacked, $a$ cannot be defended.
\end{example}
Our goal is to develop a suitable splitting procedure analogous to the case for splitting attacks. We consider a $\sigma$-extension of $F_1$ and modify $F_2$ so that, for each $\sigma$-extension $E_2$ of (the modified version of) $F_2$, it holds that $E_1\cup E_2$ is a $\sigma$-extension of $F$. Likewise, each $\sigma$-extension $E$ of $F$ should satisfy $E\cap A_1$ and $E\cap A_2$ both are $\sigma$-extensions of $F_1$ and (the modified version of) $F_2$, respectively. 

Given the BSAF from Example~\ref{ex:support split 1} and the set $E_1=\{b\}$, 
we can perform the split by simply removing the shared support. 
The split correctly deduces that $\{b\}$ and $\{b,d\}$ are admissible in $F$.
Such an easy solution of simply removing the shared support is, however, not always possible.%
\begin{example}\label{ex:support needs constrains}
We consider the BSAF $F$ below, with $F_1$ and $F_2$ left resp.\ right of the dotted line and $S_3\!=\!\{(\{c_1,c_2\},b)\}$.
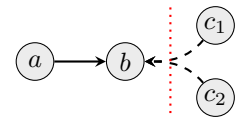
\begin{center}
            \begin{tikzpicture}[scale=0.8,>=stealth]
			\path
            (-1,2)node[arg] (a){$a$}
			(0.5,2)node[arg] (b){$b$}
			(2,2.6)node[arg] (cone){$c_1$}
			(2,1.4)node[arg] (ctwo){$c_2$}
			;
			\path[thick,->]
            (a) edge (b)
			(cone)edge[dashed, out=-130,in=0](b)
			(ctwo)edge[dashed, out=130,in=0](b)
			;
			
			\draw [thick,dotted,red] (1.25,2.9) -- (1.25,1.1);
			\end{tikzpicture}
\end{center}
Let $E_1=\{a\}$. Since $b$ is defeated, not both $c_1$ and $c_2$ can be accepted in $F_2$; otherwise, $b$ would be accepted in $F$. 
\end{example}
The example above shows that defeating the head of a shared support induces a constraint.
There are two natural options to encode constraints in BSAFs; 
and we will make use of both of them.
Given $(T,h)\in S_3$, the first option is to add a new argument $*_1$ which is attacked by the empty set and supported by $T\cap A_2$ (\emph{type-1-constraint}); the second option is to add a new argument $*_2$ which is attacked by $(T\cap A_2)\cup \{*_2\}$ and supported by $T\cap A_2$ (\emph{type-2-constraint}).%
\begin{example}\label{ex:support splitting constraints}
Consider the BSAF from Example~\ref{ex:support needs constrains}. We can prevent the joint acceptance of $c_1$ and $c_2$ in $F_2$ by applying one of the following modifications. On the left-hand side, we add an argument that is attacked by the empty set ($*_1$), on the right-hand side, we add an argument ($*_2$) which is the head of the self-attacking set.
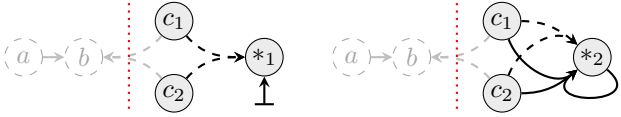
\begin{center}
    \begin{tikzpicture}[scale=0.8,>=stealth]
			\path
            (-0.5,2)node[argd] (a){$a$}
			(0.5,2)node[argd] (b){$b$}
			(2,2.6)node[arg] (cone){$c_1$}
			(2,1.4)node[arg] (ctwo){$c_2$}
			(3.5,2)node[arg] (d){$*_1$}
			;
			\path[thick,->]
            (cone) edge[dashed,out=-50,in=180] (d)
            (ctwo) edge[dashed,out=50,in=180] (d)
            (a) edge[gray!50] (b)
			(cone)edge[dashed, out=-130,in=0,gray!50](b)
			(ctwo)edge[dashed, out=130,in=0,gray!50](b)
            (3.5,1.2) edge[|->] (d)
			;
			
			\draw [thick,dotted,red] (1.25,2.9) -- (1.25,1.1);

            \begin{scope}
                [xshift=5.45cm]
			\path
            (-0.5,2)node[argd] (a){$a$}
			(0.5,2)node[argd] (b){$b$}
			(2,2.6)node[arg] (cone){$c_1$}
			(2,1.4)node[arg] (ctwo){$c_2$}
			(3.5,2)node[arg] (d){$*_2$}
			;
			\path[thick,->]
            (cone) edge[dashed,out=0,in=140] (d)
            (ctwo) edge[dashed,out=70,in=140] (d)
            (cone) edge[out=-70,in=220] (d)
            (ctwo) edge[out=0,in=220] (d)
            (a) edge[gray!50] (b)
			(cone)edge[dashed, out=-130,in=0,gray!50](b)
			(ctwo)edge[dashed, out=130,in=0,gray!50](b)
            (d) edge[in=220,out=-40,looseness=5] (d)
			;
			
			\draw [thick,dotted,red] (1.25,2.9) -- (1.25,1.1);
            \end{scope}
			\end{tikzpicture}
\end{center}
We observe that the type-1-constraint (left) defeats all attacks originating from (supersets of) $T=\{c_1,c_2\}$. This is because the closure of $T$ is attacked by the empty set.
\end{example}
We break down which constraint is appropriate in which situation.
First, we observe that a shared link $(T,h)$ induces a constraint whenever all arguments in $T\cap A_1$ are contained in the selected extension $E_1$, but $h\notin E_1$.
\begin{definition}
    Let $(F_1,F_2,S_3)$ be a support splitting for BSAF $F$, where $F_1=\tuple{A_1,R_1,S_1}$. Let $E_1\subseteq A_1$. By 
    \begin{align*}
        \mathcal{T}_{S_3}(E_1):=&\, \{T\subseteq A\mid \exists h\in A_1:(T,h)\in S_3, \\
        & \hspace{30pt} T\cap A_1\subseteq E_1, \ h\notin E_1\}
    \end{align*}
    we denote the collection of sets of arguments that are \emph{support-incompatible} with respect to $E_1$ and $S_3$.
\end{definition}
As observed in Example~\ref{ex:support splitting constraints}, 
adding a type-1-constraint for a shared link $(T,h)$ by adding $*_1$ which is attacked by the empty set and a support from $T\cap A_2$ implies that all attacks originating from $T\cap A_2$ are defeated. 
We can apply this construction therefore only in the very restricted setting in which the set $T$ is already defeated in $F_1$. 
We also require $T\cap A_1=\emptyset$ because in all other cases, 
no outgoing attacks from $T$ exist (recall that $F_1$ and $F_2$ do not share attacks). 
\begin{definition}
    Let $(F_1,F_2,S_3)$ be a support splitting for BSAF $F$, where $F_1=\tuple{A_1,R_1,S_1}$. Let $E_1\subseteq A_1$. By 
    \begin{align*}
        \mathcal{D}_{S_3}(E_1):=&\, \{T\subseteq A\mid T\in \mathcal{T}_{S_3}(E_1),\ T\cap A_1=\emptyset,\ \\
        & \hspace{50pt}  \cl_F(T)\cap (E_1)^+_R\neq \emptyset \}\
    \end{align*}
    we denote the collection of sets of arguments that are \emph{closure-defeated} with respect to $E_1$ and $S_3$.
\end{definition}
We are ready to state our first modification. 
\begin{definition}[Type-1-Modification]

        Let $(F_1,F_2,S_3)$ be a support splitting for a BSAF $F$, let $E_1$ be an extension of $F_1$ and let $\mathcal{D}_{S_3}(E_1)$ denote all sets which are closure-defeated wrt.\ $E_1$ and $S_3$. 
        We define the \emph{type-1-modification} of $F_2$ as the BSAF $F^{C_1}_2:=F_2$ if $\mathcal{D}_{S_3}(E_1)=\emptyset$ and, otherwise, 
        $F^{C_1}_2:=(A_2\cup \{*_1\},R_2^{C_1},S^{C_1}_2)$ with
\begin{align*}
	R^{C_1}_2:={}&
	R_2  \,\cup \{(\emptyset,*_1)\}\\
    S^{C_1}_2:={}&
    S_2 \;\cup \{(T\cap A_2, *_1) \mid T\in \mathcal{D}_{S_3}(E_1)\}.
\end{align*}
\end{definition}
\begin{example}
    Consider the BSAF $F$ with $\BF_1$ (left) and $\BF_2$ (right) and the shared support $(c,b)$ below. Let $E_1=\{a\}$.%
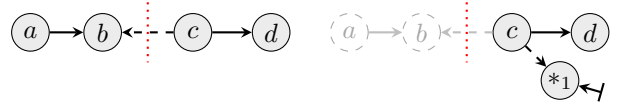
    \begin{center}
		\begin{tikzpicture}[scale=0.8,>=stealth]
            \path
    		(0.8,2)node[arg] (a){$a$}
    		(2,2)node[arg] (b){$b$}
    		(3.5,2)node[arg] (c){$c$}
            (4.8,2)node[arg] (d){$d$}
    		;
    		\path[thick,->]
    		(a)edge(b)
    		(c)edge[dashed](b)
            (c)edge(d)
            ;
			\draw [thick,dotted,red] (2.75,2.5) -- (2.75,1.5);

            \begin{scope}
                [xshift=5.3cm]
                \path
    		(0.8,2)node[argd] (a){$a$}
    		(2,2)node[argd] (b){$b$}
    		(3.5,2)node[arg] (c){$c$}
            (4.8,2)node[arg] (d){$d$}
            (4.3,1.2)node[arg] (x){$*_1$}
    		;
    		\path[thick,->]
    		(a)edge[gray!50](b)
    		(c)edge[dashed,gray!50](b)
            (c) edge[dashed] (x)
            (c)edge(d)
            (5,1) edge[|->] (x) 
            ;
			\draw [thick,dotted,red] (2.75,2.5) -- (2.75,1.5);
            \end{scope}
		\end{tikzpicture}
	\end{center}
    In $F$, $\cl_F(\{c\})=\{c,b\}$, is attacked, thus, $d$ is defeated. 
    This is correctly captured by the type-1-modification $F^{C_1}_2$: the set $\cl_{F^{C_1}_2}(\{c\})=\{c,*_1\}$ is attacked by the empty set.
\end{example}
We consider an example where this construction fails.
\begin{example}\label{ex:type-2-modification}
    Consider the BSAF $F$ with $\BF_1$ (left) and $\BF_2$ (right) and the shared support $(\{c,e\},b)$ below. Let $E_1\in \{\{e\},\{a,e\}\}$, then $\{c,e\}\in \mathcal{T}_{S_3}(E_1)$.
    
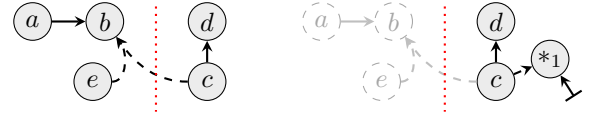
\begin{center}
		\begin{tikzpicture}[xscale=0.8,>=stealth]
            \path
    		(0.8,2)node[arg] (a){$a$}
    		(2,2)node[arg] (b){$b$}
    		(3.7,1.2)node[arg] (c){$c$}
            (3.7,2)node[arg] (d){$d$}
            (1.8,1.2)node[arg] (e){$e$}
    		;
    		\path[thick,->]
    		(a)edge(b)
    		(c)edge[dashed, in=-50,out=180](b)
    		(e)edge[dashed, in=-50,out=0,looseness=1.1](b)
            (c)edge(d)
            ;
			\draw [thick,dotted,red] (2.85,2.2) -- (2.85,0.8);

            \begin{scope}
                [xshift=4.8cm]
                \path
    		(0.8,2)node[argd] (a){$a$}
    		(2,2)node[argd] (b){$b$}
    		(3.7,1.2)node[arg] (c){$c$}
            (3.7,2)node[arg] (d){$d$}
            (4.6,1.5)node[arg] (x){$*_1$}
            (1.8,1.2)node[argd] (e){$e$}
    		;
    		\path[thick,->]
    		(a)edge[gray!50](b)
    		(c)edge[dashed, in=-50,out=180,gray!50](b)
    		(e)edge[dashed, in=-50,out=0,looseness=1.1,gray!50](b)
            (c)edge(d)
            (c) edge[dashed] (x)
            (5,1) edge[|->] (x) 
            ;
			\draw [thick,dotted,red] (2.85,2.2) -- (2.85,0.8);
            \end{scope}
		\end{tikzpicture}
	\end{center}
    Neither $\{e,d\}$ nor $\{a,e,d\}$ are admissible in $F$. However, 
    in $F^{C_1}_2$, argument $d$ is defended since $\cl_{F^{C_1}_2}(c)$ is attacked.
\end{example}
We thus require constraints of the second type in this case.%
\begin{definition}[Type-2-Modification]
    	Let $(F_1,F_2,S_3)$ be a support splitting for a BSAF $F$, let $E_1$ be an extension of $F_1$ and let $\mathcal{D}_{S_3}(E_1)$ and $\mathcal{T}_{S_3}(E_1)$ denote all sets which are closure-defeated resp.\ support-incompatible wrt.\ $E_1$ and $S_3$. 
        Let $\mathcal{T}^{\text{-}\mathcal{D}}_{S_3}(E_1):= \mathcal{T}_{S_3}(E_1)\setminus \mathcal{D}_{S_3}(E_1)$
        We define the \emph{type-2-modification} of $F_2$ as the BSAF $F^{C_2}_2:=F_2$ if $\mathcal{T}^{\text{-}\mathcal{D}}_{S_3}(E_1)=\emptyset$ and, otherwise,  
        $F^{C_2}_2:=(A_2\cup \{*_2\},R_2^{C_2},S^{C_2}_2)$ with
\begin{align*}
	R^{C_2}_2:={}&
	R_2  \,\cup \{((T\cap A_2)\cup \{*_2\}, *_2) \mid T\in \mathcal{T}^{\text{-}\mathcal{D}}_{S_3}(E_1)\}\\
    S^{C_2}_2:={}&
    S_2 \;\cup 
    \{(T\cap A_2, *_2) \mid T\in \mathcal{T}^{\text{-}\mathcal{D}}_{S_3}(E_1)\}
\end{align*}
\end{definition}
\begin{example}
We continue Example \ref{ex:type-2-modification} and apply the type-2-modification
    for $E_1\in \{\{e\},\{a,e\}\}$.
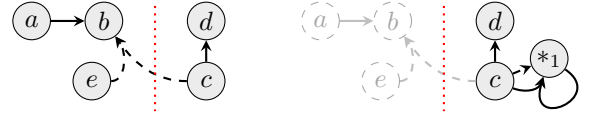
    \begin{center}
		\begin{tikzpicture}[xscale=0.8,>=stealth]
            \path
    		(0.8,2)node[arg] (a){$a$}
    		(2,2)node[arg] (b){$b$}
    		(3.7,1.2)node[arg] (c){$c$}
            (3.7,2)node[arg] (d){$d$}
            (1.8,1.2)node[arg] (e){$e$}
    		;
    		\path[thick,->]
    		(a)edge(b)
    		(c)edge[dashed, in=-50,out=180](b)
    		(e)edge[dashed, in=-50,out=0,looseness=1.1](b)
            (c)edge(d)
            ;
			\draw [thick,dotted,red] (2.85,2.2) -- (2.85,0.8);

            \begin{scope}
                [xshift=4.8cm]
                \path
    		(0.8,2)node[argd] (a){$a$}
    		(2,2)node[argd] (b){$b$}
    		(3.7,1.2)node[arg] (c){$c$}
            (3.7,2)node[arg] (d){$d$}
            (4.6,1.5)node[arg] (x){$*_1$}
            (1.8,1.2)node[argd] (e){$e$}
    		;
    		\path[thick,->]
    		(a)edge[gray!50](b)
    		(c)edge[dashed, in=-50,out=180,gray!50](b)
    		(e)edge[dashed, in=-50,out=0,looseness=1.1,gray!50](b)
            (c)edge(d)
            (c) edge[dashed] (x)
            (c) edge[in=-110,out=-20] (x)
            (x) edge[in=-110,out=-30,looseness=5] (x) 
            ;
			\draw [thick,dotted,red] (2.85,2.2) -- (2.85,0.8);

            \end{scope}
		\end{tikzpicture}
	\end{center}
    Now, in $F^{C_2}_2$, the argument $c$ supports $*_2$ and $d$ is attacked; thus, no argument can be accepted and the only admissible extension is $\emptyset$. This correctly detects the admissible extensions $\{e\}$ and $\{a,e\}$ of $F$. For $E_1=\{a\}$, we have $\{e,c\}\notin \mathcal{T}_{S_3}(E_1)$, thus the support is simply removed.  
\end{example}
We define the \emph{S-reduct} as a combination of both modifications. 
The union of two BSAFs is defined componentwise.%
\begin{definition}[S-reduct]\label{def:S-reduct}
    	Let $(F_1,F_2,S_3)$ be a support splitting for a BSAF $F$, let $E_1$ be an extension of $F_1$ and let $\mathcal{D}_{S_3}(E_1)$ and $\mathcal{T}_{S_3}(E_1)$ denote all sets which are closure-defeated resp.\ support-incompatible wrt.\ $E_1$ and $S_3$.  
        We define the \emph{$(E_1,S_3)$-reduct} (or simply \emph{S-reduct}) of $F_2$ as the BSAF $F\redEone_2=(A\redEone_2,R\redEone_2,S\redEone_2)$ where $F\redEone_2:=F_2^{C_1}\cup F_2^{C_2}$.
\end{definition}
The splitting procedure preserves admissible, complete and stable semantics. 
\begin{restatable}{theorem}{SuppSplit}\label{thm:SuppSplit}
Let $(F_1,F_2,S_3)$ be a support splitting for a BSAF $F=(A,R,S)$ with $F_1=(A_1,R_1,S_1)$, $F_2=(A_2,R_2,S_2)$,  and $\sigma \in \{\stb,\adm,\com\}$. Further, let $F\redEone_2$ be the $S$-reduct of $F_2$. 
    \begin{enumerate}
        \item If $E_1\in \sigma(F_1)$ and $E_2\in \sigma(F\redEone_2)$, then $E_1\cup (E_2\setminus\{*_2\})\in \sigma(F)$. 
		\item If $E\in \sigma(F)$, then $E_1=E\cap A_1\in \sigma(F_1)$ and ($E\cap A_2\in \sigma(F\redEone_2)$ or $(E\cap A_2)\cup \{*_2\}\in \sigma(F\redEone_2)$).
    \end{enumerate}
\end{restatable}
For preferred and grounded semantics, our procedure constructs only some of the extensions of the entire framework. We illustrate this in the following example.
\begin{example}
Consider the BSAF $F$ and a support splitting $(F_1,F_2,S_3)$ as depicted below. 
Let $E_1=\{a,d\}$.  We compute the reduct $F\redEone_2$ by adding the dummy argument $*_2$ which is supported by $\{b\}$ and attacked by $\{b,*_2\}$ (depicted right).%
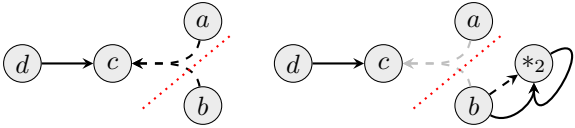
\begin{center}
			\begin{tikzpicture}[scale=0.8,>=stealth]
			\begin{scope}[shift={(0,0)}]
			\path
            (-1,2)node[arg] (d){$d$}
			(0.5,2)node[arg] (c){$c$}
			(2,2.7)node[arg] (a){$a$}
			(2,1.3)node[arg] (b){$b$}
			;
			\path[thick,->]
            (d)edge(c)
			(a)edge[dashed,out=-100,in=0](c)
			(b)edge[dashed,out=100,in=0](c);
			
			\draw [thick,dotted,red] (1,1.25) -- (2.5,2.5);

			\end{scope}
			\begin{scope}[shift={(4.5,0)}]
			\path
            (-1,2)node[arg] (d){$d$}
			(0.5,2)node[arg] (c){$c$}
			(2,2.7)node[arg] (a){$a$}
			(2,1.3)node[arg] (b){$b$}
            (3,2) node[arg] (*) {$*_2$}
			;
			\path[thick,->]
            (d)edge(c)
			(a)edge[dashed,out=-100,in=0,gray!50](c)
			(b)edge[dashed,out=100,in=0,gray!50](c)
			(b)edge[dashed] (*)
            (b) edge[in=-90,out=-30] (*)
            (*) edge[in=-90,out=30,looseness=5] (*)
            ;
            
			\draw [thick,dotted,red] (1,1.25) -- (2.5,2.5);
		
			\end{scope}
			\end{tikzpicture}
		\end{center}
In $F\redEone_2$, we get $\prf(F\redEone_2)=\{\emptyset\}$, obtaining only $\{a,d\}$ as preferred extension of the whole framework. 
\end{example}
\vspace{-5pt}
We get a splitting theorem for preferred and grounded semantics that considers only the first direction. 
\begin{restatable}{theorem}{PrefSuppSplit}\label{thm:PrefSupp}
    Let $(F_1,F_2,S_3)$ be a support splitting for a BSAF $F=(A,R,S)$, and let $F\redEone_2$ be the $S$-reduct of $F_2$ and $\sigma\in\{\grd,\pref\}$. If $E_1\in \sigma(F_1)$ and $E_2\in \sigma(F\redEone_2)$, then $E_1\cup (E_2\setminus \{*_2\})\in \sigma(F)$. 
\end{restatable}

 \section{Splitting Collective Attacks and Supports}
 \label{sec:combined splitting}
Now that we have the building blocks of our approach, we combine the notions of attack and support splitting and provide results similar to the section before.

\begin{definition}[Splitting]\label{def:bsetaf supp splitting}
	Let $F=(A,R,S)$ be a BSAF, $F_1=(A_1,R_1,S_1)$ and $F_2=(A_2,R_2,S_2)$ two sub-frameworks of $F$ such that $A_1\cap A_2=\emptyset$, $A=A_1 \cup A_2$, $R=R_1 \cup R_2\cup R_3$, and $S=S_1\cup S_2 \cup  S_3$ with:
    \begin{align*}
        R_3\subseteq &\, \{(T,h)\in R\mid T\cap A_1\neq\emptyset, T\subseteq A, h\in A_2\},\\
        S_3\subseteq &\,\{(T,h)\in S\mid T\cap A_2\neq\emptyset, T\subseteq A, h\in A_1\}.
    \end{align*}
      We call the 4-tuple $(F_1,F_2,R_3,S_3)$ a \emph{splitting} of $F$.
\end{definition}

We apply the splits we developed in the previous two sections consecutively.
Since attacks $(T,h)$ in $F_2$ can support arguments in $F_1$, we start our splitting procedure by closing attacks in $R_2$, before we apply the modifications we developed in Sections~\ref{sec:slitting over collective attacks} and~\ref{sec:CollSupp}.
Overall, we perform the following splitting procedure.
Given a splitting $(F_1,F_2,R_3,S_3)$ for a BSAF $F$ and a set $E_1\subseteq A_1$, we proceed as follows. 
\begin{enumerate}
    \item We close the attacks in $R_2\cup R_3$. Note that this may transform $R_2$ attacks into shared attacks
    between $F_1$ and $F_2$ since the tail of an attack $R_2$ can support an argument in $A_1$.
    To account for this, we move the attacks in question into the set $R_3$;
    by $\hat F_2$ and $\hat R_3$, we denote the adjusted BSAF and set of common negative links, respectively.
    \item We proceed by constructing the R-reduct $\hat F\redEone_2$ of the attack splitting $(F_1,\hat F_2,\hat R_3)$ wrt.\ $E_1$, following Definition~\ref{def:R-reduct}.
    \item Next, we compute the modification $F^\star_2=mod^{E_1}_{\hat R_3}(\hat F\redEone_2)$ as in Definition~\ref{def_mod}. The tuple $(F_1,F^\star_2,S_3)$ is a support splitting for the BSAF $(A_1\cup A_2^\star, A_1\cup A_2^\star, S)$.
    \item Finally, we construct the S-reduct $F_2^\circledast$
    of $F^\star_2$ for the support splitting $(F_1,F^\star_2,S_3)$ wrt.\ $E_1$, using Definition~\ref{def:S-reduct}.
\end{enumerate}

\begin{definition}\label{def:combined splitting procedure}
    Let $(F_1,F_2,R_3,S_3)$ be a splitting for a BSAF $F$ with $F_1=(A_1,R_1,S_1)$, $F_2=(A_2,R_2,S_2)$. 
    Let $R_3\clink$ and $R_2\clink$ denote the sets of closed negative links wrt.\ $R_3$ and $R_2$, respectively. 
    Define $$\hat R_2:=\{(T,h)\in R_2\clink\mid T\cap A_1=\emptyset\}$$
    and let $\hat F_2:=(A_2,\hat R_2,S_2)$ be the updated BSAF.
    Further, let $$\hat R_3:=\{(T,h)\mid (T,h)\in R_3\clink\cup  R_2\clink, T\cap A_1\neq \emptyset\}.$$
    Let $\hat F\redEone_2$ be the R-reduct of the attack splitting $(F_1,\hat F_2,\hat R_3)$ wrt.\ $E_1$ and $F^\star_2:=mod^{E_1}_{\hat R_3}(\hat F\redEone_2)$. 
    Finally, we let $F_2^\circledast$ denote the S-reduct for the support splitting $(F_1,F^\star_2,S_3)$ of the BSAF $(A_1\cup A_2^\star, A_1\cup A_2^\star, S)$ wrt.\ $E_1$.
\end{definition}
We demonstrate the procedure in the following example.

\begin{example}\label{ex:final ex}
We consider a BSAF $F$ with splitting $(F_1,F_2,R_3,S_3)$, as depicted below.
Further, let $E_1=\{a\}$.
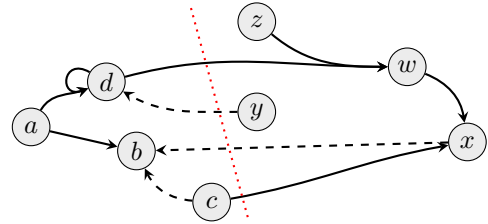
    \begin{center}
			\begin{tikzpicture}[scale=0.8,>=stealth]
            \path
    		(-3.75,0.75)node[arg] (a){$a$}
    		(-2,0.35)node[arg] (b){$b$}
    		(-0.75,-0.5)node[arg] (c){$c$}
            (-2.5,1.5)node[arg] (d){$d$}
            (0,1) node[arg] (y) {$y$}
            (3.5,0.5) node[arg] (x) {$x$}
            (2.5,1.75) node[arg] (w) {$w$}
            (-0,2.5) node[arg] (*0) {$z$}
    		;
    		\path[thick,->]
            (a)edge[out=60,in=-160](d)
            (d)edge[in=-160,out=150,looseness=5] (d)
    		(a)edge(b)
    		(c)edge[dashed,bend left](b)
            (c)edge[out=10,in=190] (x)
            (x)edge[dashed] (b)
            (y) edge [dashed,out=180, in=-30] (d)
            (d)edge[out=15,in=180] (w)
            (*0) edge[out=-40,in=180] (w)
            (w)edge[bend left] (x)
            ;
			\draw [thick,dotted,red] (-0.15,-0.75) -- (-1.1,2.75);
			
			\end{tikzpicture}
		\end{center}
First, we proceed by closing the negative links in $R_2$ and $R_3$, as depicted below. 
The set $R_3$ contains the single attack $(c,x)$. Since $c$ supports $b$, 
the attack $(\{b,c\},x)$ is contained in $\hat R_3$. Moreover, we get $\hat R_2=R_2$ because $x$ and $y$ do not attack any argument in $F_2$. Thus we consider the following attack splitting $(F_1,F_2,\hat R_3)$.

  
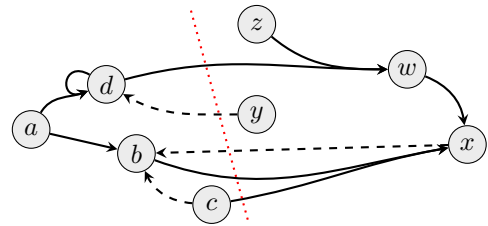
\begin{center}
    \begin{tikzpicture}[scale=0.8,>=stealth]
            \path
    		(-3.75,0.75)node[arg] (a){$a$}
    		(-2,0.35)node[arg] (b){$b$}
    		(-0.75,-0.5)node[arg] (c){$c$}
            (-2.5,1.5)node[arg] (d){$d$}
            (0,1) node[arg] (y) {$y$}
            (3.5,0.5) node[arg] (x) {$x$}
            (2.5,1.75) node[arg] (w) {$w$}
            (-0,2.5) node[arg] (*0) {$z$}
    		;
    		\path[thick,->]
            (a)edge[out=60,in=-160](d)
            (d)edge[in=-160,out=150,looseness=5] (d)
    		(a)edge(b)
    		(c)edge[dashed,bend left](b)
            (c)edge[out=10,in=190] (x)
            (x)edge[dashed] (b)
            (y) edge [dashed,out=180, in=-30] (d)
            (b) edge[out=-20,in=190] (x)
            (d)edge[out=15,in=180] (w)
            (*0) edge[out=-40,in=180] (w)
            (w)edge[bend left] (x)
            ;
			\draw [thick,dotted,red] (-0.15,-0.75) -- (-1.1,2.75);

			\end{tikzpicture}
		\end{center}
We apply the R-reduct and the modification for the attack splitting (see Section~\ref{sec:slitting over collective attacks}). We depict $F_1$ (left) and the result of the procedure, $F^\star_2$, (right) below.
  
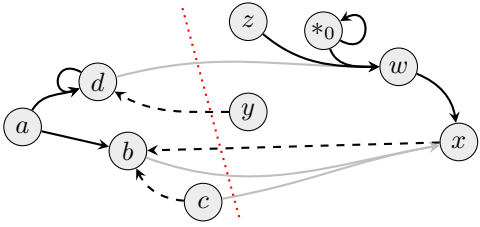
\begin{center}
			\begin{tikzpicture}[scale=0.8,>=stealth]
            \path
    		(-3.75,0.75)node[arg] (a){$a$}
    		(-2,0.35)node[arg] (b){$b$}
    		(-0.75,-0.5)node[arg] (c){$c$}
            (-2.5,1.5)node[arg] (d){$d$}
            (0,1) node[arg] (y) {$y$}
            (3.5,0.5) node[arg] (x) {$x$}
            (1.25,2.35) node[arg] (z) {$*_0$}
            (2.5,1.75) node[arg] (w) {$w$}
            (-0,2.5) node[arg] (*0) {$z$}
    		;
    		\path[thick,->]
            (a)edge[out=60,in=-160](d)
            (d)edge[in=-160,out=150,looseness=5] (d)
    		(a)edge(b)
    		(c)edge[dashed,bend left](b)
            (c)edge[out=10,in=190,gray!50] (x)
            (x)edge[dashed] (b)
            (y) edge [dashed,out=180, in=-30] (d)
            (b) edge[out=-20,in=190,gray!50] (x)
            (z) edge[in=30,out=-30,looseness=5] (z)
            (d)edge[out=15,in=180,gray!50] (w)
            (z) edge[out=-70,in=180] (w)
            (*0) edge[out=-40,in=180] (w)
            (w)edge[bend left] (x)
            ;
			\draw [thick,dotted,red] (-0.15,-0.75) -- (-1.1,2.75);

			\end{tikzpicture}
		\end{center} 
Finally, we apply the S-reduct to handle the remaining links. Thus, we obtain $F_2^\circledast$ (\emph{right}) adding a type-1-constraint on $x$ and a type-2-constraint on $y$.
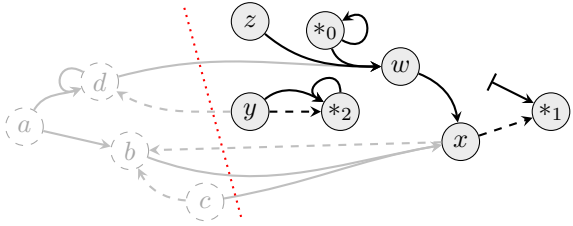
  \begin{center}
			\begin{tikzpicture}[scale=0.8,>=stealth]
            \path
    		(-3.75,0.75)node[argd] (a){$a$}
    		(-2,0.35)node[argd] (b){$b$}
    		(-0.75,-0.5)node[argd] (c){$c$}
            (-2.5,1.5)node[argd] (d){$d$}
            (0,1) node[arg] (y) {$y$}
            (3.5,0.5) node[arg] (x) {$x$}
            (1.25,2.35) node[arg] (z) {$*_0$}
            (2.5,1.75) node[arg] (w) {$w$}
            (-0,2.5) node[arg] (*0) {$z$}
            (5,1) node[arg] (*1) {$*_1$}
            (1.5,1) node[arg] (*2) {$*_2$}
    		;
    		\path[thick,->]
            (a)edge[out=60,in=-160,gray!50](d)
            (d)edge[in=-160,out=150,looseness=5,gray!50] (d)
    		(a)edge[gray!50](b)
    		(c)edge[dashed,bend left,gray!50](b)
            (c)edge[out=10,in=190,gray!50] (x)
            (x)edge[dashed,gray!50] (b)%, out=165,in=10
            (y) edge [dashed,out=180, in=-30,gray!50] (d)
            (x) edge[dashed] (*1)
            (b) edge[out=-20,in=190,gray!50] (x)
            (z) edge[in=30,out=-30,looseness=5] (z)
            (y) edge[out=40,in=150] (*2)
            (*2) edge[out=90,in=150,looseness=4] (*2)
            (d)edge[out=15,in=180,gray!50] (w)
            (z) edge[out=-70,in=180] (w)
            (*0) edge[out=-40,in=180] (w)
            (w)edge[bend left] (x)
            (4,1.5)edge[|->] (*1)
            (y) edge[dashed] (*2)
            ;
			\draw [thick,dotted,red] (-0.15,-0.75) -- (-1.1,2.75);

			\end{tikzpicture}
		\end{center} 
\end{example}

We are ready to state our main result for splitting BSAFs. 

\begin{restatable}{theorem}{AttSuppSplit}\label{thrm:AttSuppSplit}
    Let $(F_1,F_2,R_3,S_3)$ be a splitting for a BSAF $F=(A,R,S)$ with $F_1=(A_1,R_1,S_1)$, $F_2=(A_2,R_2,S_2)$, and $\sigma \in \{\stb,\adm,\com\}$. 
    Let $F^\circledast_2$ be defined as in Definition~\ref{def:combined splitting procedure}.
    \begin{enumerate}
        \item If $E_1\in \sigma(F_1)$ and $E_2\in \sigma(F_2^{\circledast})$, then $E_1\cup (E_2\setminus \{*_2\})\in \sigma(F)$. 
		\item If $E\in \sigma(F)$, then $E_1=E\cap A_1\in \sigma(F_1)$ and $E\cap A_2\in \sigma(F_2^\circledast)$ or  $(E\cap A_2)\cup\{*_2\}\in \sigma(F_2^\circledast)$.
    \end{enumerate}
\end{restatable}
\begin{example}
    We use the splitting theorem to compute the extensions of BSAF $F$ from Example~\ref{ex:final ex}. 
    First, note that the set $E_1=\{a\}$ is admissible in $F_1$. 
    Next, we compute the admissible extensions of $F^\circledast_2$ and obtain the $\emptyset$ and $\{z\}$. Hence, we retrieve $\{a\}$ and $\{a,z\}$ in $\adm(F)$.
\end{example}

Note that we inherit the limitations with respect to preferred and grounded semantics. 
\begin{restatable}{theorem}{PrefGrdSuppAttSplit}\label{thm:PrefSuppAtt}
    Let $(F_1,F_2,R_3,S_3)$ be a splitting for the BSAF $F=(A,R,S)$, let $F^\circledast_2$ be as in Definition~\ref{def:combined splitting procedure} and $\sigma\in\{\grd,\pref\}$. 
    If $E_1\in \sigma(F_1)$ and $E_2\in \sigma(F^\circledast_2)$, then $E_1\cup (E_2\setminus \{*_2\})\in \sigma(F)$. 
\end{restatable}

\section{Related Work}
Splitting has a long tradition in the field of knowledge representation and reasoning,
in particular, for non-monotonic formalisms. The first splitting procedure has been proposed for logic programs \cite{LifschitzT94}
and has received much attention since then, 
e.g., in the context of answer set programming~\cite{BeiserHUW24},
extended fragments of logic programming~\cite{BenEliyahuZohary25,CabalarFC21},
for conditional knowledge bases~\cite{HeyninckKMHB23},
and for abstract and structured forms of argumentation~\cite{Baumann11,BaumannBW11,BuraglioDKW24,Linsbichler14,Buraglio25}.

\subsubsection{Incremental Computation for Supports} 
While most works in computational argumentation focus on splitting attacks, ADFs constitute a notable exception. In ADFs, each argument has an acceptance condition, capable of expressing positive relations between arguments.
Linsbichler~(\citeyear{Linsbichler14}) investigated splitting techniques for ADFs, while more recent work has explored serialization~\cite{BengelT25ADFSerialisation} for abstract dialectical frameworks. The idea behind serialization is based on finding so-called \textit{initial models} of a given ADF in an iterative way, thereby constructing serialization sequences that allow the incremental computation of its extensions.  
We notice, however, that even though ADFs have been proven to capture collective attacks~\cite{DvorakZW23}, this does not hold for collective supports under the closure condition. 
From a semantic standpoint, this constitutes a substantial difference with respect to BSAFs. Consequently, existing incremental and splitting results for ADFs cannot be transferred directly to our setting.

\subsubsection{Directionality, Modularization, and SCC-recursiveness}
Further approaches that compute extensions in an incremental manner closely related to splitting techniques include 
schemes based on strongly connected components (SCCs)~\cite{BaroniGG05a}, modularization~\cite{BaumannBU22} and directionality~\cite{BaroniG07}.
Modularization is directly connected to the splitting schema via a similar notion of reduct~\cite{BertholdBR25}.
The principle of directionality is connected with the possibility of evaluating sub-frameworks independently from the remaining parts, provided that there is no incoming negative link towards them. 
The relation to attack splitting is thus clear given that we can project extensions of $F$ to those of $F_1$ without losing information.
The SCC-recursive property~\cite{BaroniGG05a} of argumentation semantics defines the possibility to retrieve extensions of a given framework by breaking down the computation along its SCCs. Such decomposition ensures that the evaluation of the whole graph can be assessed via a generalized semantics that takes into account previous decisions concerning already evaluated SCCs. Splitting represents a more flexible alternative as it does not dictate that $F_1$ and $F_2$ be strongly connected components, while ensuring that every SCC of the graph is in either of the two.

\section{Conclusion}
In this paper, we introduced a splitting schema for bipolar set-based argumentation frameworks. Similar to existing literature on splitting for abstract forms of argumentation, we developed a modification-based approach for the incremental computation of extensions of a given BSAF.
In particular, after splitting the framework $F$ into two sub-frameworks $F_1$ and $F_2$, we defined syntactic modifications of $F_2$ that allow one to reconstruct extensions of the original framework from those of the components.
By leveraging the richer syntax of BSAFs, we have considered splits over collective attacks, supports, and a combination of both. 
Importantly, each step of our procedure is computationally feasible and does not induce an exponential blow-up in the size of the modified sub-frameworks, 
making our approach a promising starting point for efficient algorithms. 

For future work, we plan to extend our splitting techniques to cases where (a certain number of) backward attacks and forward supports are present.
In the context of AFs, similar settings have been studied under the term \emph{parameterized splitting}~\cite{BaumannBDW12}. 
In addition, we plan to implement our procedure and investigate its benefits for dynamic programming approaches in argumentation.
\textcolor{black}{In this regard, observe that for Dung-style AFs, all possible splittings can be computed in linear time by identifying the strongly connected components of the graph~\cite{BaumannBW11,Tarjan72}.}
\textcolor{black}{We anticipate that for BSAFs, a similar approach to computing splittings in AFs can be adapted, by treating the support and attack relations separately and by exploiting the so-called \emph{primal graph}~\cite{DvorakKUW24}.}
\textcolor{black}{
On top of this,~\citeauthor{BaumannBW11}~(\citeyear{BaumannBW11}) register an average acceleration of 60\% for executions with splitting in comparison to an execution without splitting in the context of Dung-style AFs. 
}

Crucially, BSAFs are closely related to general (non-flat) ABAFs, as demonstrated by~\citeauthor{BertholdRU24}~(\citeyear{BertholdRU24}), 
which is one of the most popular formalisms in the area of structured argumentation. 
Recent work exploits the close relation of SETAFs and flat ABAFs to translate SETAF splitting results to this fragment of ABA~\cite{Buraglio25}.
We thus expect that our results for BSAF splitting can be transferred to general ABA in a similar fashion. 
Our hope is that this will enhance existing non-flat ABA solvers~\cite{LehtonenRTU24}. We are furthermore optimistic that our results could  provide further valuable insights into structured argumentation dynamics~\cite{Rapberger024,RapbergerU23,Prakken23,BaumannB10,CayrolSL10,BertholdR023}
and other popular forms of structured argumentation~\cite{DBLP:journals/corr/abs-1804-06763}.

\textcolor{black}{
While the present work has focused on the canonical semantics for BSAFs (which capture ABA semantics), the literature on bipolar argumentation provides a wide range of alternative interpretations of support~\cite{CohenGGS14,CohenPSM18,Amgoud08,Polberg16,BoellaGTV10}. Furthermore,~\cite{BertholdRU24} develop novel BSAF semantics aimed at addressing certain undesirable properties of the canonical approach. Extending established notions of support to the setting of set-supports, as well as developing corresponding splitting techniques for these alternative BSAF semantics, is a promising direction for future research.
}

Lastly, we plan to study the relation of our splitting techniques to the aforementioned principles of directionality, modularity and SCC-recursiveness. 
For doing so, appropriate notions of these principles in the presence of supports need to be identified, which we deem an interesting direction of future work on its own. As suitable starting points for these studies,
we identify the principle-based analyzes for SETAFs~\cite{DvorakKUW24} and BAFs~\cite{YuAVLT23}.

\section*{Acknowledgments}
This work has been supported by the European Union's Horizon 2020 research and innovation programme (under grant agreement 101034440) and by the Deutsche Forschungsgemeinschaft (projects ``Argumentative Reasoning in Nonsensical Situations'', grant 550735820, and ``Explainable Belief Merging'', grant 465447331).

%% The file kr.bst is a bibliography style file for BibTeX 0.99c
\bibliographystyle{kr}
\bibliography{bib}

\newpage
\appendix

\clearpage
\section{Omitted Proofs of Section \ref{sec:slitting over collective attacks}}
\begin{lemma}
\label{lem:closedattacks-cfclosed}
Let $F=(A,R,S)$ be a BSAF, $(T,h)\in R$ and $F'=(A,R',S)$ with $R'=(R\setminus \{(T,h)\})\cup \{(\cl_F(T),h)\}$.
\begin{enumerate}
    \item Let $E\subseteq A$ be a closed set. Then $E\in\cf(F)$ iff $E\in\cf(F')$.
    \item Let $E\subseteq A$. Then $E$ is closed in $F$ iff $E$ is closed in $F'$
\end{enumerate}    
\end{lemma}
\begin{proof}
\begin{enumerate}
    \item Suppose first $E$ is not conflict-free in $F$, then there exists $(T,h)\in R$ with $T\subseteq E,h\in E$. Since $E$ is closed, $cl(T)\subseteq E$. So either $(T,h)$ or $(cl(T),h)\in R'$ is an attack of $E$ on itself in $F'$, so $E$ not conflict-free in $F'$. 

    For the other direction suppose now, $E$ is not conflict-free in $F'$, then there exists $(T',h)\in R'$, such that $T'\subseteq E,h\in E$. By definition of $R'$ there must exist a subset $T\subseteq T'$ such that $T'=cl(T)$ and $(T,h)\in R$. But then also $T\subseteq E, h\in E$, so $E$ is not conflict-free in $F$.

    \item Since argument set and support relation coincide between $F$ and $F'$ and closedness is defined solely over the support relation, the closed sets of $F$ and $F'$ coincide.\qedhere
\end{enumerate}
\end{proof}

\CompleteAttacksInR*

\begin{proof}

\begin{description}
    \item[($\adm$)] ($\subseteq$) First, let $E\in\adm(F)$. Then by Lemma \ref{lem:closedattacks-cfclosed} $E$ is closed and conflict-free in $F'$. It is left to show that $E$ defends itself against closed attackers in $F'$. Let $(T',h)\in R'$ with $h\in E,T'\subseteq A$. Then there exists some $T\subseteq T'$ such that $(T,h)\in R$. Now any closed set $D$ containing $T'$ also contains $T$ and for any closed set containing $T$ there exists an attack $(S,t)\in R$ with $S\subseteq E,t\in D$. By definition of $R'$ therefore an attack $(S,t)$ or $(cl(S),t)\in R'$ on $D$ exists. Since $E$ is closed, $cl(S)\subseteq E$, it therefore attacks $D$ in $F'$ and thus defends itself.

    ($\supseteq$) Now let $E\in\adm(F')$. With the lemma $E$ is closed and conflict-free in $F$ and it is left to show that $E$ defends itself. Let $D$ be a closed set and $(T,h)\in R$ with $T\subseteq D,h\in E$. Then either $(T,h)$ or $(cl(T),h)\in R'$, so since $D$ is a closed set and $T\subseteq D$, $cl(T)\subseteq D$ and $D$ attacks $E$ in $F'$. Since $E$ is admissible in $F'$, there exists an attack $(S',t)\in R'$ with $S'\subseteq E, t\in D$ and therefore by definition of $R'$ an attack $(S,t)\in R$ with $S\subseteq S'\subseteq E, t\in D$. So $E$ defends itself in $F$.

    \item[($\com$)] Let $E\in\com(F)$. Then by the previous item $E\in\adm(F')$. It is left to show that $E$ contains every argument it defends in $F'$. Suppose to the contrary, $E$ defends some $h$ in $F'$ and does not contain it. Then for every closed attacker $D$ of $h$ there exists an attack $(S',d)\in R'$ with $S'\subseteq E,d\in D$ in $F'$. But then by definition of $R'$ there exists an attack $(S,d)\in R$ with $S\subseteq S'\subseteq E$, so $E$ defends $h$ in $F$ (since the closed sets of $F$ and $F'$ coincide by Lemma \ref{lem:closedattacks-cfclosed}, the closed attackers on $h$ coincide in $F$ and $F'$), but does not contain $h$. So $E$ is not complete in $F$. Contradiction.

    Now let $E\in\com(F')$, then $E\in\adm(F)$. Suppose again $E$ defends some $h$ in $F$ it does not contain. Then for every closed attacker $D$ of $h$ there exists an attack $(S,d)\in R$ with $S\subseteq E,d\in D$ in $F$. But then by definition of $R'$ there exists an attack either $(S,d)$ or  $(cl(S),d)\in R'$ and since $E$ is closed, $cl(S)\subseteq E$, so $E$ defends $h$ in $F'$ (since the closed sets of $F$ and $F'$ coincide by Lemma \ref{lem:closedattacks-cfclosed}), but does not contain $h$. So $E$ is not complete in $F'$. Contradiction.

    \item[($\grd$)] Follows directly from $\com(F)=\com(F')$.
    \item[($\prf$)] Follows directly from $\adm(F)=\adm(F')$.
    \item[($\stb$)] Let $E\in\stb(F)$. Then by the first item $E\in\adm(F')$. It is left to show, that $E$ attacks every $x\in A\setminus E$ in $F'$, so let $x\in A\setminus E$. Since $E$ is stable in $F$ there exists an attack $(T,x)\in R$ with $T\subseteq E$. Therefore an attack either $(T,x)$ or $(cl(T),x)\in R'$ exists, and since $E$ is closed we have $cl(T)\subseteq E$. So $E$ attacks $x$ in $F'$. For the other direction ($\stb(F')\subseteq\stb(F)$) consider that for any attack $(S',x)\in R'$ with $S'\subseteq E$ by definition of $R'$ there exists a subset $S\subseteq S'\subseteq$ such that $(S,x)\in R$, so $E$ attacks $x$ in $F$.
    \qedhere
\end{description}
\end{proof}

\Closure*
\begin{proof}
    \begin{enumerate} 
        \item Suppose $E_1=cl_{F_1}(E_1)$ and $E_2=cl_{F^{\star}_2}(E_2)$. We show that $E=(E_1\cup E_2)\setminus \{*_0\}$ is closed. Let $a\in A$ s.t.\ there is $(T,a)\in S$ with $T\subseteq E$. Then, either $T\subseteq E_1$, $T\subseteq E_2$, or $T$ has non-empty intersection with both $E_1$ and $E_2$. 
        We proceed by case distinction.
        \begin{itemize}
            \item \underline{Case 1: $T\subseteq E_1$.} Then $a\in E_1$ since $E_1=cl_{F_1}(E_1)$.
            \item \underline{Case 2: $T\subseteq E_2$.} Then $a\in E_2$ since $E_2=cl_{F^{\star}_2}(E_2)$. 
            \item \underline{Case 3:} Suppose $T\cap E_1\neq \emptyset$ and $T\cap E_2\neq \emptyset$. Then there is a support $(T,a)\in S$ with $T\cap A_1\neq \emptyset$ and $T\cap A_2\neq \emptyset$; therefore, $(F_1,F_2,R_3)$ is not a proper attack splitting since $S_1$ and $S_2$ is not a partition of $S$. 
        \end{itemize}
        \item Suppose $E=cl_F(E)$ for some $E\subseteq A$, and let $E_1=E\cap A_1$ and $E_2=E\cap A_2$. 
        \begin{itemize}
            \item We show that $E_1$ is closed: Let $a\in A_1$ with $(T,a)\in S_1$. By construction, $(T,a)\in S$, thus $a\in E$. Since $S_1$ and $S_2$ partitions $S$, we obtain $a\in E\cap A_1=E_1$.
            \item We show that $E_2$ is closed: Let $a\in A_2$ with $(T,a)\in S_2$. Analogous to the first item, we obtain $a\in E_2$ since each support is either fully contained in $A_1$ or $A_2$. \qedhere
        \end{itemize}
    \end{enumerate}
\end{proof}

\begin{lemma}
    \label{lem:conflict-freeness attack splitting}
		Let $(F_1,F_2,R_3)$ be a attack splitting for a BSAF $F$ with $F_1=(A_1,R_1,S_1)$, $F_2=(A_2,R_2,S_2)$. Let $E_1\subseteq A_1$ and $F^{\star}_2=mod^{E_1}_{R_3\clink}(F^{E_1}_2)$ where $R\clink_3$ is the set of closed attacks wrt.\ $R_3$.
		\begin{enumerate}
			\item If $E_1\in \adm(F_1)$ and $E_2\in \adm(F^{\star}_2)$, then $E_1\cup E_2\in \cf(F)$.
			\item If $E\in \cf(F)$, then $E_1=E\cap A_1\in \cf(F_1)$ and $E_2=E\cap A_2\in \cf(F\redEone_2)$.
		\end{enumerate}
\end{lemma}
\begin{proof}
    \begin{enumerate}
        \item Suppose $E_1\in \adm(F_1)$ and $E_2\in \adm(F^{\star}_2)$.
        Let $E=E_1\cup E_2$. We show $E\in \cf(F)$. Let $(T,h)\in R$ with $T\subseteq E$. Either $(T,h)\in R_1$, $(T,h)\in R_2$, or $(T,h)\in R_3$. We proceed by case distinction.
        \begin{itemize}
            \item \underline{Case 1: $(T,h)\in R_1$.} Then $h\notin E$ since $E_1\in \adm(F_1)$. 
            \item \underline{Case 2: $(T,h)\in R_2$.} It holds that $T\subseteq A\redEone_2,h\in A\redEone_2$. In this case, we obtain $h\notin E_2$ since $E_2\in \adm(F_2)$ and since $R_2\subseteq R$.
            
            \item \underline{Case 3: $(T,h)\in R_3$.} Let $T_c=\cl_F(T)$. 
            Then either $(T_c,h)\in U_{R_3\clink}^{E_1}$ or $(T_c,h)\notin U_{R_3\clink}^{E_1}$. 
            \begin{itemize}
                \item 
                \underline{Case 3.i: $(T_c,h)\in U_{R_3\clink}^{E_1}$.} 
                By definition of undecided link there is some $t\in T_c\cap A_1$ such that $t\in A_1\setminus (E_1)_{R_1}^\oplus$. Since $T_c\cap (E_1)^+_{R_1\cup R_3\clink}=\emptyset$ we have $t\in A_1\setminus E_1$.
                Since $t\in T_c=\cl_F(T)$ and since $F_1$ and $F_2$ have no common supports, it holds that $t\in \cl_{F_1}(T\cap A_1)$. We arrived at a contradiction: since $E=E_1\cup E_2$, we have $T\cap A_1\subseteq E_1$. Since, by assumption, $E_1$ is admissible in $F_1$ it is also closed and therefore contains $t$.

                \item \underline{Case 3.ii: $(T_c,h)\notin U_{R_3\clink}^{E_1}$.} 
                Let $T_c=\cl_F(T)$.
                Then either $T_c\cup (E_1)^+_{R_1\cup R_3\clink}\neq \emptyset$ or 
                $T_c\cap A_1 \subseteq (E_1)^\oplus_{R_1}$. 
                Furthermore observe that $T_c\cap A_1\subseteq E_1$ since $T\cap A_1\subseteq E_1$ and since $F_1$ and $F_2$ have no common supports.
                
                We proceed by case distinction.

                (a) Suppose $T_c\cup (E_1)^+_{R_1\cup R_3\clink}\neq \emptyset$. That is, there is some $t\in T_c$ which is attacked by $E_1$ via some attack $(T',t)$ in $R_1$ or $R_3$.
                
                (a.1) If $t\in A_1$ then $(T',t)\in R_1$ and we obtain a contradiction to $E_1\in \adm(F_1)$. 
                
                (a.2) If $t\in A_2$ then $(T',t)\in R_3\clink$ and $t\in (E_1)^+_{R_3\clink}$. 
                By assumption, $T'\subseteq E_1$ (thus, $T'\cap A_1\subseteq E_1$) 
                and $t\in A_2$. 

                (a.2.i) Case $T'\cap (E_1)^+_{R\clink_3}=\emptyset$. Then, by definition of the R-reduct, $R_2^\star$ contains $(\emptyset,t)$. Since $t\in E$, we obtain contradiction to the conflict-freeness of $E_2$ in $F_2^\star$. 

                (a.2.ii) Case $T'\cap (E_1)^+_{R\clink_3}\neq \emptyset$. By definition of attack splitting, this is impossible because $T'\subseteq E_1$.

                (b) If $T\cap A_1\cap E^+_{R_1}\neq \emptyset$, proceed as in (a.1). Otherwise $T\cap A_1 \subseteq E_1$.
                Then $(T\setminus A_1,h)\in R_2\redEone$. Therefore, $h\notin E_2$, and thus $h\notin E$. 
            \end{itemize}
        \end{itemize}
        \item Suppose $E\in \cf(F)$. We show that $E_1=E\cap A_1\in \cf(F_1)$ and $E_2=E\cap A_2\redEone\in \cf(F\redEone_2)$.
        \begin{itemize}
            \item $E_1\in \cf(F_1)$ since $R_1\subseteq R$.
            \item We show that $E_2\in \cf(F\redEone_2)$.
            Let $(T,h)\in R\redEone_2$ with $T\subseteq E_2$.  
            By definition of $R\redEone_2$, one of the following applies:
             \begin{itemize}
                \item $(T,h)\in R_2$ and $T\subseteq A\redEone_2,h\in A\redEone_2$; \hfill (Case 1)
                \item $(T,h)=(T'\setminus A_1, h)$ for some $T'\subseteq A$ such that $(T',h)\in R_3$, $T'\cap (E_1)^+_{R_3}=\emptyset$, $T\cap A_1\subseteq E_1$, and $h\in A_2\redEone$. \hfill (Case 2)
            \end{itemize}
            We proceed by case distinction.
   \begin{itemize}
       \item \underline{Case 1.} $(T,h)\in R_2$ with $T\subseteq A\redEone_2,h\in A\redEone_2$. In this case, we obtain $h\notin E_2$ since $E\in \cf(F)$ and since $R_2\subseteq R$.
       \item \underline{Case 2.} 
       From $T=T'\setminus A_1 \subseteq E_2$ and $T'\cap A_1\subseteq E_1$ we obtain $T'\subseteq E$. We obtain $h\notin E_2$ since $E\in \cf(F)$ and since $R_3\subseteq R$.
   \end{itemize} 
   We have shown that $E_2\in \cf(F\redEone_2)$.\qedhere
        \end{itemize}
    \end{enumerate}
\end{proof}

\AttSplit*
\begin{proof} 

Closure is proven using Proposition \ref{prop:closure}; conflict-freeness is proven using Lemma \ref{lem:conflict-freeness attack splitting}.
We proceed by proving the statement for each semantics separately.

\paragraph{Stable semantics.}
    \begin{enumerate} 
        \item Suppose $E_1\in \stb(F_1)$ and $E_2\in \stb(F_2^{\star})$.
        We show that $E=E_1\cup E_2\in \stb(F)$. 
        \begin{itemize}
            \item $E$ is conflict-free: 
            Note that $F^\star_2=F\redEone_2$ because $U^{E_1}_{R_3\clink}=\emptyset$.
            By Lemma \ref{lem:conflict-freeness attack splitting} and since $E_1\in \stb(F_1)$ and $E_2\in \stb(F_2\redEone)$, we obtain $E\in \cf(F)$.
            \item  $E$ is closed: this follows from Proposition \ref{prop:closure}.

            \item $E$ attacks all remaining arguments: 
            Let $a\in A\setminus E$. We proceed by case distinction.

            \begin{itemize}
                \item \underline{Case 1: $a\in A_1$.} Then $a\in (E_1)^+_{R_1}$ since $E_1\in \stb(F_1)$, thus $a\in E^+_{R}$.
                \item \underline{Case 2: $a\in A_2$.} By hypothesis, we know that $a\in (E_2)^+_{R_2\redEone}$. Hence, there is a $T\subseteq E_2$ such that $(T,a)\in R_2\redEone$. We distinguish two cases.
                \begin{itemize}
                    \item \underline{Case 2.i: $(T,a)\in R_2$.} Then, $a\in (E_2)^+_{R_2}$, from which it immediately follows $a\in (E)^+_{R}$.
                    \item \underline{Case 2.ii: $(T,a)\notin R_2$.}  Then, there is a $T\supseteq T$ such that $(T',a)\in R_3$ with $T'\cap A_1\subseteq E_1$. Hence, $T\subseteq E_1\cup E_2=E$, deriving $a\in E^+_{R_3\clink}$, \ie $a\in (E)^+_{R}$.
                \end{itemize}
            \end{itemize}
            In all cases we get $a\in E^+_R$, therefore  $E\in \stb(F)$.
        \end{itemize}
    \item Suppose $E\in \stb(F)$. It holds that $E^\oplus_R=A=A_1\cup A_2$. 
    \begin{itemize}
        \item We show that $E_1=E\cap A_1\in \stb(F_1)$. From 
        Lemma \ref{lem:conflict-freeness attack splitting}, we obtain $E\cap A_1\in \cf(F_1)$.
        Let $a\in A_1$. Since $(F_1,F_2,R_3)$ is an attack splitting of $F$, it holds that $a$ is attacked by some $T\subseteq A_1$, thus $(E\cap A_1)^\oplus_{R_1} =A_1$ and therefore $E_1\in \stb(F_1)$.
        \item We show that $E_2=E\cap A_2\in \stb(F_2^\star)$.
        By Lemma \ref{lem:conflict-freeness attack splitting}, we obtain $E_2\in \cf(F\redEone_2)$; by Proposition \ref{prop:closure}, it holds that $E_2$ is closed. It remains to prove that $E_2$ attacks all remaining arguments.

	   Let $a\in A_2\setminus E_2$. We show $a\in (E_2)^+_{R_2\redEone}$.
       Since $E\in \stb(F)$ we have $a\in E^+_R$. Thus either $a\in E^+_{R_2}$ or $a\in E^+_{R_3}$.
       We proceed by case distinction.
       \begin{itemize}
           \item \underline{Case 1: $a\in E^+_{R_2}$.} There exists an attack $(T,a)\in R_2$ with $T\subseteq E_2$, and since $R_2\redEone\supseteq R_2$,
           we know $(T,a)\in R_2\redEone$. Therefore, $a\in (E_2)^+_{R_2\redEone}$.
           
           \item \underline{Case 2: $a\in E^+_{R_3}$.} There is some $(T,a)\in R_3$ with $T\subseteq E$, i.e., $T\cap A_1\subseteq E_1$. Since $E$ is conflict-free in $F$ we have $T\cap (E_1)^+_{R_1\cup R_3}=\emptyset$, thus, in $F_2\redEone$ we have an attack $(T\cap A_2\redEone,a)$. Since $T\cap A_2\redEone\subseteq E_2$ we get $a\in (E_2)^+_{R_2\redEone}$.
       \end{itemize}
	We obtain $E_2\in \stb(F\redEone_2)=\stb(F^\star_2)$.     
    \end{itemize}
    \end{enumerate}
\paragraph{Admissible Semantics.} 
    \begin{enumerate} 
        \item Suppose $E_1\in \adm(F_1)$ and $E_2\in \adm(F_2^{\star})$. Observe that $*_0\notin E_2$ since $*_0$ is self-attacking. 
        
        We show that $E=E_1\cup E_2\in \adm(F)$. 
            By Lemma \ref{lem:conflict-freeness attack splitting}, we obtain $E = \cl_F(E)$; moreover, $E$ is closed by Proposition \ref{prop:closure}. 
            It remains to show that $E$ defends itself against each closed set.

            Let $T_c\subseteq A$ denote a closed attacker of $E$. Then there is $(T,a)\in R$ with $T\subseteq T_c$ and $a\in E$.
            Either $a\in E_1$ or $a\in E_2$.
            We proceed by case distinction.
                \begin{itemize}
                    \item \underline{Case 1: $a\in E_1$.} Then $(T,a)\in R_1$ and $T_c\subseteq A_1$ since $A_1$ and $A_2$ are not linked by supports. We obtain that $E$ defends $a$ since $E_1\in \adm(F_1)$.
                    \item \underline{Case 2: $a\in E_2$.} Then $(T,a)\in R_2\cup R_3\clink$.

                    \underline{Case 2.i: $(T,a)\in R_2$.} Then $T\subseteq A_2$ and $T_c\subseteq A_2$. Thus $E_2$ is attacked by $T_c$ in $E_2$. We obtain that $E$ defends $a$ since $E_2\in \adm(F_2^\star)$. 

                    \underline{Case 2.ii:  $(T,a)\in R_3$.} Then there is $T'\supseteq T$ such that $(T',h)\in R_3\clink$. 
                    It holds that $\cl_F(T)=\cl_F(T')$, and therefore $T'=T_c$. We proceed by case distinction.
                    
                    \underline{Case 2.ii.a: $T'\cap (E_1)^+_{R_1\cup R_3\clink}\neq \emptyset$.}
                    Then $a$ is defended by $E_1$ against $(T',a)$, thus also by $E$.

                    \underline{Case 2.ii.b: $T'\cap (E_1)^+_{R_1\cup R_3\clink}= \emptyset$.}

                    In this case, either $(T'\cap A_2,a)\in R_2^\star$ (via the reduct) or $((T'\cap A_2)\cup \{*_0\},a)\in R_2^\star$ (via the modification).         
                    Recall that $a\in E_2$ and, moreover, $E_2\in \adm(F_2^\star)$.
                    Observe that $\cl_{F_2^\star}(T'\cap A_2)=\cl_{F_2^\star}((T'\cap A_2)\cup \{*_0\})$ since $*_0$ is not contained in the tail of any support in $S$. 
                    Thus, in both cases, $E_2$ defends itself (in $F_2^\star$) against the closed attacker $T_c'=\cl_{F_2^\star}(T'\cap A_2)$: there is $U\subseteq E_2$, $t\in T_c'$, such that $(U,t)\in R_2^\star$. 

                    Furthermore note that $T_c'\subseteq T_c$ since $T'\cap A_2\subseteq T'$. Therefore, $t\in T_c$. 

                    Now, either 
                    (a)~$(U,t)\in R_2$ or 
                    (b)~there is some $(U',t)\in R_3$ with $U'\supset U$ s.t.\ $U'\cap A_1\subseteq E_1$.
                    In the former case, $(U,t)\in R$ and thus $a$ is defended against $T_c$ by $E$.
                    In the latter case, we have $U'\cap A_1\subseteq E_1$ and $U=U'\cap A_2\subseteq E_2$ and thus $U'\subseteq E$; it follows that $E$ defends $a$ in $F$.
                    
                    In any case, $a$ is defended in $F$ by $E$.
                \end{itemize}
	
    \item Suppose $E\in \adm(F)$. We show that $E_1=E\cap A_1$ and $E_2=E\cap A_2$ are admissible in $F_1$ resp.\ $F^\star_2$.
    By Lemma \ref{lem:conflict-freeness attack splitting}, $E_1=E\in \cf(F_1)$ and $E_2\in \cf(F_2\redEone)$; by Proposition \ref{prop:closure} both sets are closed. We show that $E_1$ and $E_2$ defend themselves in $F_1$ and $F^\star_2$, respectively, against each attacker. Recall that each attacker in $R_3\clink$ is closed.  
    \begin{itemize}
        \item $E_1$ defends itself in $F_1$ follows since $E$ defends itself in $F$ and no argument in $E_1$ is attacked by a subset of $A_2$ or defended by $E_2$.
        \item $E_2$ defends itself in $F^\star_2$:
        
        Consider a closed set $T_c\subseteq A_2$ which attacks $E_2$. 
        Then there is an attack $(T,a)\in R_2^\star$, $T\subseteq T_c$ and $a\in E_2$.
        Wlog we assume $T_c=\cl_{F^\star_2}(T)$.

        \underline{Case 1: $*_0\in T_c$.} By assumption $T_c=\cl_{F^\star_2}(T)$ and since $*_0$ is not the head of any support, $(T,h)=((V\cap A_2)\cup \{*_0\},h)$ for some $(V,h)\in U^{E_1}_{\hat R_3}$. 
        Then there is $V'\subseteq V$ such that $\cl_F(V')=V$ and $(V',h)\in R_3$.
        $V$ is a closed set which attacks $E$ in $F$ on $h$. It holds that $V\cap (E_1)^+_R=\emptyset$, thus there is $(U,b)\in R$ such such that $U\subseteq E$ and $b\in V\cap A_2$. It holds that $(U,b)\in R_2\cup R_3$.

        \begin{itemize}
            \item \underline{Case 1.i: $(U,b)\in R_2$.} Then $(U,b)\in R_2\redEone$  and thus $E_2$ defends itself against $T_c$ in $F^\star_2$.
            \item \underline{Case 1.ii: $(U,b)\in R_3$.} Let $U_c=\cl_F(U)$ and observe that $U_c\subseteq E$ since $E$ is admissible in $F$. Further note that $U_c\subseteq E_2$ since no common positive links between $F_1$ and $F_2$ exist. Since $U_c\cap A_1\subseteq E_1$, the link is not undecided. Therefore, $(U\cap A_2,b)\in R_2\redEone$ and thus $E_2$ defends itself against $T_c$.
        \end{itemize}

        \underline{Case 2: $*_0\notin T_c$.}
        In this case, it holds that either (i) $(T,a)\in R_2$ 
        or 
        (ii) $(T,a)$ is of the form $(T'\cap  A_2\redEone,a)$, for some attack $(T',a)\in R_3\clink$, $T'\cap A_1 \subseteq E_1$.
        Note that $(T,a)$ cannot be of the form $((T'\cap  A_2\redEone)\cup \{*_0\},a)$ since we assume $*_0\notin T_c$.
        We proceed by case distinction.
                \begin{itemize}
                    \item \underline{Case 2.i: $(T,a)\in R_2$.} 
                    Since $A_1$ and $A_2$ have no common support links, it holds that $\cl_{F^\star_2}(T)=\cl_F(T)=T_c\subseteq A_2$.
                    Since $a$ is defended by $E$ in $F$, there is a counter-attack $(U,b)\in R_2\cup R_3$ s.t.\ $U\subseteq E$ and $b\in T_c$.

                    Consider the case $(U,b)\in R_2$. 
                    It holds that $(U,b)\in R_2\redEone$ and $U\subseteq E_2$. We obtain that $E_2$ defends $a$ against $T_c$ in $F_2^\star$.
                    
                    In case $(U,b)\in R_3$ we have $U\cap A_1\subseteq E_1$ (since $U\subseteq E$), and hence we get an attack $(U\cap A_2\redEone,b)\in R_2\redEone$ which defends $a$ in $F^\star_2$.

                    \item \underline{Case 2.ii: $(T,a)=(T'\cap  A_2\redEone,a)$.} 
                    Here, $(T',a)\in R_3\clink$ and $T'\cap A_1 \subseteq E_1$.
                    By assumption, $T'$ is closed in $F$. 
                    As in Case 1, we have $T_c = \cl_{F^\star_2}(T'\cap A_2\redEone) = \cl_{F^\star_2}(T)$ since $F_1$ and $F_2$ are not linked via supports.
                    In particular, we have $T_c=T'\cap A_2$ since $T'$ is closed in $F$. 
                    
                    The set $T'$ attacks $E$ in $F$. 
                    It holds that $T \cap (E_1)^+_{R}=T \cap (E_1)^+_{R_1\cup R_3}=\emptyset$. Moreover, $T'\cap A_1\subseteq E_1$. 
                    Since $E$ defends itself against $T'$, there is a counter-attack $(U,b)\in R_2$ with $b\in T'\cap A_2$ and $U\subseteq E$. Thus, there is an attack $(U\cap A_2,b)\in R_2\redEone$. 
                    We therefore obtain that $E_2$ defends $a$ against $T_c$ in $F_2^\star$.
                \end{itemize}
    \end{itemize}
\end{enumerate}

\paragraph{Complete Semantics.}
\begin{enumerate}
    \item Let $E_1\in\com(F_1),E_2\in\com(F_2^\star)$. Then for $E=E_1\cup E_2$ we have by the previous item $E\in\adm(F)$. It is left to show that $E$ contains everything it defends. So let $a\in A$ such that for every closed attacker $D$ on $a$ with some $(T,a)\in R,T\subseteq D$ there exists
an attack $(S,h)$ from $S\subseteq E$ to $h\in D$. We proceed by case distinction
\begin{description}
    \item[($a\in A_1$)] Then $D\subseteq A_1$ and therefore $(S,h)\in R_1$, so $E_1$ defends $a$, so $a\in E_1\subseteq E$, since $E_1$ is complete in $F_1$.
    \item[($a\in A_2$)] Then $D\subseteq A_2$ or $D=D_1\cup D_2$ with $D_i\subseteq A_i$ for $i=1,2$. For any such $D$ let $(S,h)$ be the attack from $E$ to $D$, then there are two cases:
    \begin{enumerate}
    \item If $(S,h)\in R_2$, then $S\subseteq E_2$, so $E_2$ defends $a$ in $F_2^\star$. But then $a\in E_2$, since $E_2$ is complete.
    \item If $(S,h)\in R_3$, then, since $S\subseteq E$, we have $S\cap A_1\subseteq E_1$ and, since $E$ is conflict-free, that $S\cap {(E_1)}^+_R=\emptyset$. By Definition \ref{def_reduct} it follows $(cl(S)\cap A_2,h)\in R_2^\star$, and since $E$ is closed, we have $cl(S)\cap A_2\subseteq E_2$ so $E_2$ defends $a$ in $F^\star_2$, so $a\in E_2$.
    \end{enumerate}
\end{description}
Therefore in any case $a\in E$, so $E\in\com(F)$.

\item For the other direction suppose $E\in\com(F)$. Then by the previous item $E_1=E\cap A_1$ is admissible in $F_1$ and $E_2=E\cap A_2$ is admissible in $F^\star_2$. Note that this implies that $*_0\notin E_2$. It is left to show that both $E_1$ and $E_2$ contain every argument they defend.
\begin{enumerate}
    \item Let $a$ be defended by $E_1$ in $F_1$, then for every closed Attacker $D$ on $a$ there exists an attack $(S,h)\in R_1$ with $S\subseteq E_1, h\in D$. But then $E$ defends $a$ in $F$, since $S\subseteq E$ and $R_1\subsetneq R$. So, since $E\in\com(F)$ it follows that $a\in E\cap A_1=E_1$, so $E_1\in\com(F_1)$.
    \item Let $a$ be defended by $E_2$ in $F_2^\star$. Then for every closed Attacker $D$ on $a$ there exists an attack $(S,h)\in R_2^\star$ with $S\subseteq E_2, h\in D$. If $(S,h)\in R_2$, then $(S,h) \in R$, so $E$ defends $a$ and therefore $E\cap A_2=E_2$ must contain $a$, since $E\in\com(F)$. Since $E$ is conflict-free, it cannot contain $*_0$, so it cannot be the case that $(S\cup\{*_0\},h)$ is one of the attacks added in the modification. So if $(S,h)\notin R_2$, there must be an attack $(S^*,h)\in R_3\clink$ with $S^*\cap {(E_1)}_{R_1\cup R_3\clink}^+=\emptyset, S^*\cap A_1\subseteq E_1$ and $S^*\cap A_2=S$. Then $S^*\subseteq E$, so $E$ defends $a$ and therefore, since $E$ is complete, $E_2=E\cap A_2$ contains $a$. 
\end{enumerate}
\end{enumerate}

\paragraph{Preferred Semantics}
\begin{enumerate}
    \item The first direction follows 
directly from the correspondence results for admissible semantics. Suppose $E_1\in\prf(F)$ and $E_2\in\prf(F^\star_2)$ but $E=E_1\cup E_2\notin\prf(F)$, then by the previous item $E\in\adm(F)$, so there exists a preferred superset $E'\supsetneq E$ of $F$. But then either $E'\cap A_1=E'_1\supsetneq E_1$ and since $E'$ is admissible in $F$ we have $E_1\in\adm(F_1)$ so $E_1\notin\prf(F_1)$ or  $E_1'=E_1$ and $E'\cap A_2=E'_2\supsetneq E_2$  with $E_2\in\adm(F^\star_2)$, so $E_2\notin\prf(F^\star_2)$. Contradiction in both cases.
\item Now suppose $E\in\prf(F)$. Then $E\cap A_1=E_1\in\adm(F_1)$ and $E\cap A_2=E_2\in\adm(F_2^\star)$. If there existed some $E'_2\supsetneq E_2$ with $E'_2\in\adm(F^\star_2)$, then $E_1\cup E'_2\in\adm(F)$ and $E\subsetneq E_1\cup E'_2$, so $E$ is not preferred. Contradiction. So $E_2\in\prf(F^\star_2)$.

It is left to show that $E_1\in\prf(F_1)$. Suppose there exists some $E^\dagger_1\supsetneq E_1,E_1^\dagger\in\adm(F_1)$. We will show that in this case $E_2\in\adm(F^{\star\dagger}_2)$ is admissible in the modification of the reduct ${F^{E_1^\dagger}}_2$ wrt. $E^\dagger_1$. 
We have strictly less undecided links, so $R_2^{\star\dagger}\setminus R_2^{E_1^\dagger}\subseteq R_2^\star\setminus R_2\redEone$ 
and we have ${R^{E_1^\dagger}}_2=R\redEone_2\cup\{(T\setminus A_1,h)\mid (T,h)\in R_3,T\cap (E_1^\dagger)^+_{R_1\cup R_3\clink}=\emptyset,
T\cap A_1\subseteq E_1^\dagger, T\cap A_1\nsubseteq E_1,h\in A_2\}$ (since $E_1^\dagger$ is conflict-free, no attacks are lost). 

Now suppose $E_2\notin\adm(F^{\star\dagger}_2)$. Then $E_2$ is either not closed, not conflict-free or does not defend itself. Nothing changed about the arguments or the support, so $E_2$ is closed in $F^{\star\dagger}_2$. We add only attacks in the set $R_2^{E_1^\dagger}=R\redEone_2\cup\{(T\setminus A_1,h)\mid (T,h)\in R_3,T\cap (E_1^\dagger)^+_{R_3}=\emptyset, T\cap A_1\subseteq E_1^\dagger, T\cap A_2\nsubseteq E_1,h\in A_2\}$. Suppose $(T',h)\in\{(T\setminus A_1,h)\mid (T,h)\in R_3,T\cap (E_1^\dagger)^+_{R_3}=\emptyset, T\cap A_1\subseteq E_1^\dagger, T\cap A_1\nsubseteq E_1,h\in A_2\}$ and $h\in E_2$. Then there is an attack $(T'\cup\{*_0\},h)\in R^\star_2$, since the attack $(T,h)\in R$ with $T'=T\cap A_2$ is an undecided link wrt. the original reduct $F\redEone_2$ (because $E_1^\dagger$ is conflict-free it cannot be a link from defeated arguments). So $(T'\cup\{*_0\},h)$ is an attack on $E_2$ in $F^\star_2$. Since $E_2\in\adm(F^\star_2)$, $E_2$ defends itself against this attack in $F^\star_2$, but this means $E_2$ attacks itself, since there are no attacks on $*_0$. But $E_2$ is also conflict-free in $F^\star_2$. Contradiction. 
So $E_2$ is conflict-free in $F^{\star\dagger}_2$. It is left to show that $E_2$ defends itself. We can distinguish the following attacks $(T',h)$ on $E_2$ in $F^{\star\dagger}_2$:

\begin{itemize}
    \item \underline{$(T',h)\in R_2^{E_1^\dagger}$.} That is, $(T',h)\in\{(T\setminus A_1,h)\mid (T,h)\in R_3,T\cap (E_1^\dagger)^+_{R_3}=\emptyset, T\cap A_1\subseteq E_1^\dagger, T\cap A_2\nsubseteq E_1,h\in A_2\}$. 
    As we just argued, for each such an attack there is a corresponding attack $(T'\cup\{*_0\},h)\in R^\star_2$ stemming from an undecided link. $E_2$ defends itself against this attack by attacking every closed set containing $T'\cup \{*_0\}$ in $F^\star_2$. Since $*_0$ is not supported by anything and only attacked by itself, for each closed set $D\cup\{*_0\}$ the subset $D$ is closed. So since $E_2$ attacks every closed $D$ containing $T'\cup\{*_0\}$ and does not attack $*_0$, it attacks any closed $D$ containing $T'$. Furthermore, none of the attacks stemming from undecided links are used for this, since $*_0\notin E_2$, so $E_2$ also attacks $D$ in $F^{\star\dagger}_2$.
    \item \underline{$(T',h)\notin R_2^{E_1^\dagger}$.} Since the arguments and supports are the same between $F^\star_2$ and $F^{\star\dagger}_2$, $T'$ is contained in the same closed sets $D$, so $E_2$ attacks every closed set containing $T'$ in $F^\star_2$, since $E_2$ is admissible, so it also does so in $F^{\star\dagger}_2$, since all attacks going out from $E_2$ are preserved.
\end{itemize}
So $E_2$ defends itself in $F^{\star\dagger}$, so $E_2\in\adm(F^{\star\dagger}_2)$. But then $E^{\dagger}_1\cup E_2\in\adm(F)$. Contradiction. So $E_1\in\prf(F_1)$.
\end{enumerate}

\GrdAttSplit*
Suppose $E_1\in\grd(F_1)$ and $E_2\in\grd(F^\star_2)$ but $E=E_1\cup E_2\notin\grd(F)$. Then, by the Theorem \ref{thrm:AttSplit}, $E\in\com(F)$. Thus, there exists a grounded subset $E'\subsetneq E$ of $F$. But then either $E'\cap A_1=E'_1\subsetneq E_1$ and since $E'$ is complete in $F$ we have $E_1\in\com(F_1)$ so $E_1\notin\com(F_1)$ or  $E_1'=E_1$ and $E'\cap A_2=E'_2\subsetneq E_2$  with $E_2\in\com(F^\star_2)$, so $E_2\notin\com(F^\star_2)$. Contradiction in both cases.
\end{proof}

\clearpage
\section{Omitted Proofs of Section \ref{sec:CollSupp}}

\begin{lemma}
    \label{lem:conflict-freeness support splitting}
		Let $(F_1,F_2,S_3)$ be a support splitting for a BSAF $F$ with $SF_1=(A_1,R_1,S_1)$, $F_2=(A_2,R_2,S_2)$. Let $E_1\subseteq A_1$ and $F\redEone_2$ the $S$-reduct of $F_2$ wrt. $E_1$.
		\begin{enumerate}
			\item If $E_1\in \cf(F_1)$ and $E_2\in \cf(F\redEone_2)$, then $E_1\cup (E_2\setminus\{*_2\})\in \cf(F)$. 
			\item If $E\in \cf(F)$, then $E_1=E\cap A_1\in \cf(F_1)$ and $E\cap A_2\in \cf(F\redEone_2)$.
		\end{enumerate}
\end{lemma}
\begin{proof} 
We prove the statements separately. 
    \begin{enumerate}
        \item Suppose $E_1\in \cf(F_1)$ and $E_2\in \cf(F\redEone_2)$. Let $E=E_1\cup (E_2\setminus *_2)$. We distinguish two cases:
        \begin{itemize}
            \item $\mathcal{T}_{S_3}(E_1)=\emptyset$. In this case, we have $F\redEone_2=F_2$ and thus $E_2\in \cf(F_2)$. Therefore, $E_1\cup E_2\in \cf(F)$ follows from the fact that there is no attack from $F_1$ and $F_2$. 
            \item $\mathcal{T}_{S_3}(E_1)\neq \emptyset$. In this case, we have $R_2\redEone\supseteq R_2\cup \{(\emptyset,*_1)\}$, so that $E_2\in \cf(F_2)$. Moreover, $*_1 \notin E_2$ because $E_2$ is conflict-free in $F_2\redEone$.
            Furthermore, $E_2\setminus\{*_2\}\subseteq E_2$ and hence $(E_2\setminus \{*_2\}) \in \cf(F_2)$. Therefore, since there are no attacks between $F_1$ and $F_2$, we derive that $E_1\cup (E_2\setminus \{*_2\}) \in \cf(F)$. \qedhere.
        \end{itemize}
        \item Suppose $E\in \cf(F)$. We show that $E_1=E\cap A_1\in \cf(F_1)$ and $E_2=E\cap A_2\in \cf(F\redEone_2)$. 
        \begin{itemize}
            \item Since $R_1\subseteq R$, it immediately follows that $E_1=E\cap A_1\in \cf(F_1)$. 
            \item We now consider $E_2=E\cap A_2\in \cf(F\redEone_2)$. Since $R_2\subseteq R$ and $(F_1,F_2,S_3)$ is a proper support splitting, we know that $E\cap A_2\in \cf(F_2)$. Further, since $E\subseteq A$, we derive that $*_i\notin E\cap A_2$ for all $i\in \{1,2\}$. Thus, since the reduct only adds attacks towards $*_1$ or $*_2$, we know that $E\cap A_2\in \cf(F\redEone_2)$. \qedhere
        \end{itemize}
    \end{enumerate}
\end{proof}

\begin{lemma}\label{lem: closure supp}
    Let $(F_1,F_2,S_3)$ be a support splitting for BSAF $F=(A,R,S)$ with $F_1=(A_1,R_1,S_1)$, $F_2=(A_2,R_2,S_2)$.
    Further, let $E_1\subseteq A_1$, $E_2\subseteq A_2$, $E=E_1\cup E_2$, and let $F_2\redEone=(A\redEone_2,R\redEone_2,S\redEone_2)$ the $S$-reduct of $F_2$ wrt. $E_1$.
	\begin{enumerate}
	    \item If $E_1\in \adm(F_1)$ and $E_2\in \adm(F\redEone_2)$, then $E= \cl_F(E)$ where $E=E_1\cup (E_2\setminus\{*_2\})$. 
		\item If $E\in \adm(F)$, then $E\cap A_1=E_1=cl_{F_1}(E_1)$ and $E_2=cl_{F\redEone_2}(E_2)$ where $E_2=E\cap A_2$.
	\end{enumerate}
\end{lemma}
\begin{proof}
    \begin{enumerate}
        \item Let $E_1\in \adm(F_1)$ and $E_2\in \adm(F\redEone_2)$. We show that $E=E_1\cup (E_2\setminus\{*_2\})$ is closed.
        
        Let $(T,h)\in S$ with $T\subseteq E$. It holds that either $(T,h)\in S_1$, $(T,h)\in S_2$, or $(T,h)\in S_3$. We proceed by case distinction.

        \begin{itemize}
            \item \underline{Case 1: $(T,h)\in S_1$.} Then $T\subseteq E_1$. Since $E_1$ is closed in $F_1$ by assumption, we get $h\in E_1\subseteq E$.
            \item \underline{Case 2: $(T,h)\in S_2$.} Then $T\subseteq E_2$. Moreover, $S_2\subseteq S_2^\star$. Since $E_2$ is closed in $F_2^\star$, we obtain $h\in E_2\subseteq E$.
            \item \underline{Case 3: $(T,h)\in S_3$.} It holds that $h\in A_1$ by definition of $S_3$. Recall that $T\subseteq E$, thus $T\cap A_1\subseteq E_1$. Towards a contradiction, suppose $h\notin E_1$.
            We proceed by case distinction.

            \underline{Case 3.i: $T\cap A_1=\emptyset$.} In this case, $T\in \mathcal{D}_{S_3}(E_1)$. By construction of the S-reduct, $(T,*_1)\in S_2^\star$. Since $T\subseteq E_2$ and $E_2$ is closed in $F_2^\star$, we have $*_1\in E_2$. Since $(\emptyset,*_1)\in R_2^\star$, $E_2$ is attacked by the empty set and thus not admissible in $F_2^\star$; contradiction to our initial assumption. 
            
            \underline{Case 3.ii: $T\cap A_1\neq \emptyset$.} In this case, $T\in \mathcal{T}_{S_3}(E_1)\setminus \mathcal{D}_{S_3}(E_1)$. 
            By construction of the S-reduct, $(T\cap A_2,*_2)\in S_2^\star$. Since $T\cap A_2\subseteq E_2$ and $E_2$ is closed in $F_2^\star$, we have $*_2\in E_2$. Again, we can derive a contradiction to admissibility of $E_2$ in $F_2^\star$: By construction of the S-reduct, $F_2^\star$ contains the attack $((T\cap A_2)\cup \{*_2\},*_2)$, thus $E_2\notin \cf(F_2^\star)$. 
        \end{itemize}
            In each case, we have $h\in E$; therefore, $E=cl_F(E)$. 
        \item Assume $E$ is admissible in $F$. Let $E_1=E\cap A_1$ and $E_2=E\cap A_2$. We show that $E_1$ is closed in $F_1$ and $E_2$ is closed in $F\redEone_2$.
        \begin{itemize}
            \item Let $E_1=E\cap A_1$. We prove $E_1=cl_{F_1}(E_1)$. Suppose towards contradiction that $E_1$ supports some $a\in A_1$ in $F_1$, but $a\notin E_1$. Then for some $T\subseteq E_1$ and $a\in A_1\setminus E_1$ we have $(T,a)\in S_1$. Since $E_1\subseteq E$ and $S_1\subseteq S$, we derive that $E$ supports some $a\in A\setminus E$ in $F$. Thus, $E\neq cl_F(E)$, in contradiction with our hypothesis.
            
            \item Let $E_2=E\cap A_2$. We show $E_2=cl_{F_2\redEone}(E_2)$. 
            Towards a contradiction, suppose that there is some $T\subseteq E_2$ such that $(T,h)\in S\redEone_2$, but $a\notin E_2$. We proceed by case distinction.
            \begin{itemize}
                \item \underline{Case 1: $(T,h)\in S_2$.} Then $T\subseteq E$, contradiction to $E=\cl_F(E)$.
                \item \underline{Case 2: $h = *_1$.} Then $T\in \mathcal{D}_{S_3}(E_1)$. Therefore, there is is some $a\in A_1$ with $(T,a)\in S_3$ and $a\notin E_1$. We obtain a contradiction to $E$ being closed in $F$ since $T\subseteq E$ but $a\notin E$ (as $a\notin E_1$). 
                \item \underline{Case 3: $h = *_2$.} Then the support is of the form $(T,h)=(T'\cap A_2,*_2)$ for some $T'\supseteq T$ with $T'\in \mathcal{T}_{S_3}(E_1)$. Therefore, there is is some $a\in A_1$ with $(T',a)\in S_3$ and $a\notin E_1$.
                It holds that $T'\subseteq E$ since $T'\cap A_1\subseteq E_1$ and $T=T'\cap A_2\subseteq E_1$. Consequently, in $F$, $E$ supports $a$ but $a\notin E$, contradiction to $E=\cl_F(E)$. \qedhere
            \end{itemize}
        \end{itemize}
    \end{enumerate}
\end{proof}

\Needspace{8\baselineskip}%

\SuppSplit*
\begin{proof}
    Similarly to Theorem \ref{thrm:AttSplit}, we prove the statements for each semantics. 

\paragraph{Admissible Semantics.}
\begin{enumerate}
    \item Suppose $E_1\in \adm(F_1)$ and $E_2\in \adm(F\redEone_2)$.
        We show $E=E_1\cup E_2\setminus \{*_2\}\in \adm(F)$.
        \begin{itemize}
            \item $E$ is closed in $F$: this follows from Lemma \ref{lem: closure supp}.
            \item $E$ is conflict-free: this follows from Lemma \ref{lem:conflict-freeness support splitting}. 
            \item $E$ defends itself against each closed set:
            Consider a set $T_c\subseteq A$ which is closed in $F$, i.e., $T_c=\cl_F(T_c)$, and attacks $E$ in $F$. 

            Then there is $T\subseteq T_c$, $h\in E$, such that $(T,h)\in R$. It holds that either $(T,h)\in R_1$ or $(T,h)\in R_2$. 
            \begin{itemize}
                \item \underline{Case 1: $(T,h)\in R_1$.} Then $h\in E_1$ and $T\subseteq A_1$.
                Let $T_1=\cl_{F_1}(T)$. Observe that $T_1\subseteq T_c$ since $S_1\subseteq S$, moreover, $T_1$ attacks $h$ in $F_1$.
                Since $E_1$ is admissible in $F_1$, it holds that $E_1$ attacks $T_1$. We obtain $E$ defends itself against $T_c$.
                
                \item \underline{Case 2: $(T,h)\in R_2$.} Then $h\in E_2$ and $T\subseteq A_2$. 
                Let $T_2=\cl_{F_2}(T)$. We first show the following:
                \bigskip
                
                \textbf{Claim 1.}\ 
                \textit{For all $a\in A_2$, $a\in T_2$ implies $a\in T_c$.} 
                \medskip

                \textit{Proof.}\ 
                First, we observe that for all $a\in A_2$, $a\in T_2$ iff either $a\in T$ or there is $(B,a)\in S_2$ with $B\subseteq T_2$ (by construction, the head of each newly introduced support in $F_2$ is contained in $\{*_1,*_2\}$).
                Moreover, for every $(B,a)\in S_2\redEone$, it holds that $\{*_1,*_2\}\cap B = \emptyset$.
                Consequently, $a\in A_2$ is supported by $T_2$ in $F_2\redEone$ implies either $a\in T$ or there are supports $(B_1,b_1),\dots,(B_k,b_k)\in S_2$ with $b_k=a$ and $B_i\subseteq T\cup \bigcup_{j<i} b_j$. In the former case, $a\in T_c$. 
                In the latter case, we observe that this support chain also exists in $F$, and since $T\subseteq T_c$, we obtain $a\in T_c$ whenever $a\in T_2$.
                \hfill $\Diamond$%
                \bigskip

                We proceed by case distinction. We distinguish the cases $*_1\in T_2$, $*_2\in T_2$, and $T_2\cap \{*_1,*_2\}=\emptyset$.
                \medskip

                \underline{Case 2.i: $*_1\in T_2$.} In this case, $E_2$ is defended against $T_2$ since $T_2$ is attacked by the empty set. 
                
                By construction, there is $(T',*_1)\in S_2\redEone$ with $T'\subseteq T_2$, and $T'\in \mathcal{D}_{S_3}(E_1)$.
                Thus it holds that $(T',a)\in S_3$, $a\in A_1$, $a\notin E_1$, and there is $b\in \cl_F(T')$ such that $b\in E_1^+$.
                Since $T'\subseteq A_2$, we obtain by Claim 1,  $T'\subseteq T_c$.
                Thus, $\cl_F(T')\subseteq T_c$. Consequently, $b\in T_c$, and therefore, $E_1\subseteq E$ defends $E$ against the attack from $T_c$.

                \underline{Case 2.ii:  $*_2\in T_2$.}
                Since $E_2$ is admissible in $F_2\redEone$, there is some $b\in T_2$ such that $b\in (E_2)_{F_2\redEone}^+$.
                Two further cases:
                \begin{itemize}
                \item \underline{Case 2.ii.a: $b=*_2$.} Then there is $T'\in \mathcal{T}_{S_3}(E_1)\setminus  \mathcal{D}_{S_3}(E_1)$ with  $T'\subseteq E_2$ and $((T'\cap A_2)\cup \{*_2\},*_2)\in R_2\redEone$.
                Contradiction to conflict-freeness of $E_1$.
                \item \underline{Case 2.ii.b: $b\neq *_2$.} Then $b\in A_2$. There is an attack $(U,b)\in R_2$ and $U\subseteq E_2$ (since $b\in A_2$, the attack has not been newly introduced in $F_2\redEone$). By Claim 1, $b\in T_c$.
                Since $U\subseteq E_2\cap A_2\subseteq E$ we obtain $E$ defends itself against $T_c$.
                \end{itemize}

                \underline{Case 2.iii: $T_2\cap \{*_1,*_2\}=\emptyset$.} In this case, $T_2\subseteq T_c$ by Claim 1. Since $E_2$ is admissible in $F_2$, it defends itself against $T_2$. Thus, there is $(U,a)\in R_2$ with $a\in T_2$, $U\subseteq E_2$. We obtain that $E$ defends itself against $T_c$ in $F$.
            \end{itemize}
        \end{itemize}
        \item Suppose $E\in \adm(F)$. We show that $E_1=E\cap A_1\in \adm(F_1)$ and $E_2=E\cap A_2\in \adm(F\redEone_2)$.
        \begin{itemize}
            \item Let $E_1=E\cap A_1$. $E_1$ is conflict-free by Lemma \ref{lem:conflict-freeness support splitting} and closed by Lemma \ref{lem: closure supp}.
            It remains to prove that $E_1$ defends itself.
            
            Let $T_c$ be a closed attacker of $E_1$. Then there is $T\subseteq T_c$, $h\in E_1$, such that $(T,h)\in R_1$. Observe that $\cl_F(T)\subseteq A_1$ since $T\subseteq A_1$ and there is no support $(B,a)\in S$ with $B\subseteq A_1$ and $a\in A_2$; all supports either stay in the respective sets $A_1$ or $A_2$ or go 'the other direction'. 

            By assumption, $E$ defends itself against the closed attacker $\cl_F(T)$. Thus, there is an attack $(U,a)\in R_1$ with $U\subseteq E$ and $a\in \cl_F(T)$. By definition of support splitting $U\subseteq A_1$, thus $U\subseteq E_1$. We obtain that $E_1$ defends itself against $\cl_F(T)$. Since $T\subseteq T_c$, we obtain $\cl_F(T)\subseteq T_c$. Thus, $E_1$ defends itself against $T_c$. 
            \item Let $E_2=E\cap A_2$. $E_2$ is conflict-free by Lemma \ref{lem:conflict-freeness support splitting} and closed by Lemma \ref{lem: closure supp}.
            It remains to prove that $E_2$ defends itself.

            Let $T_c$ denote a closed attacker of $E_2$ in $F\redEone_2$. Then there is $T\subseteq T_c$, $h\in E_2$, such that $(T,h)\in R_2$ (observe that $(T,h)\notin R_2\redEone\setminus R_2$ because $h\in A_2$). 
            Therefore, $T$ attacks $E$ in $F$. Due to admissibility of $E$ in $F$, there is some $(U,a)\in R$ with $U\subseteq E$ and $a\in \cl_F(T)$. Note that $a\in A$.

            \underline{Case 1:  $(U,a)\in R_1$.} Then $U\subseteq A_1$ and $a\in A_1$.
            Let us inspect the set $\cl_F(T)$: First, we observe that $T\subseteq A_2$, but $\cl_F(T)\cap A_1\neq \emptyset$. Thus, there is some $T'\subseteq T$ such that $(T',c)\in S_3$. Wlog, we can assume that 
            \begin{itemize}
                \item $c\notin E_1$ (otherwise, if for all $d\in (\cl_F(T)\setminus T)\cap A_1$, it holds that $d\in E_1$, we obtain $a\in E_1$; then $E_1$ attacks itself via the attack $(U,a)$, contradiction to conflict-freeness of $E_1$); and
                \item $a\in \cl_{F}(T')$ (otherwise, if there is no $T''\subseteq T$ such that $a\in \cl_F(T'')$ then $a\notin \cl_F(T)$, contradiction to our assumption). 
            \end{itemize}
            In summary, it follows that $T'\in \mathcal{D}_{S_3}(E_1)$: $(T',c)\in S_3$, $c\notin E_1$, $T\cap A_1=\emptyset$, and $\cl_F(T')$ is attacked by $E_1$ (via $a$). Consequently, $(T',*_1)\in S_2\redEone$ and $(\emptyset,*_1)\in R_2\redEone$. Therefore, the set $\cl_{F_2\redEone}(T')$ is attacked by the empty set. Since $T'\subseteq T$ we have $\cl_{F_2\redEone}(T')\subseteq \cl_{F_2\redEone}(T)$ and thus $\cl_{F_2\redEone}(T')\subseteq T_2$. In other words, $E_2$ defends itself against $T_2$ in $F_2\redEone$.
            
            \underline{Case 2: $(U,a)\in R_2$.} Then $U\subseteq A_2$ and $a\in A_2$.
            Observe that $\cl_F(T)\cap A_2\subseteq T_c$ since $S_2\subseteq S_2\redEone$ and no set of arguments in $A_1$ supports an argument in $A_2$. 
            Therefore, $U\subseteq E_1$ attacks $a\in T_c$ in $F_2\redEone$. We have shown that $E_2$ defends itself against the closed attacker. 
            
        \end{itemize}
\end{enumerate}

    \paragraph{Stable semantics.}
    
    \begin{enumerate} 
        \item Suppose $E_1\in \stb(F_1)$ and $E_2\in \stb(F\redEone_2)$.
        We show that $E=E_1\cup (E_2\setminus \{*_2\})\in \stb(F)$. 
        \begin{itemize}
            \item $E$ is conflict-free: 
            By Lemma \ref{lem:conflict-freeness support splitting} and since $E_1\in \stb(F_1)$ and $E_2\in \stb(F_2\redEone)$, we obtain $E\in \cf(F)$.
            \item $E$ is closed because of Lemma \ref{lem: closure supp} and $\stb(F)\subseteq \adm(F)$ for every BSAF $F$.
            \item $E$ attacks all remaining arguments: 
            Let $a\in A\setminus E$. We proceed by case distinction.
            \begin{itemize}
                \item \underline{Case 1: $a\in A_1$.} Then $a\in (E_1)^+_{R_1}$ since $E_1\in \stb(F_1)$, thus $a\in E^+_{R}$.
                \item \underline{Case 2: $a\in A_2$.} Since $E_2\in \stb(F\redEone_2)$, we derive that $a\in (E_2)^{+}_{R\redEone_2}$. By definition of $S$-reduct, $R_2\redEone$ augments $R_2$ with an attack towards $*_1$ and (possibly) multiple attacks towards $*_2$, i.e.\  $R\redEone_2\subseteq R_2\cup \{(\emptyset,*_1)\} \cup \{((T\cap A_2)\cup \{*_2\}, *_2)\}$ for possibly some $T\subseteq A_2$. For this, it is immediate to see that no attack towards any $a\in A_2$ is added via the $S$-reduct. Therefore, we derive that $a\in (E_2)^{+}_{R_2}$ and, consequently, $a\in E^+_{R}$.
            \end{itemize}
            In all cases we get $a\in E^+_R$, therefore  $E\in \stb(F)$.
        \end{itemize}
    \item Suppose $E\in \stb(F)$. It holds that $E^\oplus_R=A=A_1\cup A_2$. 
    \begin{itemize}
        \item We show that $E_1=E\cap A_1\in \stb(F_1)$. From Lemma \ref{lem:conflict-freeness support splitting}, we obtain $E\cap A_1\in \cf(F_1)$. 
        $E_1$ is closed in $F_1$ because of Lemma \ref{lem: closure supp}.
        Finally, we show that $E_1$ attacks all arguments $a\in A_1\setminus E_1$. 
        Let $a\in A_1$. Since $(F_1,F_2,S_3)$ is a support splitting of $F$, it holds that $R$ is partitioned into $R_1$ and $R_2$. Thus, $a$ is attacked by some $T\subseteq A_1$, and $(E\cap A_1)^\oplus_{R_1} =A_1$. We then derive that $E_1\in \stb(F_1)$.
        \item We show that $E_2=E\cap A_2\in \stb(F_2\redEone)$.
        By Lemma \ref{lem:conflict-freeness support splitting}, we obtain $E_2\in \cf(F\redEone_2)$. 
        $E_2$ is closed in $F\redEone_2$ because of Lemma \ref{lem: closure supp}.
        It remains to prove that $E_2$ attacks all remaining arguments in $F\redEone_2$, i.e.\ $(E_2)^+_{R_2\redEone}=A_2\redEone\setminus E_2$. Let $a\in A_2\redEone\setminus E_2$. By definition of $S$-reduct, we can distinguish two cases. 
        \begin{itemize}
            \item \underline{Case 1.} $a\in A_2$. For this, notice that $R$ is partitioned into $R_1$ and $R_2$ by definition of support splitting $(F_1,F_2,S_3)$. Thus, $a$ is attacked by some $T\subseteq A_2$, i.e.\ $a\in (E_2)_{R_2}^+$. Further, no attack is removed when computing the reduct $F_2\redEone$. Thus, we can derive $a\in (E_2)_{R\redEone_2}^+$. 
            \item \underline{Case 2.} $a\notin A_2$. Thus, $a\in \{*_1,*_2\}$. By definition of $S$-reduct, we know that $*_1\in (E_2)_{R\redEone_2}^+$ because $*_1\in (\emptyset)_{R\redEone_2}^+$. Consider now the argument $*_2$. If $E\cap A_2$ attacks $*_2$, then $*_2\in E\cap A_2$ in contradiction with conflict-freeness. Therefore, $(E\cap A_2) \cup \{*_2\}\in \stb(F_2\redEone)$. 
        \end{itemize}     
    \end{itemize} 
    \end{enumerate}
    
\paragraph{Complete Semantics.}
\begin{enumerate}
    \item Suppose $E_1\in \com(\BF_1)$ and $E_2\in \com(\BF\redEone_2)$.
        We show $E=E_1\cup E_2\setminus \{*_2\}\in \com(\BF)$.
        \begin{itemize}
            \item The premisses $E_1\in \com(\BF_1)$ and $E_2\in \com(\BF\redEone_2)$ directly entail $E_1\in \adm(\BF_1)$ and $E_2\in \adm(\BF\redEone_2)$, hence $E\in\adm(\BF)$ as shown above.
            \item It remains to be shown that $E$ defends itself in $\BF$.
            \begin{itemize}
                \item \underline{Case 1: $\mathcal{T}_{S_3}(E_1)=\emptyset$.}
                Then $F_2\redEone=F_2$, and for any $\tuple{T,h}\in S_3$: $h\in E\vee T\subsetneq E$.
                Lets assume $E$ defends $a\in A\setminus E$ in $F$. If $a\in A_1$ then $E_1$ defends $a$ in $F_1$, contradiction to $E_1\in\com(F_1)$. On the other hand, if $a\in A_2$ then $E_2$ defends $a$ in $F_2$, contradiction to $E_2\in\com(F_2)$. Thus $E$ defends itself in $F$. 
                \item \underline{Case 2: $\mathcal{T}_{S_3}(E)\neq\emptyset$.} Toward contradiction, we assume there exists an $a$ that is defended by $E$ in $\BF$, s.t. $a\notin E$. Let $T\subseteq A$ be a minimal closed set that $c$ is defended against by $E$ in $\BF$, i.e. $\exists T'\subseteq T:\tuple{T',c}\in R$ and $\exists t\in T\exists T_{E}\subseteq E: \tuple{T_{E},t}\in R$.
                \begin{itemize}
                    \item \underline{Case 2i: $a\in A_1$.} Then $E_1$ defends $a$ in $A_1$, which contradicts $E_1\in\com(\BF_1)$.
                    \item \underline{Case 2ii: $a\in A_2$ and $\cl_{\BF_2}(T)\cap\{*_1,*_2\}=\emptyset$.}
                        Then, $\cl_{\BF_2}(T)$ attacks $a$ and $E_2$ attacks $\cl_{\BF_2}(T)$ in $\BF_2\redEone$. Hence $E_2$ is not complete in $\BF_2\redEone$, which contradicts $E_2\in \com(\BF\redEone_2)$.
                    \item \underline{Case 2iii: $a\in A_2$ and $*_1\in\cl_{\BF_2}(T)$.} Then $T$ is attacked by $E_2$ in $\BF_2$, more specifically, $\emptyset\subseteq E_2$ attacks $\{*_1\}\subseteq \cl_{\BF_2}(T)$. Hence $E_2$ is not closed in $\BF_2$, which contradicts $E_2\in\com(\BF_2)$.
                    \item \underline{Case 2iv: $a\in A_2$ and $*_2\in\cl_{\BF_2}(T)$.} Then $E$ does not attack a $t\in T\cap A_1$, but attacks a $t\in T\cap A_2$. Hence $E_2$ attacks $t\in T\cap A_2$, thus $E_2$ is not closed in $\BF_2$, contradiction!
                \end{itemize}
            \end{itemize}
        \end{itemize}
        \item Suppose $E\in \com(F)$. We show that $E_1=E\cap A_1\in \com(F_1)$ and ($E\cap A_2\in \com(F\redEone_2)$ or $(E\cap A_2)\cup \{*_2\}\in \com(F\redEone_2)$).
        \begin{itemize}
            \item Utilizing the proof for admissible semantics, we directly conclude $E_1=E\cap A_1\in \adm(F_1)$ and $E_2=E\cap A_2\in \adm(F\redEone_2)$ or $(E\cap A_2)\cup \{*_2\}\in \adm(F\redEone_2)$).
            \item There are no attacks from $A_2$ towards $A_1$, neither are there supports from $A_1$ towards $A_2$. Hence any argument in $A_1$ is defended by $E$ in $F$ iff it is defended by $E_1=E\cap A_1$ in $F_1$. Thus $E_1$ defends itself in $F_1$, therefore $E_1=E\cap A_1\in \com(F_1)$.
            \item It remains to be shown that $E_2$ defends itself in $F\redEone_2$.
            We assume toward a contradiction that there exists a $c\in A_2\setminus E_2$ s.t. $c$ is defended by $E_2$ in $F_2\redEone$.
            \begin{itemize}
                \item \underline{Case 1: $\mathcal{T}_{S_3}(E)=\emptyset$.} Then $F_2\redEone=F_2$, and for any $\tuple{T,h}\in S_3$: $h\in E$. Thus any argument $c\in A_2$ defended by $E_2$ in $F\redEone_2$ is defended by $E$ in $F$. The statement $c\notin E_2$ entails $c\notin E$. This contradicts the premise that $E$ is complete.
                \item \underline{Case 2: $\mathcal{T}_{S_3}(E)\neq\emptyset$.} Let $T\subseteq A_2$ be a closed set that $c$ is defended against by $E_2$ in $F_2\redEone$, i.e. $\exists T'\subseteq T:\tuple{T',c}\in R\redEone_2$ and $\exists t\in T\exists T_{E_2}\subseteq E_2: \tuple{T_{E_2},t}\in R\redEone_2$.
                \begin{itemize}
                    \item \underline{Case 2i: $\cl(T\setminus\{*_1,*_2\})\cap\{*_1,*_2\}=\emptyset$.} Then, by the nature of the construction of $F_2\redEone$, $T\setminus\{*_1,*_2\}$ does not support any $a\in A_1\setminus E_1$ in $F$. If $T$ supports an $a\in E_1$, then $E$ attacks $a$, hence $E$ is not conflict-free, and thus not a complete extension, contradiction.
                    
                    Otherwise, in $F$, $\cl(T\setminus\{*_1,*_2\})$ is closed, attacks $c$ and is attacked by $E\supseteq E_2\setminus\{*_2\}$.
                    \item \underline{Case 2ii: $*_1\in \cl(T\setminus\{*_1,*_2\})$.} Then by the construction of $F_2\redEone$ we can infer that there is a subset $T'\subseteq T\setminus\{*_1,*_2\}$ s.t. $T'\in\mathcal{T}_{S_3}(E_1)\setminus\mathcal{D}_{S_3}(E_1)$. Thus $\exists b\in\cl_\BF(T): b\in E^+$, meaning that either $T$ is not closed in $\BF$, or $E$ defends $c$ against $T$ in $\BF$. This contradicts $E\in\com(\BF)$.
                    \item \underline{Case 2iii: $*_2\in \cl(T\setminus\{*_1,*_2\})$.} Then by the construction of $F_2\redEone$ we can infer that there is a subset $T'\subseteq T\setminus\{*_1,*_2\}$ s.t. $T'\in\mathcal{D}_{S_3}(E_1)$.

                    Thus in $F$, $\exists b\in\cl_F(T\setminus\{*_2\})$, s.t. $E$ attacks $c$. We thus infer that $E$ defends $c$ from $\cl_F(T)$. This contradicts the premise that $E$ is complete.
                \end{itemize}
            \end{itemize}
        \end{itemize}
\end{enumerate}
\end{proof}

\PrefSuppSplit*
\begin{proof}
We prove the statement for each $\sigma\in\{\grd,\pref\}$.
\paragraph{Grounded Semantics.}
Suppose $E_1\in \grd(\BF_1)$ and $E_2\in \grd(\BF\redEone_2)$. Then $E_1\in \com(\BF_1)$ and $E_2\in \com(\BF\redEone_2)$, thus by Theorem~\ref{thm:SuppSplit}: $E=E_1\cup E_2\setminus\{*_2\}\in\com(\BF)$. Assuming toward contradiction $\exists G\subsetneq E:G\in\com(\BF)$, then $\exists G_1\subsetneq E_1:G_1\in\com(\BF_1)$ or $\exists G_2\subsetneq E_2:G_2\in\com(\BF_2\redEone)$. Contradiction! Thus $E\in\grd(\BF)$.

\paragraph{Preferred Semantics.}
Suppose $E_1\in \prf(F_1)$ and $E_2\in \prf(F\redEone_2)$. Since $E_1$ and $E_2$ are admissible in $F_1$ and $F_2\redEone$, respectively, we obtain $E\in \adm(F)$.
    We show that $E=E_1\cup E_2\setminus \{*_2\}\in \prf(F)$.

    Towards a contradiction, suppose $E\notin \prf(F)$. 
    Then there is $E'\supsetneq E$ with $E'\in \adm(F)$.
    Then $E_1'=E'\cap A_1\in \adm(F_1)$ and $E_2'=E'\cap A_2\in \adm(F_2\redEone)$. 
    $E'$ is strictly larger than $E$. Proceed with case distinction.
    
    \underline{Case 1:} $E_1'$ is strictly larger than $E_1$. This is in contradiction to $E_1\in \prf(F_1)$.
    
    \underline{Case 2:} $E_2'$ is strictly larger than $E_2\setminus \{*_2\}$.
    
    \underline{Case 2.i: $*_2\in E_2$.} Let $E''_2=E_2\cup \{*_2\}$.
    If $E_2''\in \adm(F_2\redEone)$ then $E''_2\supsetneq E_2'$, contradiction to $E_2\in \prf(F_2)$.

    We show that $E_2''\in \adm(F_2\redEone)$:
    $E_2''$ is closed since $E_2'$ is closed and $*_2\notin U$ for any $(U,b)\in S_2\redEone$. 
    $E_2''$ defends $*_2$ since $E_2\subseteq E_2''$ defends $*_2$.
    $E_2''$ is conflict-free: Towards a contradiction, suppose there is an attack $(T,h)\in R_2\redEone$ such that $T\subseteq E_2''$ and $h\in E_2''$. 

    \underline{Case 2.i.a: $h = *_2$.} By construction, $T$ has the form $((T'\cap A_2)\cup \{*_2\},*_2)$ for some $T'\in \mathcal{T}_{S_3}(E_1)$.  Thus, $T'\cap A_2\subseteq E_2'$. 
    Since $(T'\cap A_2,*_2)\in S_2\redEone$, this implies $E_2'$ is not closed. Contradiction to $E_2'\in \adm(F_2\redEone)$.

    \underline{Case 2.i.b: $h \in A_2$.} In this case, $T\subseteq E_2'$ and we have a contradiction to the conflict-freeness of $E_2'$.

    \underline{Case 2.ii: $*_2\notin E_2$.} Then $E_2'\supsetneq E_2'$,
    contradiction to $E_2\in \prf(F_2)$.
\end{proof}

\clearpage
\section{Omitted Proofs of Section \ref{sec:combined splitting}}
We recall the notation from Definition \ref{def:combined splitting procedure}. Throughout this section, we will make use of the following conventions. 

Let $(F_1,F_2,R_3,S_3)$ be a splitting for a BSAF $F$ with $F_1=(A_1,R_1,S_1)$, $F_2=(A_2,R_2,S_2)$. 
    Let $R_3\clink$ denote the set of closed attacks wrt.\ $R_3$ and
    $R_2\clink$ denote the set of closed attacks wrt.\ $R_2$.
    Let $\hat R_2=\{(T,h)\in R_2\clink\mid T\cap A_1=\emptyset\}$
    and let $\hat F_2=(A_2,\hat R_2,S_2)$ denote the updated BSAF.
    Further, let $$\hat R_3=\{(T,h)\mid (T,h)\in R_3\clink\cup  R_2\clink, T\cap A_1\neq \emptyset\}.$$
    Let $\hat F\redEone_2$ be the R-reduct of the attack splitting $(F_1,\hat F_2,\hat R_3)$ wrt.\ $E_1$ and $F^\star_2=mod^{E_1}_{\hat R_3}(\hat F\redEone_2)$.
    
    We let $F_2^\circledast=\tuple{A_2^\circledast,R_2^\circledast,S_2^\circledast}$ denote the S-reduct for the support splitting $(F_1,F^\star_2,S_3)$ wrt.\ $E_1$.

\begin{lemma}\label{lem:conflict-freeness attack support splitting}
		Let $(F_1,F_2,R_3,S_3)$ be a splitting for a BSAF $F=(A,R,S)$ with $F_1=(A_1,R_1,S_1)$, $F_2=(A_2,R_2,S_2)$.
    Let $F^\circledast_2$ be defined as in Definition~\ref{def:combined splitting procedure}.
		\begin{enumerate}
			\item If $E_1\in \adm(F_1)$ and $E_2\in \adm(F_2^\circledast)$, then $E=E_1\cup (E_2\setminus\{*_2\})\in \cf(F)$. 
			\item If $E\in \adm(F)$, then $E_1=E\cap A_1\in \cf(F_1)$ and $E\cap A_2\in \cf(F^\circledast_2)$.
		\end{enumerate}
\end{lemma}
\begin{proof}
\begin{enumerate}
    \item Let $E_1\in \adm(F_1)$ and $E_2\in \adm(F_2^\circledast)$. We show that $E=E_1\cup (E_2\setminus\{*_2\})\in \cf(F)$.
    Let $(T,h)\in R$ s.t. $T\subseteq E$. Then either $(T,h)\in R_1$, $(T,h)\in R_2$, or $(T,h)\in R_3$.
    \begin{itemize}
            \item \underline{Case 1: $(T,h)\in R_1$.} Then $h\notin E$ since $E_1\in \adm(F_1)$. 
            \item \underline{Case 2: $(T,h)\in R_2$.} It holds that $T\subseteq A_2,h\in A_2$. 
            \begin{itemize}
                \item \underline{Case 2i: $\cl(T)\subseteq A_2$.} Then $(\cl(T),h)\in R_2^\circledast$.
                \item \underline{Case 2ii: $\cl(T)\cap A_1\neq\emptyset$.} Then, $T\cap E^+=\emptyset$ entails $((\cl(T)\setminus A_1),h)\in R_2^\circledast$
            \end{itemize}
            In both cases we have $\cl(T)\subseteq E_2$ and $h\notin E_2$ since $E_2\in \adm(F_2^\circledast)$. Thus $h\notin E$.
            \item \underline{Case 3: $(T,h)\in R_3$.} Let $T_c=\cl_F(T)$. 
            Then either $(T_c,h)\in U_{R_3\clink}^{E_1}$ or $(T_c,h)\notin U_{R_3\clink}^{E_1}$. 
            \begin{itemize}
                \item \underline{Case 3.i: $(T_c,h)\in U_{R_3\clink}^{E_1}$.} 
                Then $(T_c\cap A_2)\cup \{*_2\}$ attacks $h$ in $F_2^\circledast$.
                Thus, $h\notin E_2$ and therefore $h\notin E$.

                \item \underline{Case 3.ii: $(T_c,h)\notin U_{R_3\clink}^{E_1}$.}

                From $(T_c,h)\notin U_{R_3\clink}^{E_1}$, we get that 
                $T_c\cup (E_1)^+_{R_1\cup R_3\clink}\neq \emptyset$ or 
                $T_c\cap A_1 \subseteq (E_1)^\oplus_{R_1}$. 
                This amounts to: $T_c\cap (E_1)^+_{R_1}\neq \emptyset$ or $T_c\cap (E_1)^+_{R_3\clink}\neq \emptyset$ or
                $T_c\cap A_1 \subseteq E_1$.
                
                We proceed by case distinction. 
                
                (a) $T_c\cap (E_1)^+_{R_1}\neq \emptyset$. That is, there is some $t\in T_c$ which is attacked by $E_1$ via some attack $(T',t)$ in $R_1$. Since $(T',t)\in R_1$ then $t\in A_1$. We obtain a contradiction to $E_1\in \adm(F_1)$. 

                (b) $T_c\cap (E_1)^+_{R_3\clink}\neq \emptyset$. That is, there is some $t\in T_c$ which is attacked by $E_1$ via some attack $(T',t)$ in $R_3$. Then $t\in A_2$ and $t\in (E_1)^+_{R_3\clink}$. 
                By assumption, $T'\subseteq E_1$ and $t\in A_2$. 
                Thus, $T'\cap (E_1)^+_{R\clink_3}=\emptyset$. Then, by definition of the R-reduct, $R_2^\circledast$ contains $(\emptyset,t)$. Since $t\in E$, we obtain contradiction to the conflict-freeness of $E_2$ in $F_2^\circledast$. 

                (c) If $T\cap A_1 \subseteq E_1$, we obtain a contradiction to conflict-freeness of $E_1$.
            \end{itemize}
    \end{itemize}
    \item Suppose $E\in \adm(F)$. We show that $E_1=E\cap A_1\in \cf(F_1)$ and $E_2=E\cap A_2\in \cf(F^\circledast_2)$.
    \begin{itemize}
        \item $E_1\in \cf(F_1)$ since $R_1\subseteq R$.
        \item We show that $E_2\in \cf(F^\circledast_2)$.
            Let $(T,h)\in R^\circledast_2$ with $T\subseteq E_2$.  
            By definition of $E_2$, we have $h\in A_2$. Thus, 
            by definition of $R^\circledast_2$, one of the following applies for $(T,h)$:
             \begin{itemize}
                \item $(T,h)\in \hat R_2$ and $T\subseteq A_2,h\in A_2$; \hfill (Case 1)
                \item $(T,h)=(T'\setminus A_1, h)$ for some $T'\subseteq A$ such that $(T',h)\in \hat R_3$, $T'\cap (E_1)^+_{R_1\cup \hat R_3}=\emptyset$, $T\cap A_1\subseteq E_1$, and $h\in A_2$; \hfill (Case 2)
                \item $(T,h)=((T'\cap A_2) \cup \{*_0\}, h)$ for some $(T',h)\in U^{E_1}_{\hat R_3}$. \hfill (Case 3)
            \end{itemize}
        We proceed by case distinction.
   \begin{itemize}
       \item \underline{Case 1.} $(T,h)\in \hat R_2$ with $T\subseteq A_2,h\in A_2$. By definition of $\hat R_2$, we know that $(T,h)\in R_2\redEone$ and $T\cap A_1=\emptyset$. Thus, there is some $T'\subseteq T\subseteq A_2$ such that $T=cl_S(T')$ and $(T',h)\in R_2$. Thus  $h\notin E_2$ since $E\in \cf(F)$.
       \item \underline{Case 2.} 
       From $T=T'\setminus A_1 \subseteq E_2$ and $T'\cap A_1\subseteq E_1$ we obtain $T'\subseteq E$. By definition of $\hat R_3$, there is $T''\subseteq T'$ such that $(T'',h)\in R_3$ or $(T'',h)\in R_2$.
       We obtain $h\notin E_2$ since $E\in \cf(F)$ and since $R_3\subseteq R$.
       \item \underline{Case 3.} It follows that $*_0\in E_2\subseteq E$, in contradiction to $E\subseteq A$.
   \end{itemize} 
   We have shown that $E_2\in \cf(F^\circledast_2)$.\qedhere
         \end{itemize}
\end{enumerate}
\end{proof}

\Needspace{8\baselineskip}%
\begin{lemma}\label{lem:closure attack support splitting}  
  Let $(F_1,F_2,R_3,S_3)$ be a splitting for a BSAF $F=(A,R,S)$ with $F_1=(A_1,R_1,S_1)$, $F_2=(A_2,R_2,S_2)$.
    Let $F^\circledast_2$ be defined as in Definition~\ref{def:combined splitting procedure}.
		\begin{enumerate}
		\item If $E_1\in \adm(F_1)$ and $E_2\in \adm(F_2^{\circledast})$, then $E= \cl_F(E)$ where $E=E_1\cup (E_2\setminus\{*_2\})$. 
		\item If $E\in \adm(F)$, then $E\cap A_1=E_1=cl_{F_1}(E_1)$ and $E_2=cl_{F_2^{\circledast}}(E_2)$ where $E_2=E\cap A_2$.
		\end{enumerate}
\end{lemma}
\begin{proof} 
We recall: $F_1=(A_1,R_1,S_1)$ and $F_2^\circledast=(A_2^\circledast, R_2^\circledast,S_2^\circledast)$ with $A_2^\circledast=A_2\cup \{*_0,*_1,*_2\}$ and $R_2^\circledast$, $S_2^\circledast$ as defined by the respective modifications for support and attack splitting. 
    \begin{enumerate}
        \item $E_1\in \adm(F_1)$ and $E_2\in \adm(F_2^{\circledast})$. We show that $E=E_1\cup (E_2\setminus\{*_2\})$ is closed in $F$. The proof is similar to the case for support splitting. 

        Let $(T,h)\in S$ with $T\subseteq E$. It holds that either $(T,h)\in S_1$, $(T,h)\in S_2$, or $(T,h)\in S_3$. We proceed by case distinction.

        \begin{itemize}
            \item \underline{Case 1: $(T,h)\in S_1$.} Then $T\subseteq E_1$. Since $E_1$ is closed in $F_1$ by assumption, we get $h\in E_1\subseteq E$.
            \item \underline{Case 2: $(T,h)\in S_2$.} Then $T\subseteq E_2$. Moreover, $S_2\subseteq S_2^\circledast$. Since $E_2$ is closed in $F_2^\circledast$, we obtain $h\in E_2\subseteq E$.
            \item \underline{Case 3: $(T,h)\in S_3$.} It holds that $h\in A_1$ by definition of $S_3$. Recall that $T\subseteq E$, thus $T\cap A_1\subseteq E_1$. Towards a contradiction, suppose $h\notin E_1$.
            We proceed by case distinction.

            \underline{Case 3.i: $T\in\mathcal{D}_{S_3}(E_3)$.} By construction of the S-reduct, $(T,*_1)\in S_2^\circledast$. Since $T\subseteq E_2$ and $E_2$ is closed in $F_2^\circledast$, we have $*_1\in E_2$. Since $(\emptyset,*_1)\in R_2^\circledast$, $E_2$ is attacked by the empty set and thus not admissible in $F_2^\circledast$; contradiction to our initial assumption. 
            
            \underline{Case 3.ii: $T\notin\mathcal{D}_{S_3}(E_3)$.} In this case, $T\in \mathcal{T}_{S_3}(E_1)\setminus \mathcal{D}_{S_3}(E_1)$. 
            By construction of the S-reduct, $(T\cap A_2,*_2)\in S_2^\circledast$. Since $T\cap A_2\subseteq E_2$ and $E_2$ is closed in $F_2^\circledast$, we have $*_2\in E_2$. Again, we can derive a contradiction to admissibility of $E_2$ in $F_2^\circledast$: By construction of the S-reduct, $F_2^\circledast$ contains the attack $((T\cap A_2)\cup \{*_2\},*_2)$, thus $E_2\notin \cf(F_2^\circledast)$. 
        \end{itemize}
        \item  Let $E\in \adm(F)$. We show that $E\cap A_1=E_1=cl_{F_1}(E_1)$ and $E\cap A_2=E_2=cl_{F_2^{\circledast}}(E_2)$. 
        \begin{itemize}
            \item Let $E_1=E\cap A_1$ and let $(T,h)\in S_1$ with $T\subseteq E_1$. Since $T\subseteq E$ and $S_1\subseteq S$, we obtain $h\in E_1$. Therefore, $E_1 = \cl_{F_1}(E_1)$.
            \item Let $E_2=E\cap A_2$ and let $(T,h)\in S_2^\circledast$ with $T\subseteq E_2$. By construction of the S-reduct, either $(T,h)\in S_2$ or $h = *_1$ or $h=*_2$. Note that splitting attacks (closing links in $R_3$, constructing the R-reduct and the modification) leaves the supports untouched. Thus, the proof is analogous to the case for support splitting.

            Towards a contradiction, suppose $h\notin E_2$. This implies $h\notin E$. Proceed by case distinction.

            \begin{itemize}
                \item \underline{Case 1: $(T,h)\in S_2$.} Then $T\subseteq E$, contradiction to $E=\cl_F(E)$.
                \item \underline{Case 2: $h = *_1$.} Then $T\in \mathcal{D}_{S_3}(E_1)$. Therefore, there is is some $a\in A_1$ with $(T,a)\in S_3$ and $a\notin E_1$. We obtain a contradiction to $E$ being closed in $F$ since $T\subseteq E$ but $a\notin E$ (as $a\notin E_1$). 
                \item \underline{Case 3: $h = *_2$.} Then the support is of the form $(T,h)=(T'\cap A_2,*_2)$ for some $T'\supseteq T$ with $T'\in \mathcal{T}_{S_3}(E_1)$. Therefore, there is is some $a\in A_1$ with $(T',a)\in S_3$ and $a\notin E_1$.
                It holds that $T'\subseteq E$ since $T'\cap A_1\subseteq E_1$ and $T=T'\cap A_2\subseteq E_1$. Consequently, in $F$, $E$ supports $a$ but $a\notin E$, contradiction to $E=\cl_F(E)$. \qedhere
            \end{itemize}
        \end{itemize}
    \end{enumerate}
\end{proof}

\AttSuppSplit*
\begin{proof} We prove the theorem for all considered semantics. 
\paragraph{Admissible Semantics.}
\begin{enumerate}
    \item Suppose $E_1\in \adm(F_1)$ and $E_2\in \adm(F^\circledast_2)$. 
    We show that $E=E_1\cup E_2 \in \adm(F)$.
    By Lemma~\ref{lem:conflict-freeness attack support splitting}, $E$ is conflict-free in $F$; by Lemma~\ref{lem:closure attack support splitting}, $E$ is closed in $F$. It remains to prove that $E$ defends itself in $F$.

    Let $T_c$ denote a closed attacker of $E$ in $F$.
    That is, there is $T\subseteq T_c$ with $(T,h)\in R$ and $h\in E$. Wlog, we can assume $\cl_F(T)=T_c$. 
    
    As usual, we have three cases to consider: $(T,h)\in R_1$, $(T,h)\in R_2$, or $(T,h)\in R_3$.
    In the process of splitting up the attacks, we close all attacks in $R_2$ and $R_3$; the resulting sets are denoted by $\hat R_2$ and $\hat R_3$; we let $\hat F_2$ denote the updated BSAF. Recall that some of the links that originally stem from $R_2$ may now lie in $\hat R_3$. For our attack $(T,h)\in R$, we thus obtain the following adjusted three cases: 
    $(T,h)\in R_1$, $(T_c,h)\in \hat R_2$, or $(T_c,h)\in \hat R_3$.
    
    We proceed by case distinction.
    \begin{itemize}
        \item \underline{Case 1: $(T,h)\in R_1$.} In this case, $T_c\subseteq A_1$ since $T\subseteq A_1$ and there are no supports $(U,b)\in S$ with $U\subseteq A_1$ and $b\in A_2$. We therefore obtain that $E_1$ defends itself against the attack in $F_1$. Thus, $E$ defends itself against $T_c$ in $F$.  

        \item \underline{Case 2: $(T_c,h)\in \hat R_2$.} It holds that $T_c\subseteq A_2$ and $(T_c,h)\in R_2^\circledast$. 
        By assumption, $E_2$ is admissible in $F_2^\circledast$, 
        thus $E_2$ defends itself against the closed attacker $T_c'=\cl_{F^\circledast_2}(T_c)$ in $F_2^\circledast$.
        That is, there is $(U,b)\in R_2^\circledast$ such that $U\subseteq E_2$ and $b\in T_c'$. 

        We have several cases to consider. We go through the different options for $b$.

        \underline{Case 2.i: $b=*_0$.} Then $(U,b)$ is of the form $( \{*_0\},*_0)$. Since $U\subseteq E_2$, it follows that $E_2$ contains the self-attacker $*_0$ and is therefore not admissible, contradiction to our initial assumption.
            
        \underline{Case 2.ii: $b=*_1$.} Then $U=\emptyset$ and $*_1\in T_c'$.
        Then there is some $V\subseteq T_c'$ so that $(V,*_1)\in S^\circledast_2$. By construction of the S-support, $V\in \mathcal{D}_{S_3}(E_1)$. Therefore, there is $a\in A_1$ such that $(V,a)\in S_3$, $V\cap A_1=\emptyset$, $a\notin E_1$, and there is some $c\in \cl_F(V)$ which is attacked by $E_1$ in $F$.
        We argue that $\cl_F(V)\subseteq T_c$: Indeed, by construction of the splitting procedure, the head of all supports that got added to $S^\circledast_2$ lies in $\{*_1,*_2\}$ (by definition of the attack splitting procedure, no new supports are added; by definition of the support splitting procedure, all supports lie either in $S_2$ or satisfy the statement above). 
        Thus, $V\subseteq T_c$ and therefore $\cl_F(V)\subseteq \cl_F(T_c)=T_c$. 
        Consequently, $T_c$ is attacked by $E_1$ in $F$. It follows that $E$ attacks $T_c$ in $F$.
            
        \underline{Case 2.iii: $b=*_2$.} Then $(U,b)$ is of the form $((V\cap A_2)\cup \{*_2\},*_2)$. Since $U\subseteq E_2$, it follows that $E_2$ contains $*_2$, attacks itself on $*_2$, and is therefore not admissible, contradiction to our initial assumption.

        \underline{Case 2.iv: $b\in A_2$.} 
        We show that $b\in T_c$: it holds that $b\in T_c'$ iff $b\in T$ or there is $(W,h)\in S^\circledast_2$ with $W\subseteq T_c'$. In the former case, $b\in T_c$; in the latter case, we have $(W,h)\in S_2$ since $h\in A_2$ and by construction of the S-support, the head of each novel support either corresponds to $*_1$ or $*_2$.

        Thus, $(U,b)\in R_2^\circledast$ with $U\subseteq E_2$ and $b\in T_c$. 
        We will now consider the different forms the attack may have.
        By construction of the combined split and since $b\in A_2$, we have three options: 
        Either (a) $(U,b)\in \hat R_2$;
        or (b) $(U,b)$ is of the form $(V\setminus A_1,b)$ for some $(V,b)\in \hat R_3$, $V\cap (E_1)^+_R=\emptyset$ and $V\cap A_1\subseteq E_1$;
        or (c) $(U,b)$ is of the form $(V\setminus A_1)\cup \{*_0\},b)$ for some $(V,b)\in \hat R_3$, $V\cap (E_1)^+_R=\emptyset$ and $V\cap (A_1\setminus E_1)\neq \emptyset$.

        \begin{itemize}
            \item \underline{Case 2.iv.a:} We have $(U,b)\in \hat R_2$.
            Thus, there is $U'\subseteq U$ so that $(U,b)\in R_2$. Therefore, $E$ defends itself against $T_c$ in $F$ since $U\subseteq E$ and $b\in T_c$.

            \item \underline{Case 2.iv.b:} Suppose $(U,b)$ is of the form $(V\setminus A_1,b)$ for some $(V,b)\in \hat R_3$, $V\cap (E_1)^+_R=\emptyset$ and $V\cap A_1\subseteq E_1$.
            Then $V\subseteq E$ since $U=V\setminus A_1 \subseteq E_2$ and $V\cap A_1\subseteq E_1$. Thus there is an attack $(V,b)\in \hat R_3$ with $V\subseteq E$ and $b\in T_c$; in other words, $E$ defends itself against $T_c$ in $F$.

            \item \underline{Case 2.iv.c:} Since $U\subseteq E_2$, it follows that $E_2$ contains $*_0$, attacks itself on $*_0$, and is therefore not admissible, contradiction to our initial assumption.
        \end{itemize}
        This concludes Case 2; we have shown that $E$ defends itself against $T_c$ if $(T_c,h)\in \hat R_2$.
        
        \item \underline{Case 3: $(T_c,h)\in \hat R_3$.} 
        By definition of $\hat R_3$, it holds that $T_c\cap A_1\neq \emptyset$.
        We proceed by case distinction.

        \begin{itemize}
            \item \underline{Case 3.i: $T_c\cap (E_1)^+_R\neq \emptyset$.}
            In this case, $E$ defends itself $T_c$ in $F$.
                
            \item \underline{Case 3.ii: $T_c\cap (E_1)^+_R=\emptyset$.}
            In this case, either (a) $(T_c\cap A_2,h)\in R^\circledast_2$ (via the R-reduct) or (b) $((T_c\cap A_2)\cup \{*_0\},h)\in R^\circledast_2$ (via the modification). 
            Let $T_c'=\cl_{F^\circledast_2}(T_c\cap A_2)$. 
            Observe that 
            \begin{align*}
                \cl_{F_2^\circledast}((T_c\cap A_2)\cup \{*_0\})=&\, \cl_{F_2^\circledast}(T_c\cap A_2)\cup \{*_0\}\\=&\,T_c'\cup \{*_0\}
            \end{align*}
            since $*_0$ is not contained in the tail of any support. 
            
            $E_2$ defends itself against $T_c'$ (in Case a) resp.\ $T_c'\cup \{*_0\}$ (in Case b) since it is admissible in $F_2^\circledast$.
            Thus, there is $(U,b)\in R_2^\circledast$ such that $U\subseteq E_2$ and $b\in T_c'$ (in Case a) or $b\in T_c'\cup \{*_0\}$ (in Case b).

            As in Case 2, we consider the different options for $b$.
            \begin{itemize}
                \item 
                \underline{Case 3.ii.I: $b=*_0$.} Then $(U,b)=(\{*_0\},*_0)$.
                As in Case 2.i, we obtain that $E_2$ contains the self-attacker $*_0$ and is therefore not admissible, contradiction to our initial assumption.
                \item 
                \underline{Case 3.ii.II: $b=*_1$.} Analogous to Case 2.ii. 
                We sketch the proof: As in Case 2.ii, we have $U=\emptyset$ and $*_1\in T_c'$ and thus there is some $V\subseteq T_c'$ so that $(V,*_1)\in S^\circledast_2$ with $V\in \mathcal{D}_{S_3}(E_1)$. Thus $E_1$ attacks $\cl_F(V)$. Analogous to Case 2.ii, we can show that $\cl_F(V)\subseteq T_c$. We obtain that $E$ attacks $T_c$ in $F$.
                \item 
                \underline{Case 3.ii.III: $b=*_2$.} Analogous to Case 2.iii.
                \item
                \underline{Case 3.ii.IV: $b\in A_2$.} 
                As in case 2.iv, we can prove that $b\in T_c$ in this case. That is, we have $(U,b)\in R_2^\circledast$ with $U\subseteq E_2$ and $b\in T_c$. 
                We consider the different forms the attack may have.
                By construction of the combined split and since $b\in A_2$, we have three options: 
                Either (a) $(U,b)\in \hat R_2$;
                or (b) $(U,b)$ is of the form $(V\setminus A_1,b)$ for some $(V,b)\in \hat R_3$, $V\cap (E_1)^+_R=\emptyset$ and $V\cap A_1\subseteq E_1$;
                or (c) $(U,b)$ is of the form $(V\setminus A_1)\cup \{*_0\},b)$ for some $(V,b)\in \hat R_3$, $V\cap (E_1)^+_R=\emptyset$ and $V\cap (A_1\setminus E_1)\neq \emptyset$.
                Proceed as in Case 2.iv.a, 2.iv.b, and 2.iv.c, respectively. 
            \end{itemize}
        \end{itemize} 
        
    \end{itemize}
    This concludes the proof. We have shown that $E$ defends itself against $T_c$, regardless of whether the attack $(T,h)$ with $\cl_F(T)=T_c$ stems from $R_1$, $R_2$, or $R_3$.

    \item Suppose $E\in \adm(F)$. We show that $E\cap A_1\in \adm(F_1)$ and $E\cap A_2\in \adm(F^\circledast_2)$. By Lemma \ref{lem:closure attack support splitting} we know that both sets are closed; moreover, by Lemma \ref{lem:conflict-freeness attack support splitting} they are also conflict-free. It remains to prove that the sets defend itself. 
    \begin{itemize}
        \item The case for $E_1=E\cap A_1$ holds since for all $(T,h)\in R$ with $h\in E_1$ it holds that $(T,h)\in R_1$. Moreover, $\cl_F(T)=\cl_{F_1}(T)$ since no positive links go from $F_1$ to $F_2$. Thus $E_1$ defends itself against all closed attacker in $F_1$ since $E$ is admissible in $F$, by assumption.
        
        \item Let $E_2=E\cap A_2$. Let $T_c$ be a closed attacker of $E_2$ in $F^\circledast_2$ such that there is $(T,h)\in R^\circledast_2$ with $T_c=\cl_{F^\circledast_2}(T)$. 

        We note that $h\in A_2$ since $h\in E_2=E\cap A_2$.
        Thus, the attack $(T,h)$ cannot stem from constructing the S-reduct (the head of newly introduced attacks is either $*_1$ or $*_2$). 
        Therefore, $(T,h)$ satisfies one of the following:
        \begin{itemize}
            \item $(T,h)\in \hat R_2$; 
            %in this case, $T=\cl_F(T)=T_c\setminus \{*_0,*_1,*_2\}$;
            \hfill (Case 1)
            \item $(T,h)=(T'\setminus A_1,h)$ for some $(T',h)\in \hat R_3$, $T'\cap (E_1)^+_R=\emptyset$ and $T'\cap A_1\subseteq E_1$;  
            %in this case $T'\setminus A_1=\cl_F(T'\setminus A_1)=T_c\setminus \{*_0,*_1,*_2\}$
            \hfill (Case 2)
            \item $(T,h)=(T'\setminus A_1)\cup \{*_0\},h)$ for some $(T',h)\in \hat R_3$, $T'\cap (E_1)^+_R=\emptyset$ and $T'\cap (A_1\setminus E_1)\neq \emptyset$.  \hfill (Case 3)
        \end{itemize}
        Let us first inspect the relationship between $\cl_F(T)$ and $T_c$.
        We show that 
        %This is because $T$ and $T'$, respectively are closed in $F$, and the newly introduced supports may additionally add one of the newly introduced arguments only. 
        \begin{align*}
            T_c\setminus \{*_0,*_1,*_2\}=\cl_F(T).
        \end{align*}
        For Case 1, this follows from $T_c=\cl_{F^\circledast_2}(T)$: It holds that $a\in \cl_{F^\circledast_2}(T)$ iff $a\in T$ or there are supports $(W_1,a_1),\dots,(W_k,a_k)\in S^\circledast_2$ with $a_k=a$ and $W_i\subseteq T\cup \bigcup_{j<i} a_j$. 
        Since $(T,h)\in \hat R_2$, we have $T\subseteq A_2$, thus $*_0\notin T$.
        Moreover, 
        since $*_1$ and $*_2$ do not appear in the tail of any support we can assume that they do not appear in any $(W_i,a_i)$ for $i < k$. Thus, if $a\in A_2$, then $(W_i,a_i)\in S_2$ for all $i\leq k$, thus $a\in \cl_F(T)$; in the latter case, $a \in \{*_1,*_2\}$.
        That is, in $F^\circledast_2$, $T$ may additionally support some of the newly introduced arguments, thus the statement follows. 
        For Case 2 and 3, we show $T_c\setminus \{*_0,*_1,*_2\}=\cl_F(T'\cap A_2)$: first note that $T'\cap A_2=\cl_F(T\cap A_2)$ since $T'$ is closed in $F$ (let $(W,a)\in S$ with $W\subseteq T'\cap A_2$; then $a\in \cl_F(T')$; if $a\in A_2$ we have $a\in T'\cap A_2$). 
        Since $T=T'\cap A_2$ we obtain $T_c\setminus \{*_0,*_1,*_2\}=\cl_F(T'\cap A_2)=\cl_F(T)$ analogously to the first case. \hfill $\Diamond$
        \medskip

        Having settled the relation between $\cl_F(T)$ and $T_c$, we observe that in all cases, $E$ is attacked by a closed set $T$ ($T'$, respectively) in $F$. By assumption, $E$ defends itself in $F$. Thus, there is $(U,b)\in R$ with $U\subseteq E$ and $b\in T$ ($b\in T'$, respectively). 
        We furthermore observe that in all cases, $b\in A_2$: in Case 1, this follows since $T\subseteq A_2$ by definition of $\hat R_2$; in Case 2 and 3, this follows from the condition $T'\cap (E_1)^+_R=\emptyset$ (suppose $b\in A_1$; by definition of the splitting tuple, we have $(U,b)\in R_1$; thus we also have $U\subseteq A_1$ and therefore $U\subseteq E_1$  and thus $E_1$ attacks $T'$; contradiction).

        By construction of the combined split, 
        we thus get an attack $(U_c,b)\in \hat R_2\cup \hat R_3$ with $\cl_F(U)=U_c$ with $b\in T$ (which amounts to $b\in T'\cap A_2$ in Case 2 and 3). As shown above $\cl_F(T)\subseteq T_c$.
        
        Below, we distinguish the cases $(U_c,b)\in \hat R_2$ and $(U_c,b)\in \hat R_3$.

        \begin{itemize}
            \item \underline{Case i: $(U_c,b)\in \hat R_2$.} Then $(U_c,b)\in R_2^\circledast$ and thus $E_2$ defends itself against $T_c$ in $F^\circledast_2$.

            \item \underline{Case ii: $(U_c,b)\in \hat R_3$.} 
            We have $U_c\cap A_1\subseteq E_1$ and $U_c$ is not attacked by $E_1$ in $F$ (otherwise, $E\notin \adm(F)$, contradiction), therefore, $R^\circledast_2$ contains an attack of the form $(U_c\cap A_2,b)$. Since $U_c\cap A_2\subseteq E_2$ and $b\in T$, we have $E_2$ defends itself against $T_c$ in $F^\circledast_2$.
        \end{itemize}
    \end{itemize}
\end{enumerate}

\paragraph{Complete Semantics.}
\begin{enumerate}
    \item Suppose $E_1\in \com(F_1)$ and $E_2\in\com(F^\circledast_2)$. 
    We show that $E=E_1\cup (E_2\setminus\{*_2\}) \in \com(F)$.
    
    The premisses $E_1\in \com(\BF_1)$ and $E_2\in \com(\BF^\circledast_2)$ directly entail $E_1\in \adm(\BF_1)$ and $E_2\in \adm(\BF^\circledast_2)$, hence $E\in\adm(\BF)$ as shown above.
    It remains to be shown that $E$ defends itself in $\BF$.

    Assume toward contradiction that there exists an argument $d\in A\setminus E$ that is defended by $E$ in $F$.
    
    We proceed by case distinction.
    \begin{itemize}
        \item \underline{Case 1: $d\in A_1$.} 
        From $E_1\in\com(F_1)$ and $d\notin E_1$, we conclude there exists $(T_1,d)\in R_1$, s.t. $\cl_{F_1}(T_1)\cap E_1^+=\emptyset$. By the construction of $R_1$ we can infer that $(T_1,d)\in R$ and $\cl_{F}(T_1)\cap E^+=\emptyset$, therefore $d$ is not defended by $E$ in $F$. This contradicts our assumption.
        \item \underline{Case 2: $d\in A_2$} From $E_2\in\com(F_2^\circledast)$ and $d\notin E_2$, we conclude there exists $(T_2,d)\in R_2^\circledast$, s.t. $\cl_{F_2^\circledast}(T_2)\cap E_2^+=\emptyset$.

        By the construction of $R_2^\circledast$ we can infer that $\exists (T,d)\in R_2\cup R_3$ s.t. $T_2\subseteq(\cl_F(T)\cap A_2)\cup \{*_0\}$ holds.
        Further $\cl_{F_2^\circledast}(T_2)\cap E_2^+=\emptyset$ implies $\cl_F(T)\cap E^+=\emptyset$. Thus $d$ is not defended by $E$ in $F$. Contradiction!
    \end{itemize}
    
    \item Suppose $E\in \com(F)$. We show that $E_1=E\cap A_1\in \com(F_1)$ and ($E\cap A_2\in \com(F^\circledast_2)$ or $(E\cap A_2)\cup \{*_2\}\in \com(F^\circledast_2)$).
        \begin{itemize}
            \item Utilizing the proof for admissible semantics, we directly conclude $E_1=E\cap A_1\in \adm(F_1)$ and $E_2=E\cap A_2\in \adm(F^\circledast_2)$ or $E_2=(E\cap A_2)\cup\{*_2\}\in \adm(F^\circledast_2)$.
            \item There are no attacks from $A_2$ towards $A_1$, neither are there supports from $A_1$ towards $A_2$. Hence any argument in $A_1$ is defended by $E$ in $F$ iff it is defended by $E_1=E\cap A_1$ in $F_1$. Thus $E_1$ defends itself in $F_1$, therefore $E_1=E\cap A_1\in \com(F_1)$.
            \item It remains to be shown that $E_2$ defends itself in $F^\circledast_2$.
            We assume toward a contradiction that there exists a $c\in A^\circledast_2\setminus E_2$ s.t. $c$ is defended by $E_2$ in $F_2^\circledast$. Then for every closed attacker $D$ with $(D',c)\in R_2^\circledast, D'\subseteq D$ there exists an attack $(T,d)\in R_2^\circledast, d\in D, T\subseteq E_2$. We move on to show that in this case $E$ defends $c$ in $F$. 
            Note first that $c\neq*_0$, because $*_0$ cannot be defended. If $c=*_1$ then $E_2$ attacks the (non-empty) closure of the empty set in $F_2$, so $E$ does the same in $F$, since any support by the empty set in $S^\circledast_2$ is also in $S_2$, but then $E$ is not conflict-free. Contradiction. So $c\neq *_1$. 
            If $c=*_2$ then $E\cap A_2\cup \{*_2\}$ is complete. We already have $E\cap A_2$ is admissible in $F^\circledast_2$, so it is closed and conflict-free and defends itself. Adding $*_2$ preserves closedness, since the only supports involving $*_2$ point to itself and since we defend $*_2$, defense is maintained. Now, if adding $*_2$ would violate conflict-freeness, then there exists a $T\subseteq E\cap A_2$ such that $(T\cup{*_2},*_2)\in R_2^\circledast$, so there is a support $(T\cup T',h)$ with $h\notin E, T'\subseteq E_1$ in $F$. But then $E$ is not closed. Contradiction. So $E\cap A_2\cup\{*_2\}$ is indeed conflict-free and therefore complete.

            Suppose now $c\in A_2$. We have to show, that for every attack $(D',c)\in R$ we have $\cl(D)$ is attacked by $E$ in $F$. Let $(D',c)\in R$. We distinguish the following cases:
            \begin{itemize}
                \item (Case 1: $(D',c)\in R_3$). Then after the closure of attacks we have some closed attack $(D,c)\in \hat{R}_3$. If $D\cap E_1^+\neq \emptyset$, then $E$ defends $c$ against $D$. Otherwise there are two cases:
                \begin{itemize}
                    \item (Case 1.i $D\cap A_1\subseteq E$) Then $(D\cap A_2,c)$ is a closed attack in $F^\circledast_2$. So there exists an attack $(T,d)\in R^\circledast_2$ with $T\subseteq E_2,d\in D\cap A_2$, since $E_2$ defends $c$. But for such an attack we have $*_2\notin T$, since $*_2$ is only part of attacks on itself. $T\cap\{*_0,*_1\}=\emptyset$, since $E$ is admissible. So $T\subseteq A_2$, therefore $(T,d)\in\hat{R}_2\cup \hat{R}_3$, so there exists an attack $(T',d)\in R$ with $T'\subseteq T\subseteq E$, so $E$ defends $c$ in $F$ but does not contain it, since $c\in A_2\setminus E_2$. Contradiction.
                    \item (Case 1.ii $D\cap A_1\nsubseteq E$) Then $((D\cap A_2)\cup*_0,c)$ is an attack in $F^\circledast_2.$. We can distinguish the following cases:
                    \begin{itemize}
                        \item (Case 1.ii a - $cl_F(D\cap A_2)\subseteq A_2$
                        \item (Case 1.ii b - $cl_F(D\cap A_2)\cap E^+_1\neq \emptyset$) Then $E$ defends $c$ via $E_1$ but does not contain it. Contradiction to $E\in\com(F)$.
                        \item (Case 1.ii c - $cl_F(D\cap A_2)\cap A_1\neq\emptyset,cl_F(D\cap A_2)\cap E^+_1= \emptyset $) Then $((D\cap A_2),*_2)\in S^\circledast_2$. So $cl_{F^\circledast_2}(D\cap A_2)=(D\cap A_2)\cup\{*_2\}$ and there is an attack $(T,d)$ from $E_2$ to $(D\cap A_2)\cup\{*_2\}$ in $F^\circledast_2$. Since any attack on $\{*_2\}$ implies the attacking set is not conflict-free and $E_2$ is conflict-free, $d\neq*_2$. So $(T,d)\in\hat{R}_2\cup\hat{R}_3$, but in this case there exists an attack $(T',d)\in R$ from $E$ to $D$, so $E$ defends $c$ in $F$. Since $E$ does not contain $c$ but $E\in\com(F)$, Contradiction.
                    \end{itemize}
                \end{itemize}
                \item (Case 2: $(D',c)\in R_2$). Then after the closure of attacks we either have some closed attack $(D,c)\in \hat{R}_3$ or $(D,c)\in \hat{R_2}$. For the former refer to Case 1. In the latter case $(D,c)\in R^\circledast_2$ so there exists an attack $(T,d)\in R_2^\circledast, d\in D, T\subseteq E_2$ on $D$ in $F^\circledast_2$. We have $*_2,*_1\notin D$, since $D$ is closed in $A_2$. Furthermore $*_0\notin D$, since $*_0\notin A$. We also have $*_2\notin T$, since $*_2$ is only part of attacks on itself. So $d\in A_2, T\subseteq E_2$, so $(T,d)\in\hat{R}_2$ is a closed attack from $E_2$ to $D$, so there exists an attack $(T',d)\in R, T'\subseteq T$, so $E$ defends $c$ in $F$ and does not contain it. But $E\in\com(F)$, Contradiction. 
            \end{itemize}
        \end{itemize}
\end{enumerate}

\paragraph{Stable Semantics.}
Notice that under stable semantics, we have $U^{E_1}_{R\clink_3}=\emptyset$, therefore $F^\star_2=\hat F\redEone_2$ and $F^\circledast_2$ is the $S$-reduct of $\hat F_2\redEone$ wrt.\ $E_1$. 
\begin{enumerate}
\item Suppose $E_1\in \stb(F_1)$ and $E_2\in \stb(F^\circledast_2)$. We show that $E=E_1\cup (E_2\setminus \{*_2\})\in \stb(F)$. 
    \begin{itemize}
        \item E is conflict-free: By Lemma~\ref{lem:conflict-freeness attack support splitting} and $E_1\in \adm(F_1)$ and $E_2\in \adm(F^\circledast_2)$, we obtain $E\in \cf(F)$.
        \item E is closed because of Lemma \ref{lem:closure attack support splitting} and the fact that $\stb(F)\subseteq \adm(F)$ for every BSAF $F$.
        \item $E$ attacks every argument $a\in A\setminus E$. Let $a\in A\setminus E$. We proceed by case distinction.
            \begin{itemize}
                \item \underline{Case 1: $a\in A_1$.} Then $a\in (E_1)^+_{R_1}$ since $E_1\in \stb(F_1)$, thus $a\in E^+_{R}$.
                \item \underline{Case 2: $a\in A_2$.} By hypothesis, we know that $a\in (E_2)^+_{R_2^\circledast}$. Hence, there is a $T\subseteq E_2$ such that $(T,a)\in R_2^\circledast$. Since $R_2^\circledast \supseteq \hat R_2\redEone$, we distinguish two further cases.
                \begin{itemize}
                    \item \underline{Case 2.1. $(T,a)\in \hat R_2\redEone$.} 
                    By definition of $R$-reduct, we can distinguish whether $(T,a)\in \hat R_2$ or not. We proceed by case distinction. 
                    \begin{itemize}
                        \item \underline{Case 2.1.1. $(T,a)\in \hat R_2$.} In this case, we know by definition of $\hat R_2$ that there is a $T'\subseteq T\subseteq E_2$ such that $T=cl_S(T')$ and $(T',a)\in R_2$. Hence, $a\in (E_2)^{+}_{R_2}$ and, consequently, $a\in E^+_{R}$. 
                        \item \underline{Case 2.1.2. $(T,a)\notin \hat R_2$.} Then, it means that $(T,a)$ is introduced by the $R$-reduct. That is, there is a $T'\supseteq T$ such that $(T',a)\in R_3\clink$, $ T'\cap (E_1)^+_{R_1\cup R_3\clink}=\emptyset$ and $T'\cap A_1\subseteq E_1$. Thus, $T'\subseteq E$ and $a\in (E)^+_{R_3}$ and, consequently, $a\in E^+_{R}$. 
                    \end{itemize}
                    \item \underline{Case 2.2. $(T,a)\notin \hat R_2\redEone$.} This means that $(T,a)$ has been introduced via the $S$-reduct. 
                    \begin{itemize}
                        \item \underline{Case 2.2.1. $a=*_1$.} Then $T=\emptyset$ by definition of $S$-reduct. Thus, $*_1\notin E_2$ since it is admissible. Therefore, $*_1\notin E$  
                        \item \underline{Case 2.2.2. $a=*_2$.} Then $*_2\in T$ from the definition of $S$-reduct. From the fact that $(E_2)^\oplus_{R_2^\circledast}=A_2^\circledast$ and $E_2$ is conflict-free, we derive that either $*_2\in E_2$ or $*_2\in (E_2)^+_{R_2^\circledast}$. However, the second case is in contradiction with conflict-freeness of $E_2$. Hence, $*_2\in E_2$. 
                    \end{itemize}
                \end{itemize}
    \end{itemize}
    Under each case, we derive that $a\in E^+_{R}$. Further, $*_1\notin E_2$ and $*_2\in E_2$. Since $*_2\notin A_2$, it follows that $E_1\cup (E_2\setminus \{*_2\})^\oplus=A$.
    \end{itemize}
    \item Suppose $E\in \stb(F)$, that is $E^\oplus_R=E\cup E^+_R=A$. 
    \begin{itemize}
        \item We first show that $E_1=E\cap A_1\in \stb(F_1)$. From previous lemmata we know that $E_1$ is closed and conflict-free. It remains to prove that $E_1$ attacks every $a\in A_1\setminus E_1$. Since there is no attack from $F_2$ to $F_1$ by definition of splitting, we know that $a\in (E)^+_{R_1}$ by some $T\subseteq E\cap A_1$, i.e.\ $a\in (E\cap A_1)^+_{R_1}$.
        \item We show that $E\cap A_2\in \stb(F^\circledast_2)$ or 
        $(E\cap A_2)\cup \{*_2\}\in \stb(F^\circledast_2)$. First, consider $*_2$. If $*_2$ is not in a stable extension (as in the first disjunct), then it is attacked by it. However, the only way that $*_2$ is attacked is via a set-self-attack by definition of $S$-reduct. Hence, $E\cap A_2$ cannot be stable, because is not conflict-free. Therefore, in the remainder we prove that $(E\cap A_2)\cup \{*_2\}\in \stb(F^\circledast_2)$.
        
        From Lemma~\ref{lem:closure attack support splitting} we know that $E\cap A_2$ is closed in $F_2^\circledast$. By definition of the $S$-reduct, $*_2$ does not participate to any support in $F_2^\circledast$. Hence, $(E\cap A_2)\cup \{*_2\}$ is closed in $F_2^\circledast$. 

        From Lemma~\ref{lem:conflict-freeness attack support splitting} we know that $E\cap A_2$ is conflict-free in $F_2^\circledast$. Now suppose $(E\cap A_2)\cup \{*_2\}\notin \cf(F_2^\circledast)$. Then, $((T\cap A_2)\cup \{*_2\},*_2)\in R_2^\circledast$ for some $T\cap A_2\subseteq E\cap A_2$ such that $T\in \mathcal{T}_{S_3}(E_1)\setminus \mathcal{D}_{S_3}(E_1)$. Therefore, we know that $T\cap A_1\subseteq E_1$, from which we derive that $T\subseteq E$. Since $E$ is closed in $F$ from hypothesis and $T\in \mathcal{T}_{S_3}(E_1)$, we derive that there is an $a\in A_1\setminus E_1$ such that $(T,a)\in S_3$. This contradicts the fact that $E$ is closed in $F$.

        It remains to prove that $E\cap A_2\cup \{*_2\}$ attacks every  argument $a\in A_2^\circledast \setminus ((E \cap A_2)\cup \{*_2\})$. We proceed by case distinction. 
        \begin{itemize}
            \item \underline{Case 1. $a\in A_2$. } By hypothesis, we know that $a\in E_R^+$ because $E$ is stable in $F$. By definition of splitting, there are no attacks from $F_1$ to $F_2$. Hence, we can distinguish two cases.
            \begin{itemize}
                \item \underline{Case 1.1. $a\in E_{R_2}^+$.} Thus, there is an attack $(T,a)\in R_2$ such that $T\subseteq E\cap A_2$. 
                Let $T_c=cl_S(T)$. Since $E$ is closed in $F$ by hypothesis, we know that $T_c\subseteq E$. Hence, when computing the closure $R_2\clink$, we distinguish two further cases.
                \begin{itemize}
                    \item \underline{Case 1.1.1. $T_c\cap A_1=\emptyset$.} Then $(T_c,a)\in \hat R_2$. By definition of support and attack reducts, we know that $\hat R_2\subseteq R_2^\circledast$, and thus $a\in E_{R_2^\circledast}^+$.  
                    \item \underline{Case 1.1.2. $T_c\cap A_1\neq \emptyset$.} Then $(T_c,a)\in \hat R_3$ by definition. Further, for all $t\in T_c\cap A_1$, we know that $(T,t)\in S_3$.  
                    Since $T_c\subseteq E$, we get $T_c\cap A_1\subseteq E\cap A_1$. Then, $(T_c\setminus A_1,a)\in \hat R_2\redEone$ is in the $R$-reduct, and $a\in E_{\hat R_2\redEone}^+$. Further, since $\hat R_2\redEone\subseteq R_2^\circledast$, it follows that $a\in E_{R_2^\circledast}^+$. 
                \end{itemize}
                \item \underline{Case 1.2. $a\in E_{R_3}^+$. } There is some $T\subseteq E$ such that $(T,a)\in R_3$. Let $T_c=cl_S(T)$. Then by definition of $R_3\clink$, we know that $(T_c,a)\in R_3\clink$. Since $E$ is closed in $F$ by hypothesis, we know that $T_c\subseteq E$. Further, from $E\in \cf(F)$, it follows that $T_c\cap (E_1)^+_{R_1\cup R_3\clink}=\emptyset$. Moreover, $T_c\cap A_1\subseteq E\cap A_1$ because $T_c\subseteq E$. Hence, by definition of $R$-reduct, there is a $(T_c\setminus A_1,a)\in R_2^\circledast$, and consequently, $a\in E_{R_2^\circledast}^+$. 
            \end{itemize}
            In both cases we find $a\in E_{R_2^\circledast}^+$.  
        \end{itemize}
        \item \underline{Case 2. $a\notin A_2$.} Thus $a\in \{*_1\}$. By definition of $S$-reduct we know that $*_1\in (E\cap A_2)^+_{R_2^\circledast}$. \qedhere
    \end{itemize}
\end{enumerate}

\end{proof}

\PrefGrdSuppAttSplit*
\begin{proof}
We provide a proof for both semantics. 
%\anna{done, should be ok}
\paragraph{Preferred Semantics.}
Suppose $E_1\in \prf(F_1)$ and $E_2\in \prf(F^\circledast_2)$. Since $E_1$ and $E_2$ are admissible in $F_1$ and $F_2^\circledast$, respectively, we obtain $E\in \adm(F)$.
    We show that $E=E_1\cup E_2\setminus \{*_2\}\in \prf(F)$.

    Towards a contradiction, suppose $E\notin \prf(F)$. 
    Then there is $E'\supsetneq E$ with $E'\in \adm(F)$.
    Then $E_1'=E'\cap A_1\in \adm(F_1)$ and $E_2'=E'\cap A_2\in \adm(F_2^\circledast)$.
    
    $E'$ is strictly larger than $E$.
    Proceed with case distinction.
    \begin{itemize}
        \item 
        \underline{Case 1:} $E_1'$ is strictly larger than $E_1$. This is in contradiction to $E_1\in \prf(F_1)$.
    
        \item \underline{Case 2:} $E_2'$ is strictly larger than $E_2\setminus \{*_2\}$.
    
        \underline{Case 2.i: $*_2\in E_2$.} Let $E''_2=E_2'\cup \{*_2\}$.
        If $E_2''\in \adm(F_2^\circledast)$ then $E''_2\supsetneq E_2$, contradiction to $E_2\in \prf(F_2^\circledast)$.

        We show that $E_2''\in \adm(F_2\redEone)$:
        $E_2''$ is closed since $E_2'$ is closed and $*_2\notin U$ for any $(U,b)\in S_2'$. 
        $E_2''$ defends $*_2$ since $E_2\subseteq E_2''$ defends $*_2$.

        $E_2''$ is conflict-free: Towards a contradiction, suppose there is an attack $(T,h)\in R_2^\circledast$ such that $T\subseteq E_2''$ and $h\in E_2''$. 

        \underline{Case 2.i.a: $h = *_2$.} By construction, $T$ has the form $((T'\cap A_2)\cup \{*_2\},*_2)$ for some $T'\in \mathcal{T}_{S_3}(E_1)$.  Thus, $T'\cap A_2\subseteq E_2'$. 
        Since $(T'\cap A_2,*_2)\in S_2^\circledast$, this implies $E_2'$ is not closed. Contradiction to $E_2'\in \adm(F_2^\circledast)$.

        \underline{Case 2.i.b: $h \in A_2$.} In this case, $T\subseteq E_2'$ and we have a contradiction to the conflict-freeness of $E_2'$ in $F^\circledast_2$.

    \underline{Case 2.ii: $*_2\notin E_2$.} Then $E_2'=E_2\setminus \{*_2\}$ and therefore $E_2'\supsetneq E_2$,
    contradiction to $E_2\in \prf(F_2^\circledast)$.
    \end{itemize}

\paragraph{Grounded Semantics.}
Suppose $E_1\in \grd(\BF_1)$ and $E_2\in \grd(\BF^\circledast_2)$. 
Then $E_1\in \com(\BF_1)$ and $E_2\in \com(\BF^\circledast_2)$.
Therefore, by Theorem~\ref{thrm:AttSuppSplit}, we have $E=E_1\cup E_2\setminus\{*_2\}\in\com(\BF)$. Assuming toward contradiction $\exists G\subsetneq E:G\in\com(\BF)$, then (1) $\exists G_1\subsetneq E_1:G_1\in\com(\BF_1)$ or (2) $\exists G_2\subsetneq E_2:G_2\in\com(\BF_2^\circledast)$. 

\underline{Case 1:} Contradiction to $E_1$ is grounded in $F_1$.

\underline{Case 2:} We distinguish (i) $*_2\in E_2$ and (ii) $*_2\notin E_2$.
In case (i), we are done (then $G_2\subsetneq E_2$ or $G_2\cup \{*_2\}\subsetneq E_2$, contradiction). 
In case (ii), suppose $G_2\cup \{*_2\}$ is grounded in $F^\circledast_2$. 
Then for all attacks in $R^\circledast_2$ of the form $((T\cap A_2)\cup \{*_2\},*_2)$, it holds that $(G_2)^+_{R^\circledast_2}\cap (T\cap A_2)\neq \emptyset$. Since $G_2\subsetneq E_2$, it follows that for all attacks in $R^\circledast_2$ of the form $((T\cap A_2)\cup \{*_2\},*_2)$, $(E_2)^+_{R^\circledast_2}\cap (T\cap A_2)\neq \emptyset$. Thus, $E_2$ defends $*_2$, contradiction to $*_2\notin E_2$.

Thus $E\in\grd(\BF)$.
\end{proof}

\end{document}